\documentclass{article}
\usepackage{graphicx}
\usepackage{amsmath}
\usepackage{algorithm}
\usepackage[noend]{algpseudocode}
\usepackage{amsfonts}
\usepackage{amssymb}
\usepackage{bm}
\usepackage{natbib}
\usepackage{enumitem}
\usepackage{bbm}
\usepackage{comment}
\usepackage{stmaryrd}
\usepackage{tabularx}
\usepackage{array}
\usepackage{theoremref}
\usepackage{mathrsfs}
\usepackage{amsthm}
\usepackage{booktabs}
\usepackage{multirow}
\usepackage{tikz}
\usetikzlibrary{positioning,arrows.meta}
\usepackage{subcaption}
\usepackage[margin=1in]{geometry}

\newcommand\independent{\protect\mathpalette{\protect\independenT}{\perp}}
\def\independenT#1#2{\mathrel{\rlap{$#1#2$}\mkern2mu{#1#2}}}

\usepackage{setspace}
\usepackage{xcolor}
\usepackage{authblk}

\renewcommand{\P}{\mathbb{P}}

\newcommand{\Bern}{\mathrm{Bern}}
\newcommand{\Beta}{\mathrm{Beta}}
\newcommand{\Reg}{\mathrm{Reg}}
\newcommand{\N}{\mathbb{N}}
\newcommand{\1}{\mathbbm{1}}

\newcommand{\E}{\mathbb{E}}
\newcommand{\R}{\mathbb{R}}
\newcommand{\Q}{\mathbb{Q}}
\newcommand{\X}{\mathbb{X}}

\newcommand{\calF}{\mathcal{F}}

\newcommand{\calX}{\mathcal{X}}
\newcommand{\calV}{\mathcal{V}}

\newcommand{\calH}{\mathcal{H}}
\newcommand{\calG}{\mathcal{G}}

\newcommand{\calI}{\mathcal{I}}
\newcommand{\calR}{\mathcal{R}}
\newcommand{\calB}{\mathcal{B}}
\newcommand{\calJ}{\mathcal{J}}
\newcommand{\calT}{\mathcal{T}}
\newcommand{\calC}{\mathcal{C}}

\newcommand{\calA}{\mathcal{A}}
\newcommand{\calN}{\mathcal{N}}
\newcommand{\calS}{\mathcal{S}}

\newcommand{\calD}{\mathcal{D}}
\newcommand{\calP}{\mathcal{P}}

\newcommand{\calU}{\mathcal{U}}

\newcommand{\frakv}{\mathfrak{v}}
\newcommand{\frakp}{\mathfrak{p}}
\newcommand{\Epsilon}{\mathcal{E}}

\newcommand{\Var}{\mathrm{Var}}

\newcommand{\wh}[1]{\widehat{#1}}
\newcommand{\wt}[1]{\widetilde{#1}}
\newcommand{\wb}[1]{\overline{#1}}
\newcommand{\lb}[1]{\underline{#1}}

\newcommand{\diam}{\mathrm{diam}}

\newcommand{\vc}{\mathrm{vc}}

\newcommand{\zhan}{\texttt{StabVar}}
\newcommand{\biba}{\texttt{CADR}}
\newcommand{\dd}{\texttt{DD}}
\newcommand{\aw}{\texttt{AW}}

\algrenewcommand\algorithmicrequire{\textbf{Input:}}
\algrenewcommand\algorithmicensure{\textbf{Output:}}

\newtheorem{theorem}{Theorem}
\numberwithin{theorem}{section}
\newtheorem{lemma}[theorem]{Lemma}

\newtheorem{prop}[theorem]{Proposition}
\newtheorem{corollary}[theorem]{Corollary}

\newtheorem{ass}{Assumption}

\theoremstyle{definition}
\newtheorem{definition}[theorem]{Definition}

\theoremstyle{remark}
\newtheorem{rmk}[theorem]{Remark}

\usepackage[colorlinks=true, urlcolor=blue, linkcolor=blue, citecolor=black]{hyperref}

\usepackage{nameref}
\usepackage[noabbrev, capitalize, nameinlink]{cleveref}

\usepackage{fullpage}
\title{Inference on the Value of Policy Improvement in Adaptive Experiments}
\author[1]{Justin Whitehouse\thanks{Corresponding author: \texttt{jwhiteho@stanford.edu}.}}
\author[2]{Emma Brunskill}
\author[1]{Vasilis Syrgkanis}

\affil[1]{Management Science and Engineering, Stanford University}
\affil[2]{Computer Science, Stanford University}

\date{October 5, 2026}

\begin{document}
\maketitle
\begin{abstract}
    In this paper, we develop an asymptotically normal estimator for the \textit{value of policy improvement (VPI)} under adaptively collected data.
    This value, which is defined as the gap between the value of the best treatment policy in a class and that of a baseline policy, gauges whether or not it is worth deploying a personalized treatment policy.
    Building an estimator for the VPI is naturally more difficult than estimating the value of a fixed policy.
    This is because, in addition to contending with the correlated data and decaying propensities associated with adaptive experiments, any estimator must also handle the \textit{non-regularity}, or lack of pathwise differentiability, of the underlying estimand when the optimal policy is non-unique.
    Our proposed estimator works by sequentially fitting a weighted empirical welfare maximizing policy on historical data and then evaluating it using inverse propensity weighting on future samples. We neither assume that the optimal policy is unique nor make a margin assumption on the underlying heterogeneous treatment effects. A key difficulty in analyzing our estimator is showing that the regret incurred through policy learning is asymptotically smaller than the standard error of the estimator. We accomplish this by showing faster regret rates for policy learning under decaying propensities, which may be of independent interest.
    We not only evaluate the coverage of our estimator on a mix of synthetic and semi-synthetic data, but also employ it to estimate the VPI in an adaptive experiment evaluating the efficacy of job assistance programs in Jordan due to Caria et al.
    Although our estimates are generally more conservative than their hierarchical Bayesian analysis, which assumes that the optimal policy is unique,  we still corroborate the benefits of personalization amongst Syrian refugees.

\end{abstract}

\section{Introduction}

Researchers regularly conduct experiments in order to learn the effect of one or more interventions
(e.g., discounts or medications) on an outcome of interest (e.g., generated revenue or patient welfare). Given that assigning an individual to a  suboptimal intervention can result in worse outcomes for both providers and patrons alike, some organizations are moving from traditional experiments (e.g., A/B tests) to  \textit{contextual bandit} algorithms. By adapting treatment assignments as data are gathered, these algorithms can balance several objectives: they can direct data collection toward interventions or subgroups for which additional observations would be most informative, identify a best intervention or ``arm'', and prioritize beneficial outcomes for participants.
A common feature of contextual bandit algorithms is that suboptimal or inferior treatments are assigned with probabilities that decrease over the course of the experiment.
This flexibility
has resulted in the adoption of bandit algorithms in disparate domains from advertising/e-commerce \citep{li2010contextual, bottou2013counterfactual, sharma2022doordashbandit} to medicine/clinical trials~\citep{villar2015multi, varatharajah2022contextual, tewari2017ads}.

Given the ubiquitous adoption of adaptive experimentation across industries, it is only natural that one may want to perform statistical inference or learning \textit{post facto}, or ``after the fact'', on contextual bandit data.
However, it is precisely the ability of adaptive experiments to balance \textit{exploration} against \textit{exploitation} that makes such inference difficult.
In more detail, the correlation amongst the data alongside vanishing propensities can cause standard estimators to be biased and naive confidence intervals to miscover.
To counteract this, authors have developed a number of estimators that yield valid inference under adaptively collected data.
For instance, \citet{bibaut2021post}, \citet{zhan2021off}, and \citet{hadad2021confidence} have considered the problem of \textit{off-policy evaluation (OPE)}, developing estimators that leverage martingale central limit theorems to account for decaying propensities.
Likewise, \citet{bibaut2021risk}, \citet{zhan2024policy}, and \citet{girard2025online} study policy learning under decaying propensities, showing that certain weighted empirical welfare maximization (EWM) algorithms can yield vanishing regret.
Lastly, \citet{jia2024cramming} develop a ``cramming'' estimator for \textit{on-policy evaluation}, which provides valid confidence intervals for the value of the deployed bandit policy under sufficiently slow propensity decay and policy stability.
These results demonstrate that many policy learning-related tasks are still feasible under bandit logging data.

The existing literature, however, does leave a notable problem unaddressed. Suppose that, sometime after an adaptive experiment has concluded, a learner wants to determine whether or not it is worth deploying a new policy from a fixed policy class.
In this setting, they no longer want to perform inference on the value of a fixed policy (i.e., OPE) or the value of the final deployed policy (i.e., cramming). Rather, they might wish to understand \textit{the value of policy improvement (VPI)}, or the gap between the value of a best policy in a class $\Pi$ and  some  baseline policy, such as uniform randomization or an all control policy.
More formally, let $X \in \calX$ represent an individual's characteristics drawn from some distribution, $Y(a)$ a unit's potential outcome under some categorical treatment $a \in \calA$, and $q(a \mid x)$ a (potentially randomized) treatment policy.
The VPI is defined as
\begin{equation}
\label{eq:vpi}
\theta^\ast = \max_{\pi \in \Pi}V(\pi) - V(q), \quad \text{where}\;\; V(\pi) := \E\left[\sum_{a \in \calA}\pi(a \mid X) Y(a)\right].
\end{equation}
A high-quality estimate of $\theta^\ast$ naturally informs the learner whether or not they should conduct policy learning. For instance, if the upper confidence bound for $\theta^\ast$ does not exceed the cost of developing/deploying a policy, then it is likely not worth a learner's efforts to build a personalized treatment regime. On the other hand, if the lower bound greatly exceeds costs, then constructing a personalized policy is worthwhile.
We note that the VPI fits into a longer line of estimands related to studying the effects of personalization, such as the \textit{value of personalization} considered in \citet{li2026statistical} and the population average prescriptive effect (PAPE) introduced by \citet{imai2023experimental}, which we discuss in Section~\ref{sec:intro:related}.

When there is a unique optimal treatment rule, performing inference on the VPI is relatively straightforward. In fact, in the setting of i.i.d.\ data, one can both learn and evaluate a policy on the entire dataset and still obtain asymptotically normal inference.\footnote{Of course, under the assumption that the policy class has well-behaved metric entropy. See \citet{carcamo2020directional} and \citet{fang2019inference} for a discussion of this.} However, in many empirical settings, the optimal policy or even best overall intervention may not be unique. For example, there may be  groups of patients who do not respond to treatments, and hence would experience identical outcomes under treatment and control.
When the optimal policy is non-unique,  \textit{the value of an optimal treatment policy} $V^\ast := \max_{\pi \in \Pi}\E\left[Y(\pi(X))\right]$ (the leading term in $\theta^\ast$) is \textit{non-regular}, or not pathwise differentiable in the semiparametric sense.
While a variety of estimators have been developed for tackling this non-regularity, these approaches either focus on unrestricted policy classes, yield non-Gaussian limiting distributions, or rely on conditions not established under adaptive logging with decaying propensities~\citep{luedtke2016statistical, whitehouse2025inference, fang2019inference, shi2022statistical, chernozhukov2013gaussian, chernozhukov2013intersection}.
We discuss these approaches in detail in Section~\ref{sec:intro:related} and Appendix~\ref{app:related}.
Notably, no existing approach yields general asymptotically normal inference when there is not a unique optimal policy and policy classes are restricted, even in the i.i.d.\ setting.
Further, existing proof techniques entirely break down when data are collected from a bandit policy with decaying propensities.

In sum, although the value of policy improvement is a naturally important quantity, it is clear that constructing estimators for it is non-trivial.
To develop such an estimator in the setting of the present paper, one needs to not only contend with decaying propensities, but also the general non-regularity of the estimand itself.
Given this subtlety, we make it our goal to construct a simple, asymptotically normal estimator for the value of policy improvement that remains valid even when data are collected from adaptive experimentation and optimal policies are non-unique.

\subsection{Our Contributions}

In this paper, we present an adaptive, inverse propensity weighted (IPW) estimator for the value of policy improvement given data logged from a contextual bandit algorithm. We briefly describe the estimator at this point to provide some intuition. Our estimator works by solving a weighted empirical risk minimization on all prefixes of data, down-weighting later samples proportionally to how aggressively the bandit logging policy exploits.  The estimator then evaluates each learned policy on the next sample through inverse propensity weighting, contrasting the value with that of the fixed baseline policy.
This structure builds on the online estimator of \citet{luedtke2016statistical}, extending it to substantially different settings. In particular, we allow for contextual bandit data with vanishing propensities (as opposed to i.i.d.\ data with strong positivity), study restricted policy classes (versus unrestricted ones), and provide a concrete weighted empirical welfare maximization procedure whose regret is sufficiently small to enable valid inference.

Throughout, we operate under the assumption that propensities are bounded below by $t^{-\alpha}$ for some $\alpha \in [0, 1)$, where $t$ is the current round of adaptive data collection. When $\alpha = 0$, this corresponds to assuming strong positivity, and as $\alpha$ tends towards one, the bandit more aggressively exploits what has been learned from the data up to the current point in time.
In detail, our primary contributions are as follows.

\begin{itemize}
    \item \textbf{An estimator for the value of policy improvement.} In Section~\ref{sec:estimator}, we formalize an
    estimator for the value of policy improvement. In contrast to prior work, our estimator can be used with  adaptively collected data and we do not assume optimal policy uniqueness. We prove the asymptotic normality of the estimator in Theorem~\ref{thm:clt},  describe a more computationally efficient batched version in Algorithm~\ref{alg:online-epoch}, and establish its convergence in Corollary~\ref{cor:epoch_estimator}. We emphasize that our estimator requires neither a margin condition on the blip effects nor uniqueness of the value-maximizing policy. To complement our estimator, we also prove the appropriate convergence of a natural predictable variance estimator. We note that our estimators (both for variance and the VPI) converge for $\alpha < 1$, whereas existing results focused instead on policy evaluation due to \citet{bibaut2021post}, \citet{zhan2021off}, and \citet{jia2024cramming} all require $\alpha \leq 1/2$ or stronger assumptions.
    \item \textbf{Fast rates under decaying propensities.} In Section~\ref{sec:regret}, we show that regret of a policy $\wh{\pi}$ obtained through weighted EWM satisfies
    \[
    \E\left[\Reg(\wh{\pi})\right] := \E\left[V^\ast - V(\wh{\pi})\right] = o\left(n^{(\alpha - 1)/2}\right)
    \]
    pointwise for every fixed data-generating process satisfying our assumptions. This is notably faster than (and does not contradict) the minimax optimal expected regret rates of $\Theta\left(n^{\frac{\alpha - 1}{2}}\right)$ provided by the current literature~\citep{zhan2024policy, bibaut2021risk}. Using this result, we can argue that the cumulative regret incurred from doing policy learning is \textit{lower order} than the width of our confidence intervals, thus enabling valid inference.
    While our results are inspired by fast regret rates due to \citet{luedtke2020performance} from the i.i.d.\ setting, our bounds hold \textit{in expectation} versus \textit{in probability}, with the former being needed to show convergence. Further, our results do not require assuming that the natural empirical process associated with the data converges to a tight, limiting Gaussian process. Instead, our results follow from a direct analysis that leverages asymptotic localization arguments.
    \item \textbf{Extensive empirical evaluation.}
    We also empirically evaluate our estimator on a mix of synthetic, semi-synthetic, and real-world data.
    On synthetic   and semi-synthetic data (the latter being generated from the Jester dataset), we analyze various statistics of our estimator, such as coverage and confidence interval width, relative to several natural baselines.
    We also empirically study the convergence rate of regret, confirming the faster regret rates shown in Theorem~\ref{thm:fast-regret} hold in practice.
    We find that our estimator performs quite well across both scenarios, typically achieving nominal coverage while offering narrower confidence intervals than other baseline estimators.
    Lastly, we revisit an adaptive experiment conducted by \citet{caria2024adaptive}, which was aimed at assessing the effectiveness of various interventions designed to help refugees in Jordan obtain employment.
    As part of their work, the authors measure the effect of personalizing interventions using a Bayesian analogue of the VPI. Notably, their work assumes a unique optimal treatment policy.
    We reanalyze the data using our frequentist, assumption-light estimator.
    While we generally obtain more conservative estimates across the board than \citet{caria2024adaptive}, we do confirm the authors' finding that employment outcomes are better under personalization  amongst the Syrian subpopulation in Jordan.

\end{itemize}

Our estimator is inspired by an earlier, online estimator due to \citet{luedtke2016statistical}. We note several important distinctions from the aforementioned work. First, our estimator operates in the far more flexible setting of contextual bandit data instead of imposing strong, i.i.d.\ assumptions. Second, we focus our attention on \textit{restricted policy classes}, versus unrestricted settings where any measurable  policy is allowed.
Finally, instead of assuming access to a black-box policy learner that attains the target rates, we provide a \textit{precise, weighted welfare maximization procedure} that enjoys the necessary regret rates.

We close the introduction by noting what we believe to be the  most important technical insight of the paper: the particular choice of estimand is critical in enabling valid inference. An alternative estimand more closely aligned with the aforementioned literature on nonregular statistical inference would be the value of the optimal treatment policy $V^\ast := \max_{\pi \in \Pi}V(\pi)$. While we believe the contrast $\theta^\ast$ is actually a more interpretable estimand (because it gives a value relative to some baseline), the contrast with a policy also serves to stabilize the variance of the IPW estimates for the values of the sequentially learned policies. This stabilization aids our analysis, not only allowing us to show that the accumulated regret from policy learning is negligible relative to the overall width of our confidence intervals, but also giving our estimator the flexibility to operate under any logging exponent $\alpha \in [0, 1)$.

\section{Related Work}
\label{sec:intro:related}

Our work is situated at the intersection of three technical bodies of literature. In particular, our analysis is built upon results on inference on optimal treatment regimes (i.e., the value of the optimal treatment policy), off policy evaluation, and policy learning. We provide a much more extensive related work section in Appendix~\ref{app:related}, and opt to just focus on the most closely related results here. We also briefly note the connection to the literature on personalization at the end of this section. A summary of our comparison to closely related results can be found in Table~\ref{tab:related-work-terse}, with a more thorough comparison being presented in Table~\ref{tab:related-work} in Appendix~\ref{app:related}.
Below, we focus our discussion primarily on the results provided by the cited papers. We postpone the detailed comparison of proof techniques for relevant sections of the paper/appendices.

\begin{table}[H]
\centering
\caption{A comparison of our results with the most closely related papers. We let $V^\ast$ denote the value of an optimal policy within a class $\Pi$ and $V^\circ$ the optimal value of \textit{any} policy. An extended comparison is shown in Table~\ref{tab:related-work} in Appendix~\ref{app:related}}
\label{tab:related-work-terse}
\small
\setlength{\tabcolsep}{3pt}
\renewcommand{\arraystretch}{1.15}

\begin{tabularx}{\textwidth}{
    @{}
    >{\raggedright\arraybackslash}p{0.155\textwidth}
    >{\raggedright\arraybackslash}p{0.145\textwidth}
    >{\raggedright\arraybackslash}p{0.17\textwidth}
    >{\raggedright\arraybackslash}p{0.205\textwidth}
    X
    @{}
}
\toprule
Work
& Setting
& Target
& Result
& Key assumptions
\\
\midrule

\citet{luedtke2016statistical}
& i.i.d.\ data
& Unrestricted optimal value $V^\circ$
& Asymptotically normal estimator
& Propensity and regression nuisance convergence rates.
\\
\midrule

\citet{zhan2021off}
& Contextual bandit data
& Fixed-policy value $V(\pi)$
& Asymptotically normal estimator
& $\alpha<1/2$
\\
\midrule

\citet{bibaut2021post}
& Contextual bandit data
& Fixed-policy value $V(\pi)$
& Asymptotically normal estimator
& $\alpha\leq1/2$ ($\alpha \leq 1/3$ for variance convergence)
\\
\midrule

\textbf{This paper}
& Contextual bandit data
& Best in-class contrast $V^\ast -V(q)$
& Asymptotically normal estimator and expected regret
$o\!\left(n^{(\alpha -1)/2}\right)$
& Finite disagreement VC dimension; $\alpha<1$
\\

\bottomrule
\end{tabularx}
\end{table}

\paragraph{Inference on Optimal Treatment Regimes}
There is a rich literature focused on estimating the value of an optimal treatment policy within a class. These works can be roughly split into two buckets: approaches that assume \textit{unrestricted policy classes}\footnote{Meaning that a treatment policy can be \textit{any} measurable map from $\calX$ to a distribution over $\calA$.} and those that assume \textit{restricted classes} (e.g., VC classes or Donsker classes). There are many distinct estimation strategies for the first bucket: some papers involve repeatedly subsampling/refitting nuisances~\citep{shi2020breaking}, explicitly constructing smoothed surrogates for the target estimand~\citep{whitehouse2025inference,  goldberg2014comment}, using bootstrapping methods to handle irregular limits~\citep{chakraborty2010inference, laber2014dynamic}, or assuming almost sure uniqueness of the optimal policy~\citep{chambaz2017targeted, van2015targeted}.
In the setting of restricted, well-behaved policy classes (e.g., Donsker or VC-like classes), authors have characterized the behavior of the ``double-dipping'' estimator that uses the full set of data to both learn and evaluate a policy.
If there is a unique optimal treatment policy, the results of \citet{carcamo2020directional} and \citet{fang2019inference} guarantee that double-dipping yields asymptotically normal estimates. Otherwise, the limiting distribution is generally the supremum of a limiting Gaussian process indexed by the class of optimal policies, for which authors have constructed directional and multiplier bootstrap methods for estimating quantiles~\citep{fang2019inference, chernozhukov2013gaussian, chernozhukov2013intersection,chernozhukov2023high,ponomarev2024lower}.  The aforementioned \textit{unrestricted estimators} all make a \textit{margin assumption} on the underlying conditional average treatment effects (blip effects in the multi-treatment settings)~\citep{audibert2007fast}, while the \textit{restricted estimators} do not. Related work on the estimator of the value of personalization by \citet{li2026statistical} similarly assumes a unique optimal policy in the class for its semi-parametric efficiency results. In addition, all of the aforementioned papers (modulo \citet{chambaz2017targeted}) only yield valid inference under i.i.d.\ data, and none allows for decaying treatment propensities.
One further work that permits adaptively collected data is \citet{shen2024doubly}, who restrict their attention to linear models and assume a margin condition that implies that there is zero probability of units not responding to treatment.
We also note that none of the results from the restricted setting yields general, asymptotically normal estimators even in the i.i.d.\ setting, which will be an additional contribution of this paper.

The most closely related paper to the present work is an online/sequential estimator of \citet{luedtke2016statistical}, which works around nonregularity by progressively refitting nuisances (an outcome regression and propensity score) on prefixes of data. It then evaluates the policy associated with learned regression on subsequent samples using an augmented IPW off policy evaluation score.
While other estimators (e.g., that of \citet{shi2020breaking}) also use repeated nuisance refitting, this online estimator can get away with significantly fewer nuisance refits: with $\log(n)$ refits required under standard nuisance $L^2$ estimation rates and only a \textit{constant number of fits} if one can obtain mildly (say, a factor of $\log(n)$) faster rates.\footnote{Albeit the first $n/\log(n)$ samples will be lost to nuisance estimation.} While in principle this estimator permits the evaluation of an arbitrary, predictable policy on the $t$th sample, the expected regret of the policy must shrink as $o(n^{-1/2})$.\footnote{The initial publication of this paper stated $o_\P(n^{-1/2})$ regret rates sufficed, but was later corrected by the authors in~\citep{luedtke2021correction}.}
Existing regret analyses (due to \citet{luedtke2020performance}) yield regret that vanishes as $o_\P(n^{-1/2})$, which is notably insufficient for the estimator to yield valid coverage~\citep{bibaut2020sufficient}.
In this paper, we study an analogue of this estimator in the setting of contextual bandit data with vanishing propensities, showing convergence if the expected regret of the learned policy vanishes at $o(n^{(\alpha - 1)/2})$ rates. Unlike the estimator of \citet{luedtke2016statistical}, instead of assuming access to an abstract, black-box learner, our estimator includes a specific weighted welfare maximization objective whose solution satisfies the requisite regret rates (established in Theorem~\ref{thm:fast-regret}).
Thus, even with i.i.d.\ data, our analysis closes the loop on the validity of the sequential estimator, showing it can be used to obtain valid inference under restricted policy classes.

\paragraph{Off-Policy Evaluation}

We only discuss the most closely related works on OPE of a fixed policy in this section, deferring our discussion of more classical results from the i.i.d.\ or reinforcement learning settings for the appendices. One closely related result is due to \citet{hadad2021confidence}, who discuss inference on the value of a fixed policy under (non-contextual) bandit data. \citet{zhan2021off} extend these results to the contextual setting, discussing various predictable, contextual weights one can use to ensure asymptotic normality.
\citet{bibaut2021post} discuss a similar estimator for OPE (or more broadly, treatment effect estimation), which is closely related to the aforementioned estimator of \citet{luedtke2016statistical}. Unfortunately, this work only considers inference on the value of a \textit{fixed function}, and assumes the logging policy does not exploit too aggressively, requiring $\alpha \leq 1/2$ for convergence to hold (and $\alpha \leq 1/3$ to obtain a sufficiently quickly converging variance estimator). Lastly, we note high-probability bounds due to \citet{karampatziakis2021off} and \citet{waudby2024anytime}, who provide time-uniform, non-asymptotic bounds on the value of a treatment policy. The analyses of the estimators in these works, like the analyses conducted in this paper, all depend on concentration properties for martingales, whether they be martingale central limit theorems~\citep{brown1971martingale, hall2014martingale, ganssler1978central} for the former three estimators, or finite sample bounds based on Ville's inequality~\citep{howard2020time, howard2021time} for the latter two. We note that none of these works focuses on inference on the value of the optimal treatment policy.
We also mention the ``Cramming'' estimator of \citet{jia2024cramming}, which considers the related task of \textit{on-policy evaluation} instead of off-policy evaluation. This work provides inference for the value of the final, adaptively-chosen logging policy itself under a condition that the chosen policy stabilizes over time. This work may fail to give valid inference if the logging exponent satisfies $\alpha \geq 1/2$, and may also fail if the logging policy doesn't converge.\footnote{This may occur in practice, for instance, if the logging policy is chosen according to an empirical welfare maximization problem and there are multiple optimal policies.} On the other hand, our work places much weaker assumptions on $\alpha$, simply requiring $\alpha < 1$, and allows \textit{post facto} inference on the optimal value in some arbitrary class, not necessarily the policy class that was used for logging.

\paragraph{Policy Learning From Observational Data}

\citet{bibaut2021risk} and \citet{zhan2024policy} discuss how a weighted variant of empirical welfare maximization can be used to obtain $\Theta(n^{(\alpha - 1)/2})$ regret rates, respectively with high probability and in expectation, when the propensity of the logging bandit algorithm decays as $t^{-\alpha}$. \citet{zhan2024policy} also show that such expected regret bounds are \textit{minimax optimal}.
These two works operate under slightly different assumptions on the policy and logging policy classes, with the former assuming a bracketing entropy condition and the latter considering finite Natarajan dimension (the analogue of VC-dimension for multiclass settings, see \citet{natarajan1989learning}). We make an assumption equivalent to the latter throughout our work, and thus leverage similar chaining arguments to \citet{zhan2024policy} in our analysis. We also note recent work due to \citet{girard2025online}, which shows fast, high-probability regret bounds when a H\"olderian variance condition holds for the policy class. However, this paper requires stochastic policies, and the curvature assumption implies an \textit{almost surely unique optimal policy}. We make no such curvature or uniqueness assumptions in this paper. Our primary policy learning contribution is that we show the \textit{expected regret} vanishes at a rate of $o(n^{(\alpha - 1)/2})$, which can be viewed as infinitesimally faster than the rates noted above.

Another relevant paper to our regret rates is \citet{luedtke2020performance}. In this work, the authors show that, for Donsker classes of treatment policies and i.i.d.\ data, the regret obtained from performing empirical welfare maximization scales as $o_\P(n^{-1/2})$, slightly faster than the  $\Theta(n^{-1/2})$ expected rates suggested by \citet{athey2021policy} and \citet{kitagawa2018should}. These results are not shown in a minimax framework, but rather in a pointwise sense, with the unknown constants depending inherently on the underlying data generating distribution.
More broadly, the authors show similar fast regret rates can hold under broader settings, so long as the empirical process associated with the gathered data and policy class converges to a tight Gaussian limit. We note two key differences between the aforementioned results and our own. First, our fast regret rates do not require asymptotic convergence of the empirical process to a Gaussian limit. Instead, we prove our results by leveraging finite sample chaining arguments, which can then be asymptotically localized to obtain the fast rates we desire.
Second, we show stronger \textit{in expectation} bounds rather than only the \textit{in probability} convergence bounds shown by \citet{luedtke2020performance}. Per results of \citet{bibaut2020sufficient}, the latter results are insufficient for the convergence of the online estimator.

\paragraph{Measuring the Effect of Personalization}

We conclude by briefly discussing the most closely related works on studying the effects of personalization. First, the VPI is similar to the \textit{value of personalization} proposed in recent work by~\citet{li2026statistical}, which is defined as  $\max_{\pi \in \Pi}V(\pi) - \max_a V(a)$.\footnote{Here, we abuse notation and let $V(a)$ denote the value of the constant treatment policy that assigns all units to $a \in \calA$.} We note that the setting considered in said paper is quite different from our own, as the authors assume i.i.d.\ data and uniqueness of the optimal treatment policy. In their setting, they present a semiparametrically efficient estimator for the value of personalization.
Also related to our work is that of \citet{imai2023experimental}, who define the \textit{population average prescriptive effect (PAPE)}, which evaluates a given individualized treatment rule relative to a non-personalized randomized rule assigning the same fraction of units to treatment. They develop randomization-based inference for both fixed rules and those learned via sample splitting and cross-validation. In particular, we note that the estimand does not involve a maximization over a class of policies, which is the fundamental technical difficulty addressed in this work.

\section{Background and Notation}
\label{sec:background}
In this section, we introduce the notation and assumptions that we will leverage throughout the sequel. Our setup largely follows that of works on off-policy evaluation \citep{bibaut2021post, zhan2021off} and policy learning~\citep{zhan2024policy} under contextual bandit data.
We first introduce assumptions needed for defining and identifying the value of policy improvement and then discuss assumptions made to enable estimation.

\paragraph{Assumptions for Identification}

Throughout, we represent each unit by a tuple $W_t = \big(X_t, \left\{Y_t(a) : a \in \calA\right\}\big)$, where $X_t \in \calX$ represents covariates, $\calA$ denotes a finite treatment set, and $Y_t(a)$ represents a unit's potential outcome associated with treatment $a \in \calA$. At round $t$, we assume the learner obtains a \textit{partial observation} of the form $Z_t = (X_t, A_t, Y_t)$, where $A_t \in \calA$ denotes a realized treatment assignment. We let $(\calF_n)_{n  \geq 0}$ denote the natural filtration associated with the data, i.e., $\calF_n = \sigma(Z_t : 1 \leq t \leq n)$.

\begin{ass}[Identifying Assumptions]
\label{ass:observations}
We assume that the following hold:
\begin{enumerate}
\item \textit{Stationarity.} There is some distribution $P_W$ over tuples $W = \big(X, \{Y(a) : a \in \calA\}\big)$ such that $W_1, \dots, W_n \sim P_W$ are i.i.d.
\item \textit{Predictable Assignments.} There is a known sequence of transition kernels $p_t(a \mid x, \calF_{t - 1})$ such that, conditional on $X_t$ and $\calF_{t - 1}$, $A_t \sim p_t(\,\,\cdot \mid X_t, \calF_{t - 1})$ independently of $\{Y_t(a) : a \in \calA\}$ and future observations $W_{t + 1}, \dots, W_n$, and hence
\[
\P(A_t = a \mid X_t, \calF_{t - 1}) = p_t(a \mid X_t, \calF_{t - 1})\;\; \text{almost surely.}
\]
\item \textit{Consistency.}  $Y_t = Y_t(A_t)$ for all $t \geq 1$.
\item \textit{Positivity.} For each $t \geq 1$ and $a \in \calA$, we have $p_t(a \mid X_t, \calF_{t - 1}) > 0$ almost surely.
\end{enumerate}
\end{ass}

Note that the first two parts of Assumption~\ref{ass:observations} imply conditional exogeneity for each $t \geq 1$:
\begin{equation}
\label{eq:conditional_exogeneity}
\left\{Y_t(a) : a \in \calA\right\} \independent A_t \mid X_t, \calF_{t - 1}.
\end{equation}
In the sequel, we write $p_t(a \mid x) := p_t(a \mid x, \calF_{t - 1})$ for  convenience. The above assumptions are largely standard and match the aforementioned works.

By a treatment policy, we mean either a \textit{deterministic} policy $\pi : \calX \rightarrow \calA$ or a \textit{randomized} policy $\pi : \calA \times \calX \rightarrow [0, 1]$, where $\pi(a \mid x)$ denotes the probability that a unit with features $x$ receives treatment $a$.\footnote{More formally, a randomized policy is actually a mapping $\pi : \calX \rightarrow \Delta(\calA)$, where $\Delta(\calA)$ is the probability simplex on the set of treatments. Here, we would let $\pi(a \mid x)$ just denote the $a$th component of this vector.} We can view a deterministic policy $\pi(x)$ as a randomized one using the notation $\pi(a \mid x) := \1\{\pi(x)  = a\}$. Given any function $\pi : \calA \times \calX \rightarrow \R$ (not necessarily a treatment policy), we can define its value as $V(\pi) := \E\left[\sum_{a \in \calA} \pi(a \mid X)Y(a)\right]$.
Under Assumption~\ref{ass:observations}, it is well understood that the value of any such function $\pi$ can be identified in two ways:
\[
V(\pi) \equiv \E\left[\sum_{a \in \calA}\pi(a \mid X_t)\wt{Y}_t(a)\right] \equiv \E\left[\sum_{a \in \calA}\pi(a \mid X) \mu_0(a, X)\right],
\]
where $\wt{Y}_t(a) := \frac{\1\{A_t = a\}}{p_t(a \mid X_t)}Y_t$ is an inverse propensity weighted (IPW) pseudo-outcome for treatment $a \in \calA$ and $\mu_0(a, x) := \E\left(Y(a) \mid  X = x\right)$ is the outcome regression (which does not depend on $t$ due to stationarity).

Throughout, we let $\Pi \subset \big\{\pi : \calX \rightarrow \calA\big\}$ denote a fixed set of deterministic treatment policies, which we assume to be \textit{closed with respect to the disagreement metric} $d(\pi, \rho) := \P\big(\pi(X) \neq \rho(X)\big)^{1/2}$. Closedness is useful as, along with the complexity assumptions posited below, it ensures the value of an optimal treatment policy $V^\ast := \max_{\pi \in \Pi}V(\pi)$ is attained. We let $\Pi^\ast := \big\{ \pi \in \Pi : V(\pi) = V^\ast\big\} \subset \Pi$ denote the set of optimal policies, which may contain multiple (or even infinitely many) elements.\footnote{One could relax the closedness assumption (as done in \citet{luedtke2020performance}, for instance) at the cost of greatly complicating analysis.}
Given $\Pi$ and any function $q : \calA \times \calX \rightarrow [0, 1]$ (which may or may not be a treatment policy), we can formally define the \textit{value of policy improvement} as
\[
\theta^\ast := \max_{\pi \in \Pi}V(\pi) - V(q).
\]
Since the value $V(\pi)$ of any function $\pi$ is identified, it naturally follows that the value of policy improvement $\theta^\ast$ is identified as well.

\paragraph{Assumptions for Estimation}

We now present assumptions on the treatment policies, logging policies, and observations that will enable estimation. We start by qualifying the complexity of the deterministic class of treatment policies $\Pi$.

\begin{ass}
\label{ass:disagreement}
We assume the \textit{disagreement class} $\calD := \Big\{g_{\pi, \rho}(x) := \1\left\{ \pi(x) \neq \rho(x)\right\} : \pi, \rho \in \Pi \Big\}$ has VC-dimension $\frakv$ for some $\frakv > 0$.
\end{ass}

Under the assumption that $|\calA| < \infty$, one can show that Assumption~\ref{ass:disagreement} is equivalent to positing that $\Pi$ has finite \textit{Natarajan dimension}.\footnote{We formally show this equivalence in Appendix~\ref{app:combinatorial-dimension}, and prefer to leverage VC dimension in the main body of the text as it is a more widely studied concept.}
This complexity assumption is not crucial to our central limit theorem, but rather to obtaining faster regret rates.
We note that \citet{zhan2024policy} make an equivalent assumption in their work (directly assuming finite Natarajan dimension), and \citet{bibaut2021risk} consider a similar condition based on bracketing entropy, which in general is neither stronger nor weaker than assuming finite VC dimension.

The next assumption imposes additional regularity on the potential outcomes.

\begin{ass}
\label{ass:estimation}
We assume that $Y(a) \in [0, 1]$ almost surely and that $ \Var(Y(a) \mid X) \geq \lb{\sigma}^2$ almost surely for some $\lb{\sigma} > 0$ and all $a \in \calA$.
\end{ass}

The above  assumptions are not likely minimal (e.g., boundedness of outcomes can likely be replaced with unbounded outcomes with sufficient moment control), but relaxing said assumptions would complicate analysis without adding much insight.
This assumption is in line with the broader literature on bandit optimization~\citep{lattimore2020bandit, slivkins2019introduction} and empirical welfare maximization~\citep{zhan2024policy, bibaut2021risk, kitagawa2018should, athey2021policy}, and boundedness enables the application of Freedman-style concentration inequalities for bounded random variables~\citep{freedman1975tail}.
The second assumption ensures that the variance of all potential outcomes is conditionally bounded below, which ensures that all units (regardless of strata) exhibit some variability in terms of their responses.

Next, we discuss assumptions on the collection of logging policies. We start by presenting assumptions on the rate of exploration of the bandit logging policy, which are akin to assumptions made in \citet{bibaut2021post} and \citet{zhan2021off, zhan2024policy}.

\begin{ass}
\label{ass:propensity}
We assume there are constants $\alpha \in [0, 1)$ and $\lb{c}, \wb{c} > 0$ such that, for all $t \geq 1$, we have
\[
\lb{c} t^{-\alpha} \leq \min_{a \in \calA} p_t(a \mid X_t) \leq  \wb{c} t^{-\alpha} \text{ almost surely.}
\]
\end{ass}

When $\alpha = 0$, the above is equivalent to the assumption of strong positivity, with larger $\alpha$ indicating increased exploitation. Examples of logging policies that can be shown to satisfy this assumption include contextual $\epsilon_t$-greedy algorithms. Static (non-adaptive) logging policies that randomly assign actions to contexts also immediately satisfy this assumption, as long as all actions have a positive probability of assignment in all contexts.   Examples of algorithms that do not satisfy this assumption are algorithms that behave deterministically or have exponentially vanishing probabilities, such as upper confidence bound (UCB) or Thompson sampling algorithms. However, these approaches can be made to satisfy the propensity decay via clamping probabilities to a minimum floor~\cite{yao2021power} or adding a polynomially-decaying probability of uniform exploration~\citep{caria2024adaptive} ("tempering"). Notably, these algorithms will maintain their $O(\sqrt{n})$ minimax regret rates so long as the tempering probability scales like $t^{-\alpha}$ for some $\alpha \in [1/2, 1)$.

We conclude by positing an assumption on the rate of growth of the predictable variance of IPW pseudo-outcomes.
For any $\pi \in \Pi$, define the contrast $g_\pi(a \mid x) := \1\{\pi(x) = a\} - q(a \mid x)$ and define the IPW pseudo-outcome for the contrast between $\pi$ and $q$ as
\begin{equation}
\label{eq:double-pseudo-outcome}
\wt{Y}_t(\pi, q) := \frac{g_\pi(A_t \mid X_t)Y_t}{p_t(A_t \mid X_t)}.
\end{equation}
We define the predictable variance of said contrast as
\begin{equation}
\label{eq:pred-var}
\sigma_t^2(\pi) := \Var\left(\wt{Y}_t(\pi, q) \mid \calF_{t - 1}\right).
\end{equation}
We make the following assumption.
\begin{ass}
\label{ass:ipw_variance_scale}
There is some constant $c_\sigma > 0$ such that, for any $\pi \in \Pi$ and $t \geq 1$, we have $\sigma_t^2(\pi) \geq c_\sigma t^{\alpha}$.
\end{ass}
There are at least three natural settings in which the above assumption would hold, which we note in the following lemma.
\begin{lemma}
\label{lem:suff_conditions}
Suppose Assumptions~\ref{ass:observations}, \ref{ass:estimation}, and \ref{ass:propensity} hold. Then, if one of the following holds, Assumption~\ref{ass:ipw_variance_scale} holds as well.
\begin{enumerate}
    \item \textit{Strictly Randomized Baseline.} There is a constant $\eta_q > 0$  such that $\inf_{\substack{a \in \calA \\ x \in \calX}}q(a \mid x) \geq \eta_q$.
    \item \textit{Deterministic Baseline.} $\calA = \{0, 1\}$, $ q(a \mid x) = \1\{q(x) = a\}$ for some deterministic baseline $q : \calX \rightarrow \{0, 1\}$, and $\eta_q := \inf_{\pi \in \Pi}\P(\pi(X) \neq q(X)) > 0$.
    \item \textit{Non-Vanishing Exploration.} $\alpha = 0$, $q(a \mid x)$ is any function\footnote{Not necessarily a treatment policy.} taking values in $[0, 1]$, and there is some $\eta_g > 0$ satisfying
    \[
    \E\left(g_\pi(A_t \mid X_t)^2 \mid \calF_{t - 1}\right) \geq \eta_g \;\; \forall \pi \in \Pi, t \geq 1.
    \]
\end{enumerate}
\end{lemma}

We prove Lemma~\ref{lem:suff_conditions} in Appendix~\ref{app:additional}. The result follows largely from an application of the law of total variance. We now interpret each of the three aforementioned sufficient conditions. The first condition simply states that Assumption~\ref{ass:ipw_variance_scale} holds if the baseline policy is strictly randomized. An important example satisfying this condition is uniform randomization.
The next sufficient condition holds in the setting of binary treatments with deterministic baselines (when the policy class excludes the deterministic baseline), and as a special case corresponds to measuring the value of an optimal policy relative to the policy that does not change the status quo, i.e., assigns all units to control.\footnote{The  assumption of $\Pi$ excluding the baseline $q$ could be weakened to allowing $q \in \Pi$ so long as $V(q) < V^\ast := \max_{\pi \in \Pi}V(\pi)$. However, this latter assumption is not verifiable. Given that $\Pi$ is chosen by the learner, we opt to just assume $q \notin \Pi$. }
Lastly, if we take $q(a \mid x) \equiv 0$ in the final setting, the conditional second moment bound trivially holds by the lower bound provided by Assumption~\ref{ass:propensity}. Thus, since $V(q) \equiv 0$, the VPI in the final setting reduces to the value of the optimal policy $V^\ast$.

In summary, Assumptions~\ref{ass:observations} and \ref{ass:estimation} are mild and common assumptions about the data generating process,  and Assumption~\ref{ass:disagreement} simply states the disagreement class of the policy class considered has bounded VC dimension. Assumption~\ref{ass:propensity} is a stronger assumption on the data collection policy, that rules out certain popular algorithms but still includes  other existing adaptive bandit algorithms (that often have been created to balance online adaptive regret minimization while preserving the data required for post collection inference). The final Assumption~\ref{ass:ipw_variance_scale} is less standard, and, as we lay out in Lemma ~\ref{lem:suff_conditions}, intuitively lays out the scenarios in which we provide policy improvement estimates, in terms of the underlying characteristics of the baseline policy.

\section{An Estimator for the Value of Policy Improvement}
\label{sec:estimator}

We now present our estimator for the value of policy improvement. We first describe our estimator, which is based on an online estimator due to \citet{luedtke2016statistical}. Then, we present our primary asymptotic normality result alongside a heuristic proof sketch, deferring the full argument to Appendix~\ref{app:clt_proof}. While our main theorem assumes one has access to black-box estimates of the  predictable variance of the IPW pseudo-outcomes, we provide a theorem (Theorem~\ref{thm:variance-convergence}) in Subsection~\ref{sec:estimator:variance} that shows how one can build estimates of the predictable variance under a mild complexity assumption on the logging policy class. The rates we require on the variance estimate are much more lenient than those needed in \citet{bibaut2021post}, and as such we can show convergence under a wider range of logging exponents ($\alpha \in [0, 1)$ in our case versus $\alpha  < 1/3$). We conclude by discussing computational considerations, presenting a more computationally efficient version of our estimator that uses geometric batching.

We start by describing our estimator, first in words. For each sample size $1 \leq t \leq n$, we use the first $t - 1$ samples to compute a weighted empirical welfare maximizing policy, where weights are determined according to the rate of decay of the propensity of the logging policy. We then perform IPW off-policy evaluation for the learned policy on the $t$th sample, and take the difference between this and the IPW estimate of the value of $q(a \mid x)$.
Inspired by \citet{luedtke2016statistical} and \citet{bibaut2021post}, we then take a combination of these evaluations weighted by inverse predictable standard deviation estimates.

We now describe this estimator more formally. First, for $1 \leq t \leq n$, we define the weighted empirical welfare maximizing (EWM) policy\footnote{For $t = 1$, we can take the first estimated policy $\wh{\pi}_0$ to be an arbitrary element of $\Pi$.} using an objective outlined in \citet{zhan2024policy}:
\begin{equation}
\label{eq:ewm}
\wh{\pi}_{t - 1} \in \arg\max_{\pi \in \Pi} \sum_{s = 1}^{t - 1} w_s \wt{Y}_s(\pi) = \arg\max_{\pi \in \Pi} \sum_{s = 1}^{t - 1} w_s \sum_{a \in \calA}\pi(a \mid X_s)\wt{Y}_s(a),
\end{equation}
where we take $w_s = s^{-\alpha}$ and recall that $\wt{Y}_t(a) = \1\{A_t = a\}Y_t/p_t(a \mid X_t)$ and $\wt{Y}_t(\pi) = \pi(A_t \mid X_t)Y_t/p_t(A_t \mid X_t)$.
Our estimator is defined to be
\begin{equation}
\label{eq:estimator}
\wh{\theta}_n := \frac{\sum_{t = 1}^n \wh{\sigma}_t^{-1}\wt{Y}_t\big(\wh{\pi}_{t - 1}, q\big)}{\sum_{t = 1}^n \wh{\sigma}_t^{-1}},
\end{equation}
where for each $t \geq 1$, the quantity $\wh{\sigma}_t^2$ represents an estimate of the predictable variance $\sigma_t^2$ (defined in Equation~\eqref{eq:pred-var}) and $\wh{Y}_t(\pi, q)$ is as defined in Equation~\eqref{eq:double-pseudo-outcome}.

Before presenting our main theorem, we provide some intuition as to what the estimator is accomplishing. First, we address the policy learning subroutine. Given the data collected up to time $t$, a reasonable estimate of the optimal policy would be to compute the best-in-class policy that maximizes the empirical welfare. However, when the logging exponent satisfies $\alpha > 0$, later inverse propensity weighted samples exhibit inherently higher variance than earlier ones. To account for this gradual increase in variance as samples are collected, we adopt an online policy learning approach of~\citet{zhan2024policy}, who show that  down-weighting later observations according to the logging exponent yields minimax optimal regret.

The next step is determining how one should leverage an algorithm for learning high-quality policies to perform inference on the VPI. As proposed previously by~\citet{luedtke2016statistical} (who consider the setting of i.i.d.\ data and unrestricted policy classes), a reasonable approach is to construct a policy on the $t-1$ samples and then perform off-policy evaluation (OPE) on the $t$th sample to estimate its value for each $t \in \{1, 2, \dots, n\}$. The authors then show that these pointwise OPE estimates can be carefully averaged in order to enable valid inference. The ``right'' choice of weights for this average are inverse predictable standard deviation weights, as classical martingale central limit theorems~\citep{hall2014martingale,ganssler1978central,  brown1971martingale} require summands to have roughly unit predictable variance.
While \citet{luedtke2016statistical} consider AIPW estimates, for simplicity, and because our focus is experimental domains where the propensity is known, our estimator uses IPW estimates.\footnote{We note that our results could naturally be extended to AIPW estimates as well. However, to guarantee convergence of the variance estimate, we would need to restrict the class of functions containing the regression model $\wh{\mu}(a, x)$ to be of limited complexity, e.g., a class of finite pseudo-dimension. We choose to study IPW estimators for simplicity's sake.}

While the above paragraphs attempt to build intuition for the estimator, we note that showing convergence is still nontrivial. This is primarily driven by two factors. First, one cannot apply off-the-shelf high-probability regret bounds for weighted EWM, such as those due to \citet{zhan2024policy} and \citet{bibaut2021risk}. These works show that the regret of EWM decays, either in expectation or with high probability, as $O\left(n^{(\alpha - 1)/2}\right)$. Unfortunately, the standard error of our estimator (and, consequently, the width of our confidence intervals) scales at the same $n^{(\alpha - 1)/2}$ rate, which thus bars asymptotic normality. While the stated regret rates are \textit{minimax optimal}, we will see in Section~\ref{sec:regret} that performing weighted EWM actually yields marginally faster expected regret rates of $o(n^{(\alpha - 1)/2})$ \textit{pointwise} in the data generating distribution, which enables valid inference.
Second, constructing sufficiently consistent estimates of the predictable variance that satisfy the consistency assumption of our  following theorem is not trivial. In fact, in existing works on OPE under bandit data, \citet{bibaut2021post} require $\alpha < 1/3$ in order to obtain sufficiently consistent estimates, which is significantly more stringent than the setting considered in the present work. We discuss the difficulties of variance estimation below in Section~\ref{sec:estimator:variance}.

With these subtleties in mind, we now present our primary result, which details the asymptotic normality of a studentized version of the estimator $\wh{\theta}_n$. In what follows, we let $\theta(\pi) := V(\pi) - V(q)$ denote the gap between the value of policy $\pi$ and the baseline $q$. Note that $\theta^\ast = \theta(\pi^\ast)$ for any $\pi^\ast \in \Pi^\ast$.

\begin{theorem}
\label{thm:clt}
Suppose $(Z_n)_{n \geq 1}$ satisfies Assumptions~\ref{ass:observations}, \ref{ass:estimation}, \ref{ass:propensity}, and \ref{ass:ipw_variance_scale}, and that $\Pi$ satisfies Assumption~\ref{ass:disagreement} and is closed. Further, suppose that $\wh{\sigma}_t$ is $\calF_{t - 1}$-measurable and satisfies
\[
\frac{\wh{\sigma}_t - \sigma_t}{\sigma_t} \xrightarrow[t \rightarrow \infty]{a.s.} 0 \qquad \text{and} \qquad \wh{\sigma}_t \geq \kappa_t > 0, \forall t \geq 1,
\]
where $(\kappa_t)_{t \geq 1}$ is some sequence tending towards zero.
Then, we have $\frac{\sum_{t= 1}^n\wh{\sigma}_t^{-1}}{\sqrt{n}}\left\{\wh{\theta}_n - \theta^\ast\right\} \Rightarrow \calN(0, 1)$. Thus, if $\gamma \in (0, 1)$ and $z_{1 - \gamma/2}$ is the $(1 - \gamma/2)$th quantile of a standard normal distribution, the interval
\[
C_{1 - \gamma} := \left[\wh{\theta}_n - \frac{\sqrt{n}}{\sum_{t  = 1}^n \wh{\sigma}_t^{-1}}z_{1 - \gamma/2}, \wh{\theta}_n + \frac{\sqrt{n}}{\sum_{t  = 1}^n \wh{\sigma}_t^{-1}}z_{1 - \gamma/2}\right]
\]
forms an asymptotically valid $1 - \gamma$ confidence interval for  $\theta^\ast$.
\end{theorem}

To build intuition for the above result, we provide a high-level proof sketch. We present the full proof in Appendix~\ref{app:clt_proof} --- the sketch is just aimed at capturing the main arguments.

\begin{proof}[Proof Sketch for Theorem~\ref{thm:clt}]
The key to our proof is the following decomposition, which splits the difference between our estimate and the true value of policy improvement into two terms:
\begin{align}
\label{eq:clt_decomp}
\frac{\sum_{t = 1}^n \wh{\sigma}_t^{-1}}{\sqrt{n}}\left\{\wh{\theta}_n - \theta^\ast\right\} &=  \underbrace{\frac{1}{\sqrt{n}}\sum_{t = 1}^n \wh{\sigma}_t^{-1}\left\{\wt{Y}_t\big(\wh{\pi}_{t - 1}, q\big) - \theta\big(\wh{\pi}_{t - 1}\big)\right\}}_{\text{Martingale CLT Term}} \\
&\qquad + \underbrace{\frac{1}{\sqrt{n}}\sum_{t = 1}^n \wh{\sigma}_t^{-1}\left\{\theta\big(\wh{\pi}_{t - 1}\big) - \theta^\ast\right\}}_{\text{Cumulative Regret Term}}.\nonumber
\end{align}

For any fixed sample size $n$, the first term is a martingale under the natural filtration whose terms roughly have $\frac{1}{n}$ predictable variance. Under a martingale central limit theorem due to \citet{brown1971martingale} (Theorem~\ref{thm:martingale-clt}), one can show asymptotic normality under our variance convergence assumptions.

Next, the second, cumulative regret term represents the bias accumulated from performing OPE not on an optimal policy, but rather on a slightly suboptimal policy learned using weighted EWM. We can make this connection to regret explicit by noting that
\[
\frac{1}{\sqrt{n}}\sum_{t = 1}^n \wh{\sigma}_t^{-1}\left\{\theta\big(\wh{\pi}_{t - 1}\big) - \theta^\ast\right\} = -\frac{1}{\sqrt{n}}\sum_{t = 1}^n \wh{\sigma}_t^{-1}\Reg\big(\wh{\pi}_{t - 1}
\big).
\]
In Section~\ref{sec:regret}, we show that $\E\left[\Reg\big(\wh{\pi}_{t - 1}\big)\right] = o(t^{(\alpha - 1)/2})$. Thus, by leveraging results on the convergence of Ces\`aro summations~\citep{bibaut2020sufficient} and noting that $\wh{\sigma}_t \asymp t^{\alpha/2}$, we see that the entire cumulative regret term satisfies
\[
\frac{1}{\sqrt{n}}\sum_{t = 1}^n \E\left[\wh{\sigma}_t^{-1}\Reg\big(\wh{\pi}_{t - 1}
\big)\right] = o(1),
\]
and hence converges to zero in probability. Putting these two pieces together furnishes the desired result.

\end{proof}

\subsection{Consistently Estimating the Variance}
\label{sec:estimator:variance}
The previous subsection established our main central limit theorem for our estimator $\wh{\theta}_n$. However, we have yet to show that there exists an estimator $\wh{\sigma}_t^2$ of the predictable variance $\sigma_t^2$ that achieves the requisite convergence rates. Our focus in this subsection is to show that a natural plug-in estimate converges at the desired rate under a combinatorial assumption on the class of logging policies at each point in time.

Our variance estimator and corresponding analysis are largely based on those of \citet{bibaut2021post}, who consider the task of predictable variance estimation under adaptively collected data. We note that there are three primary difficulties that prevent us from directly leveraging the aforementioned estimator in our setting. First, said estimator is designed for a \textit{fixed policy} (more broadly, function) that is being evaluated under contextual bandit data, not one that is continually updated according to a welfare maximization problem. Second, our assumptions, which are based on finite combinatorial dimension, are not comparable to the bracketing entropy assumptions made in \citet{bibaut2021post}. These combinatorial assumptions require us to leverage a substantially different chaining argument for showing variance convergence than the one presented in the aforementioned paper, with the VC-like assumptions greatly complicating the bounding of error in finite-depth chaining.
We discuss the particulars of the proof technique in the relevant appendices (namely, Appendices~\ref{app:variance} and \ref{app:martingale-bounds}).

Third, and most problematically, is that the variance estimate in \citet{bibaut2021post} requires that $\alpha \leq 1/3$ in order to guarantee almost sure convergence. This limitation is particularly undesirable, as any minimax regret optimal logging policy must have a logging exponent of at least $\alpha \geq 1/2$. In this section, we circumvent this restriction by observing that, for performing inference on the value of policy improvement, a significantly weaker notion of variance convergence is permissible.
In more detail, our analysis actually allows the error in a predictable variance estimate to \textit{grow} with the sample size, so long as this error vanishes to zero almost surely when divided by $\sigma_t^2 \asymp t^{\alpha}$. This weaker convergence requirement will ultimately allow us to show variance consistency under logging policies with aggressive exploitation (i.e., $\alpha \approx 1$).

We now describe the variance estimator we consider. Using the contrast functions $g_\pi(a \mid x) := \1\{\pi(x) = a\} - q(a \mid x)$ outlined earlier, we define the contrast under the (predictable) learned policy at time $t$ as $g_t(a \mid x) := g_{\wh{\pi}_{t - 1}}(a \mid x)$. In terms of these contrasts, we can express the predictable variance $\sigma_t^2$ as
\begin{equation}
\label{eq:pred-variance}
\sigma_t^2  = \E\left(\frac{g_t(A_t \mid X_t)^2 Y_t^2}{p_t(A_t \mid X_t)^2} \mid \calF_{t - 1}\right) - \E\left(\frac{g_t(A_t \mid X_t)Y_t}{p_t(A_t \mid X_t)} \mid \calF_{t - 1}\right)^2.
\end{equation}
Using the above decomposition, we propose estimating $\sigma_t^2$ by estimating the first and second terms separately and then taking a difference. In more detail, we define the estimators for the first moment and the second moment as
\begin{equation}
\label{eq:linear-term}
\wh{\Phi}_{1, t} := \frac{1}{t - 1}\sum_{s = 1}^{t - 1}\frac{g_t(A_s \mid X_s)}{p_s(A_s \mid X_s)}Y_s
\end{equation}
and
\begin{equation}
\label{eq:square-term}
\wh{\Phi}_{2, t} := \frac{1}{t - 1}\sum_{s = 1}^{t - 1}\frac{g_t(A_s \mid X_s)^2}{p_s(A_s \mid X_s)p_t(A_s \mid X_s)}Y_s^2.
\end{equation}

Then, we define our estimator for the overall predictable variance at sample size $t$ as
\begin{equation}
\label{eq:pred-var-est}
\wt{\sigma}_t^2 := \wh{\Phi}_{2,t} - \wh{\Phi}_{1, t}^2,
\end{equation}
where we can let $\wt{\sigma}_1^2$ be arbitrary.
We only use the first $t - 1$ samples to construct our estimate of $\sigma_t^2$ in order to ensure the estimate is itself predictable, i.e.\ $\calF_{t - 1}$-measurable. As noted above, this is exactly analogous to the variance estimator outlined in \citet{bibaut2021post}, modulo that we consider (i) contrasts in policy values and (ii) adaptively update the policy as more samples are gathered. Again, what will differ substantially is the form of our convergence guarantee and the structure of our analysis.

Under just our complexity assumption on the class $\Pi$, one can show the consistency of $\wh{\Phi}_{1, t}$ using arguments based on empirical process theory. However, because the estimate $\wh{\Phi}_{2, t}$ also depends on the (random) choice of logging policy at time $t$, we cannot hope for the estimate to converge unless we restrict the complexity of the set of valid logging policies. We leverage the following assumption, which restricts the \textit{pseudo-dimension} of the class in which the adaptive logging policy lies, which we denote by $d_P$. We define pseudo-dimension and its relationship to other notions of combinatorial dimension in Appendix~\ref{app:combinatorial-dimension}. This assumption should again be contrasted with the bracketing-type assumptions made in \citet{bibaut2021post}.

\begin{ass}[Logging Class Complexity]
\label{ass:logging}
For each $t \geq 1$, there is some bounded fixed class $\calP_t$ such that (i) $t^{-\alpha}p_t^{-1} \in \calP_t$ almost surely, (ii) $\sup_{t \geq 1}d_P(\calP_t) \leq \frakp$ for some $\frakp > 0$, and (iii) $\sup_{\substack{t \geq 1 \\ r \in \calP_t}} \|r\|_\infty \leq P$, for some constant $P > 0$.
\end{ass}

\begin{theorem}
\label{thm:variance-convergence}
Suppose Assumptions~\ref{ass:observations}, \ref{ass:disagreement}, \ref{ass:estimation}, \ref{ass:propensity}, \ref{ass:ipw_variance_scale}, and \ref{ass:logging} hold. Then, the following holds:
\begin{equation}
\label{eq:variance-rate}
\frac{\wt{\sigma}_t^2 - \sigma_t^2}{\sigma_t^2} \xrightarrow[t \rightarrow \infty]{a.s.} 0.
\end{equation}
\end{theorem}

We prove Theorem~\ref{thm:variance-convergence} in Appendix~\ref{app:variance}. Our argument relies on developing martingale concentration inequalities for infinite collections of martingales indexed by classes of finite pseudo-dimension. We detail and prove these inequalities in the first part of Appendix~\ref{app:martingale-bounds}. We then use high-probability bounds alongside the Borel-Cantelli lemma to obtain the scaled, almost sure convergence of $\wh{\Phi}_{2, t}$ and $\wh{\Phi}_{1, t}$.

We note that Theorem~\ref{thm:clt} actually requires the variance estimates to be bounded away from zero by some decreasing sequence $(\kappa_t)$. This can be accomplished by simply performing lower truncation on our variance estimate. In more detail, we can simply set our variance estimate to be $\wh{\sigma}_t := \sqrt{\wt{\sigma}_t^2 \lor \kappa_t^2}$. Since $\sigma_t^2 \geq c_\sigma t^{\alpha}$ for sufficiently large $t$, the results of Theorem~\ref{thm:variance-convergence} imply truncation will only ever be applied to a finite number of sample sizes.

\subsection{A Batched Estimator}
The estimator described in Equation~\eqref{eq:estimator} involves repeatedly computing an empirical welfare maximizing policy on all prefixes of data, thus resulting in a number of fits proportional to the sample size.
In general, exactly solving the empirical risk minimization is a computationally expensive task, often requiring the use of enumeration or mixed integer programming (MIP) based methods. While efficient approximations for EWM policy learning exist (e.g., greedy splitting for policy trees), such  methods are not guaranteed to converge to a population maximizer in the large sample limit.

Due to the costliness of exactly solving empirical welfare maximization, we propose a more efficient, batch-based version of our estimator in this section. The modified algorithm works by splitting time into geometrically spaced epochs and refitting a policy using historical data at the beginning of each epoch. The possibility of such batched estimators was discussed in Section 6.1 of~\citet{luedtke2016statistical} in the case of i.i.d.\ data, but was not formally worked out. We formalize these ideas in Algorithm~\ref{alg:online-epoch} and Corollary~\ref{cor:epoch_estimator} below, with a more detailed discussion in the remark following the corollary. We note that our algorithm/estimator also allows (but does not require) a burn-in period of $b_n$ samples for training the first policy, which is intended to prevent the use of poorly trained policies that may amount to random guessing in practice. The proof of the following result is exactly analogous to that of Theorem~\ref{thm:clt}, and so we do not include it.

\begin{algorithm}[t]
\caption{Batched Estimator for $\theta^\ast$}
\label{alg:online-epoch}
\begin{algorithmic}[1]
\Require Samples $Z_1, \dots, Z_n$, confidence parameter $\delta \in (0, 1)$, a policy class $\Pi$, burn-in period $b_n$, indices $b_n = \tau_1 < \tau_2 < \cdots < \tau_{H + 1} = n$, truncation levels $(\kappa_t)_{t \geq 1}$,  baseline $q$.
\Ensure Estimate $\wh{\theta}_n$ of $\theta^\ast := V^\ast - V(q)$, $(1- \delta)$-confidence interval $\calC_{1 -\delta}$ for $\theta^\ast$.
\For{$h \in [H]$}
    \State Compute $\wh{\pi}_{h} := \arg\max_{\pi \in \Pi} \sum_{s = 1}^{\tau_h} s^{-\alpha}\frac{g_\pi(A_s \mid X_s)}{p_s(A_s \mid X_s)}Y_s$.
    \State For $t \in \{\tau_h + 1, \dots, \tau_{h + 1}\}$, define  $\wt{\sigma}_{t}^2$ by
    \[
    \wt{\sigma}_t^2 := \frac{1}{t - 1}\sum_{s = 1}^{t - 1}\frac{g_{\wh{\pi}_h}(A_s \mid X_s)^2}{p_s(A_s \mid X_s)p_t(A_s \mid X_s)}Y_s^2 - \left(\frac{1}{t - 1}\sum_{s = 1}^{t - 1}\frac{g_{\wh{\pi}_h}(A_s \mid X_s)}{p_s(A_s \mid X_s)}Y_s\right)^2
    \]
    \State Set  $\wh{\sigma}_t := \sqrt{\wt{\sigma}_t^2 \lor \kappa_t^2}$.
\EndFor
\State Compute estimate
\[
\wh{\theta}_n := \frac{\sum_{h = 1}^{H}\sum_{t = \tau_h + 1}^{\tau_{h + 1}}\wh{\sigma}_{t}^{-1}\frac{g_{\wh{\pi}_h}(A_t \mid X_t)}{p_t(A_t \mid X_t)}Y_t}{\sum_{h = 1}^{H}\sum_{t = \tau_h + 1}^{\tau_{h + 1}}\wh{\sigma}_t^{-1}}.
\]
\State Return the estimate $\wh{\theta}_n$ and $(1 - \delta)$ confidence interval
\[
\calC_{1 - \delta} := \left[\wh{\theta}_n - \frac{\sqrt{n - b_n}}{\sum_{t = b_n + 1}^n \wh{\sigma}_t^{-1}}z_{1 - \delta/2}, \wh{\theta}_n + \frac{\sqrt{n - b_n}}{\sum_{t = b_n + 1}^n \wh{\sigma}_t^{-1}}z_{1 - \delta/2}\right]
\]
\end{algorithmic}
\end{algorithm}

\begin{corollary}
\label{cor:epoch_estimator}
Suppose the assumptions of Theorem~\ref{thm:clt} hold and let $\wh{\theta}_n$ be as defined in Algorithm~\ref{alg:online-epoch}. Further, suppose there are constants $c \in (0, 1)$ and $C > 0$ such that $\tfrac{\tau_{h + 1}}{\max\{\tau_h, 1\}} \leq C$ for each $h$ and that $n - b_n \geq c n$. Then, we have
\[
\frac{\sum_{t = b_n + 1}^{n}\wh{\sigma}_t^{-1}}{\sqrt{n - b_n}}\left(\wh{\theta}_n - \theta^\ast\right) \Rightarrow \calN(0, 1),
\]
and hence $\calC_{1 - \delta}$ is an asymptotically-valid $1 - \delta$ confidence interval for $\theta^\ast$.
\end{corollary}

We note that the consistency of the variance estimate described in Algorithm~\ref{alg:online-epoch} follows from the proof of Theorem~\ref{thm:variance-convergence}, as said proof establishes uniform convergence over a relevant class of functions.

\begin{rmk}
Corollary~\ref{cor:epoch_estimator} shows that our estimator still provides valid coverage if we refit a welfare-maximizing policy only on the order of $\log(n)$ times, where $n$ again denotes the sample size. One may wonder then if we could get away with a constant number of refits. For concreteness, we will speculate about the possibility of simply doing \textit{sample splitting}, where $b_n = o(n)$ samples are used to learn a policy, and the remaining $n - b_n$ are used for evaluation (obviously, we assume $b_n \rightarrow \infty$). Straightforward analysis yields that we would need regret to vanish at a rate of $\Reg\big(\wh{\pi}_{b_n}\big) = o_\P\left(n^{(\alpha - 1)/2}\right)$, which under the assumption that $b_n = o(n)$ is strictly faster than the rates that we will show in Section~\ref{sec:regret} below. For instance, if $b_n = n/\log(n)$, one would need a regret rate of
\[
\Reg\big(\wh{\pi}_{b_n}\big) = o_\P\left((b_n\log(b_n))^{(\alpha - 1)/2}\right).
\]
Such faster rates may indeed be possible under margin conditions on the conditional average treatment effect or blip effects, and \citet{girard2025online} discuss such high-probability rates in online settings. However, the margin condition discussed in the aforementioned paper implicitly requires a unique optimal policy, so deriving such faster rates under non-unique optimal policies would be an interesting direction for future work.
\end{rmk}

\begin{rmk}
We now conclude this subsection by briefly discussing the heuristic batching approach outlined in Section 6.1 of \citet{luedtke2016statistical}. In this subsection, the authors note that one can first estimate the nuisances (e.g., outcome regression and propensity in the unrestricted setting) on the first $b_n$ samples and then break the remaining $n - b_n$ into a constant number of chunks/batches. However, this approach actually requires that the nuisances estimated on the first batch of data converge at \textit{faster} than standard estimation rates, with the degree faster depending on the size of the burn-in period.

To illustrate this, we consider a concrete example in the binary treatment setting under i.i.d.\ data and unrestricted policy classes. Define the conditional average treatment effect as $\tau^\ast(x) := \mu_0(1, x) - \mu_0(0, x)$ and suppose that $\P(0 < |\tau^\ast(X)| \leq t) \lesssim t^\beta$ for some $\beta > 0$ and all $t \leq t_0$.  If one refits nuisances at each step (or uses the geometric batching that we outline above), the estimate of the outcome regression would need to satisfy $\|\wh{\mu}_n - \mu_0\|_{L^2} = o_\P\left(n^{-a_\beta}\right)$, where $a_\beta := \frac{2 + \beta}{4(1 + \beta)}$. If, however, one leverages just a constant number of fits and takes $b_n = n/\log(n)$ (as in the example above), the $L^2$ error in regression estimation would need to scale as $o_\P\left((b_n\log(b_n))^{-a_\beta}\right)$. We also note that setting $b_n = n/\log(n)$ actually sacrifices a large fraction of the samples even at moderate to large sample sizes, which empirically is likely to result in worse performance. If one instead wanted to take $b_n = \sqrt{n}$, a significantly faster nuisance estimation rate would be needed. In particular, one would need $\|\wh{\mu}_{b_n} - \mu_0\|_{L^2} = o_\P\left(b_n^{-\frac{2 + \beta}{2(1 + \beta)}}\right)$.
This is thus one drawback of using a constant number of fits as outlined in \citet{luedtke2016statistical}.
\end{rmk}

\section{Faster Regret for Weighted Empirical Welfare Maximization}
\label{sec:regret}

To ensure the asymptotic normality of our estimator, we need to ensure that the ``bias'' accumulated from evaluating a learned policy is negligible relative to the width of the confidence intervals. As seen in our proof sketch, to show this accumulated bias is negligible, it suffices to show that the expected regret incurred through weighted EWM vanishes at rates faster than $n^{(\alpha - 1)/2}$, i.e.,
\[
\underbrace{\E\left[\Reg\big(\wh{\pi}_{n}\big)\right] = o\big(n^{(\alpha - 1)/2}\big)}_{\text{Fast Expected Regret Rates}} \Longrightarrow \underbrace{\frac{1}{\sqrt{n}}\sum_{t = 1}^n \wh{\sigma}_t^{-1}\left\{\theta\big(\wh{\pi}_{t - 1}\big) - \theta^\ast\right\} = o_\P(1)}_{\text{Negligible Accumulated Bias}}.
\]

At first glance, obtaining regret rates such as the above may seem impossible. In their work, \citet{zhan2024policy} show that the weighted empirical welfare maximization procedure outlined in Equation~\eqref{eq:ewm} yields expected regret rates of $O\left(n^{(\alpha - 1)/2}\right)$, which is on the same order as the width of the confidence intervals outlined in Theorem~\ref{thm:clt}. Further, they show a matching lower bound --- the minimax expected regret of \textit{any learning algorithm} cannot decay faster than $\Omega\left(n^{(\alpha - 1)/2}\right)$. However, there is an important and subtle difference between our setting and that considered by \citet{zhan2024policy}. The aforementioned authors are considering a \textit{distribution uniform, finite sample} setting, obtaining upper bounds with known constants that hold regardless of the choice of data generating distribution and sample size. We instead aim to obtain convergence guarantees that hold \textit{pointwise} for the specific data generating distribution being considered, and we want \textit{asymptotic convergence guarantees}. In particular, the exact constants in our setting can actually depend on the distribution, and sample size at which regret starts becoming negligible can also be different for each distribution.

We are not the first to realize that such faster regret rates are possible. \citet{luedtke2020performance} study exactly this phenomenon in the setting of i.i.d.\ data. Under mild distributional assumptions, they show that the regret incurred through standard EWM actually vanishes as $o_\P(n^{-1/2})$ in a distributionally pointwise sense. This is notably (infinitesimally) faster than the $\Theta(n^{-1/2})$ minimax rates known in this setting~\citep{athey2021policy, zhan2024policy}. The authors also discuss more general settings under which such fast regret rates are possible, noting that such faster rates are possible under the assumption that the empirical process associated with the accumulated data and indexed by the policy class approaches a tight Gaussian limit.

Our contribution in this section is that we show fast expected regret rates hold under adaptively collected data with decaying propensities. Unlike the arguments made in \citet{luedtke2020performance}, our approach for showing convergence is constructive, relying on finite sample expectation and high-probability bounds for collections of martingales indexed by VC-classes alongside a localization argument. In particular, our results do not rely on the assumption that the empirical process converges to a tight Gaussian limit.\footnote{The verification of such convergence is non-obvious in our setting of adaptively collected data with decaying propensities.}
Further, our guarantees are actually stronger, as the aforementioned work only demonstrates ``in probability'' regret convergence, which as noted above is insufficient for the convergence of our estimator.
Showing convergence of the expectation requires arguing the uniform integrability of the \textit{scaled regret process} $\big(n^{(1 - \alpha)/2}\Reg(\wh{\pi}_{n - 1})\big)_{n \geq 1}$. This, in turn, requires intricate, high-probability chaining-based bounds for collections of martingales. We detail these in Appendix~\ref{app:martingale-bounds}.

We now state the main contribution of this section, which is our theorem on fast expected regret convergence for weighted empirical welfare maximization.

\begin{theorem}
\label{thm:fast-regret}
Suppose $(Z_n)_{n \geq 1}$ are observations satisfying Assumptions~\ref{ass:observations}, \ref{ass:propensity}, and the first part of Assumption~\ref{ass:estimation}. Further, assume that $\Pi$ is a closed policy class satisfying Assumption~\ref{ass:disagreement}. Let $\wh{\pi}_n$ be as defined in Equation~\eqref{eq:ewm}. Then
\[
\E\left[\Reg\big(\wh{\pi}_n\big)\right] = o\left(n^{\frac{\alpha - 1}{2}}\right).
\]
\end{theorem}

We provide a full proof of the above theorem in Appendix~\ref{app:fast-regret}. To build intuition, we provide a heuristic proof sketch that shows $\Reg\big(\wh{\pi}_n\big) = o_\P(n^{(\alpha - 1)/2})$, and leave the proof of uniform integrability fully to the appendix.

\begin{proof}[Proof Sketch of Theorem~\ref{thm:fast-regret}]

 In this sketch, we consider the ``disagreement metric'' $d(\pi, \rho) := \P\left(\pi(X) \neq \rho(X)\right)^{1/2}$, under which policies are close if they agree with high probability. We also define $w_t := t^{-\alpha}$ and $S_{n, -\alpha} := \sum_{t = 1}^n w_t$ and the process
 \[
 M_n(\pi, \rho) := \sum_{t = 1}^n w_t \left\{\wt{Y}_t(\pi) - \wt{Y}_t(\rho) - \left(V(\pi) - V(\rho)\right)\right\},
 \]
 which is a martingale for any fixed $\pi, \rho \in \Pi$. We leverage two key ingredients in our proof:
 \begin{itemize}
     \item \textit{A Basic Inequality.} We can bound the regret as
     \[
     \Reg\big(\wh{\pi}_n\big) \leq \frac{1}{S_{n, -\alpha}}\inf_{\pi^\ast \in \Pi^\ast}M_n\left(\wh{\pi}_n, \pi^\ast\right) \leq \frac{1}{S_{n, -\alpha}}\sup_{d(\pi,\rho) \leq \delta}M_n(\pi, \rho),
     \]
     where the final inequality holds on the event $\Big\{d\big(\wh{\pi}_n, \Pi^\ast\big) \leq \delta\Big\}$.
     \item \textit{A Localized Bound.} The expectation of the above localized empirical process is well controlled:
     \begin{equation*}
    \E\left[\frac{1}{S_{n, -\alpha}}\sup_{d(\pi, \rho) \leq \delta}\left|M_n(\pi, \rho)\right|\right] \lesssim  \delta n^{\frac{\alpha - 1}{2}}\sqrt{\frakv\log\left(\frac{A}{\delta}\right)}  + n^{\alpha - 1}\frakv\log\left(\frac{A}{\delta}\right)\log(n)^2,
    \end{equation*}
    where again $\frakv$ is the VC-dimension of the difference class and $A \geq e$ is some constant that may depend on $|\calA|$.
 \end{itemize}

Using the above ingredients, we can prove the target convergence in probability rate in three steps. First, noting that $\Big\{d\big(\wh{\pi}_n, \Pi^\ast\big) \leq 1\Big\}$ has probability one, we can use the above basic inequality, the localized bound with $\delta = 1$, and Markov's inequality to show that
 $\Reg\big(\wh{\pi}_n\big) = o_\P(1)$.
Next, regret consistency and continuity of regret ensure $d\big(\wh{\pi}_n, \Pi^\ast\big) = o_\P(1)$. In other words, the sequence of learned policies $(\wh{\pi}_n)$ gets arbitrarily close to the optimal set $\Pi^\ast$, although they may not converge to any fixed policy.
Finally, the above convergence yields a sequence $(\delta_n)$ slowly tending towards zero such that
\[
\P\big(\Epsilon_n\big) \xrightarrow[n \rightarrow \infty]{} 1, \quad \text{where} \quad \Epsilon_n := \Big\{d\big(\wh{\pi}_n, \Pi^\ast\big) \leq \delta_n\Big\}.
\]
Heuristically, $\Epsilon_n$ occurs with overwhelming probability for large $n$, and so Markov's inequality yields for any $\epsilon > 0$ that
\begin{align*}
&\P\left(\Epsilon_n \cap \Big\{n^{(1 - \alpha)/2}\Reg\big(\wh{\pi}_n\big) \geq \epsilon\Big\}\right) \leq \frac{1}{\epsilon}\E\left[\frac{1}{\sqrt{S_{n, -\alpha}}}\sup_{d(\pi, \rho) \leq \delta_n} \left|M_n(\pi, \rho)\right|\right] \\
&\quad \leq \frac{1}{\epsilon}\left\{\delta_n\sqrt{\frakv\log\left(\frac{A}{\delta_n}\right)} + n^{\frac{\alpha - 1}{2}}\frakv\log\left(\frac{A}{\delta_n}\right)\log(n)^2\right\}.
\end{align*}
where we have leveraged the fact that $S_{n, -\alpha} \asymp n^{1 - \alpha}$.
The first term will always go to zero as $n \rightarrow \infty$. The second term will tend towards zero if $\delta_n$ decays sufficiently slowly.

The above steps, when formalized, guarantee the desired convergence rate in probability. While we do not sketch the argument for uniform integrability in this part of the paper, it follows in a similar manner. However, instead of using expectation bounds, we actually need to leverage and derive high-probability chaining bounds instead. We defer the discussion of this until the appendix.
\end{proof}

\begin{rmk}
\label{rmk:fast_regret}
We note that Theorem~\ref{thm:fast-regret} is simultaneously a very powerful and weak result. The result is strong in that it enables the asymptotic convergence of our estimator to the true value of policy improvement, which could not have been established via existing, distribution-uniform expected regret guarantees.
However, the result does not give a practitioner a quantitative measure of how rapidly the scaled regret process $n^{(1 - \alpha)/2}\Reg(\wh{\pi}_n)$ tends towards zero as the sample size grows. In our experiments in the following section, we try to ascertain how quickly the above scaled quantity vanishes. We investigate this convergence both in fully synthetic and semi-synthetic data generating processes. Further, we study the convergence when the empirical welfare maximization problem is solved exactly or approximately (such as through greedy splits for policy trees).
While our empirical findings show favorable convergence rates, they naturally do not preclude slower decay of the scaled regret. We hypothesize that it is possible to construct data generating processes/policy classes such that the scaled regret converges arbitrarily slowly to zero, but we leave the exact construction of these objects for future work.
\end{rmk}

\section{Experiments}
\label{sec:exp}

We now study the empirical behavior of our estimator for the value of policy improvement. Our experiments in the sequel are aimed at evaluating the coverage of our estimator under data generating processes where there are multiple (in fact, infinitely many) optimal treatment policies.
We consider both synthetic experiments (Section~\ref{sec:exp:synthetic}), in which the true conditional average treatment effect is known and thus the optimal policy value can be computed exactly, semi-synthetic experiments (Section~\ref{sec:exp:semi-synthetic}) based on the \textit{Jester} dataset~\citep{goldberg2001eigentaste}, and then experiments on the impact of interventions on employment of Jordanians and Syrians given a prior study~\cite{caria2024adaptive}.

Before diving into our findings, we begin by describing our comparators, providing greater detail in Appendix~\ref{app:exp}. We compare our adaptively-weighted estimator (\aw) against the natural baselines summarized in
\cref{tab:comparators}, all computed on identical adaptively-logged data. The
estimators of \citet{bibaut2021post} (\biba) and \citet{zhan2021off} (\zhan) are designed for
off-policy evaluation of a single \emph{fixed} policy; we turn them into estimators of policy
improvement by learning a policy on part of the data and evaluating it on the rest. In particular, we consider this sample splitting in two
regimes --- a sublinear burn-in regime where a vanishing fraction of the data is used for policy learning, and a $50/50$ split.
We also consider a double-dipping estimator (\dd), which
learns and evaluates a policy on the entire sample. Interestingly, for policy classes satisfying the complexity assumptions of this paper, double dipping will yield asymptotically normal inference so long as there is a unique optimal treatment policy~\citep{fang2019inference, carcamo2020directional}. However, when there are multiple optimal policies, the limiting distribution is non-normal, and is instead characterized by the supremum of an empirical process indexed by the set of treatment policies.
We note that the convergence of the above sample splitting estimators is not a priori obvious, and requires the faster regret rates from this paper to hold.\footnote{For the $50/50$ sample-splitting estimators, the faster regret rates established in this paper ensure that the regret of the learned policy is negligible relative to the standard error of its estimated value. The sublinear burn-in versions need not enjoy the same guarantee and are included as heuristic comparators.}

\begin{table}[!htb]
\centering
\footnotesize
\setlength{\tabcolsep}{5pt}
\caption{A summary of comparators and the design choices that distinguish them. \biba{} and \zhan{} are each run in two modes: a sublinear burn-in
split consisting of $b_n = o(n)$ samples, and a $50/50$ split ($b_n = n/2$).}
\label{tab:comparators}
\begin{tabular}{l l l c c c}
\toprule
Estimator & Learns policy on & Evaluates on & Inv.-SD wt. & Refits & Out-of-sample \\
\midrule
\aw{} (ours) & geometric epochs       & future rounds      & $\checkmark$ & $\checkmark$ & $\checkmark$ \\
\biba{}      & prefix ($b_n$ or $n/2$) & remaining rounds   & $\checkmark$ & $\times$   & $\checkmark$ \\
\zhan{}      & prefix ($b_n$ or $n/2$) & remaining rounds   & $\times$   & $\times$   & $\checkmark$ \\
\dd{}        & all $n$                & all $n$ (same data) & $\times$   & $\times$   & $\times$   \\
\bottomrule
\end{tabular}
\end{table}

\subsection{Synthetic Experiments}
\label{sec:exp:synthetic}

In this subsection, we evaluate the finite sample performance of our estimator on a synthetic data generating process with a binary treatment. We first describe our experimental setup, including our choice of data generating process and logging rates and policies. We then discuss our findings, which generally indicate that our estimator obtains better coverage than comparators while maintaining smaller confidence intervals. Throughout, the baseline policy will correspond to uniform randomization, i.e., $q(0 \mid x) = q(1 \mid x) = 1/2$ regardless of covariates $x$.

\paragraph{Policy Class.} In this section, we take the policy class $\Pi$ to be the set of all depth 2 policy trees.
When fitting these trees, we impose the empirical analogue, requiring at least 15\% of the training observations in each leaf.
We solve the welfare
maximization problem \emph{exactly} using the optimal policy-tree solver of
\citet{sverdrup2020policytree}.
One example of an optimal policy is a depth 1 tree that splits at $x_1 = 0.4$. This gives leaves of mass
$0.4$ and $0.6$, so the $15\%$ floor is not binding.

\paragraph{Data-generating process.} We propose a data generating process (DGP) with a binary treatment. First, covariates
$X = (X_1,\dots,X_5)$ are generated such that $X_j \stackrel{\mathrm{iid}}{\sim} \mathrm{Unif}(0,1)$.
The expected outcome of a unit under control is then given by $\mu_0(x) := 1 + 0.4\Bigl(x_2-\tfrac12\Bigr) + 0.3\Bigl(x_3-\tfrac12\Bigr)^2 - 0.2\,x_4 x_5$.
The conditional average treatment effect is given by:
\[
\tau_0(x) :=
\begin{cases}
+1, & x_1 \in [0,\,0.4),\\
\phantom{+}0, & x_1 \in [0.4,\,0.7),\\
-1, & x_1 \in [0.7,\,1].
\end{cases}
\]
The potential outcomes are $Y(a) := \mu_0(X) + a\tau_0(X) + 0.2\epsilon$ with
$\epsilon \sim \calN(0,1)$ and $a \in \{0,1\}$, and the observed outcome is $Y := Y(A)$. Under this CATE there are in fact \textit{infinitely many} optimal depth 2 policy trees. We illustrate two such trees in Figure~\ref{fig:multiple-optima} in Appendix~\ref{app:exp:synth} below, but we note that any depth 1 tree that splits at $x_1 \in [0.4, 0.7]$ indeed attains the optimal value. The true VPI in this setting is $\theta^\ast = V^\ast -
V(q) = 0.350$.

\paragraph{Logging Policies.} Treatments are  assigned per an $\epsilon_t$-greedy
logging policy with $\epsilon_t = t^{-\alpha}$. At round $t$, with probability $\epsilon_t$ the action is
drawn uniformly at random from $\calA = \{0, 1\}$. Otherwise, the treatment is assigned greedily as the action with the highest empirical predictable mean, with ties broken lexicographically:
\begin{gather*}
\wh{a}_t := \arg\max_{a \in \calA}\wh{\rho}_{t - 1}(a), \;\; \text{where}\quad \wh{\rho}_{t - 1}(a) := \frac{1}{t - 1}\sum_{s = 1}^{t - 1}\wt{Y}_s(a) \\
A_t \sim p_t(\cdot \mid X_t) \qquad \text{where } p_t(a \mid x) = \frac{\epsilon_t}{2} + (1 - \epsilon_t)\1\{a = \wh{a}_t\}.
\end{gather*}
We adopt the convention that treatments are assigned uniformly at random when $t = 1$.

\paragraph{Evaluation.} We consider  $\alpha \in \{0, 0.1, 0.25, 0.5\}$ and sample
sizes $n \in \{2000, 5000, 10000, 15000\}$ with respective sublinear burn-in periods $b_n \in \{500, 1000,
1500, 2000\}$. We fit \aw{} using Algorithm~\ref{alg:online-epoch} with geometric epochs satisfying $\tau_{h + 1} - \tau_h \approx 1.5^h$, flooring the result in order to ensure batch size is an integer. For each configuration we run $250$ independent replications and
report coverage of the nominal $95\%$ intervals and the average interval width.

\paragraph{Findings.} We show the coverage of all estimators as a function of the sample size and logging coefficient $\alpha$ in Figure~\ref{fig:nu-cov-synth}.
Further, we compare the coverage and average confidence interval width of all estimators in Figure~\ref{fig:nu-pareto-synth},  thus showing tradeoffs between coverage and confidence interval width.
We show exact coverage numbers (with 95\% Wilson CIs) for our estimator alongside top comparators (\zhan{} and \biba{} with 50/50 splits) in~\Cref{tab:nu-cov-synth} of Appendix~\ref{app:exp:synth}.
We see that, due to non-uniqueness of the optimal treatment policy, the \dd{} under-covers at every sample size, with
coverage hovering between roughly $0.51$ and $0.66$.
In contrast, \aw{} and the two $50/50$ sample splits sit at or near the nominal level for $\alpha
\le 0.25$, with \aw{} delivering the smallest confidence interval widths.
We note that all estimators exhibit diminished performance for $\alpha = 0.5$.
There, \aw{}'s coverage dips to
$\approx 0.77$--$0.82$.
As evidenced in Section~\ref{sec:exp:regret}, the \textit{scaled regret} $\Reg(\wh{\pi})/n^{(\alpha - 1)/2}$ vanishes slowly as $n$ grows for larger $\alpha$, thus indicating that larger sample sizes may be needed to meet the desired coverage. While \zhan{} has coverage that is similar to our \aw, and  better for $\alpha=0.5$,  we also note that \zhan{} has confidence intervals that are around $1.5\times$ wider than \aw{} in this setting, which may also aid in coverage.

In Figure~\ref{fig:nu-pareto-synth} we see that our \aw{} presents the strongest balance of tighter confidence intervals and coverage.

\begin{figure}[!htbp]
\centering
\includegraphics[width=\linewidth]{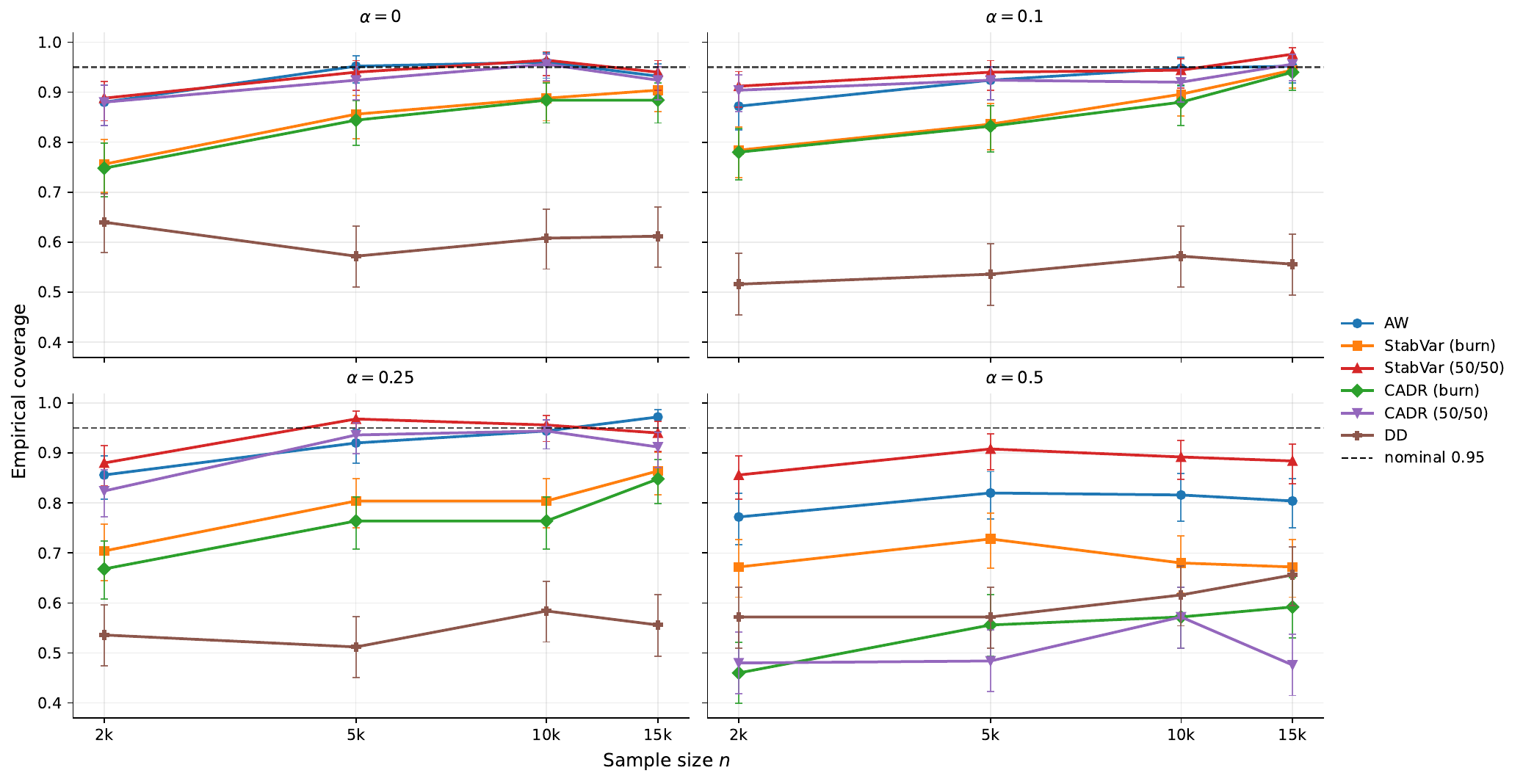}
\caption{Empirical coverage versus sample size in the synthetic DGP for all six
estimators at logging exponents $\alpha \in \{0, 0.1, 0.25, 0.5\}$. Error bars are Wilson $95\%$
intervals and the dashed line is the nominal $0.95$ coverage target. Double-dipping (\dd, the lowest curve)
under-covers and does not recover with $n$, while \aw{} and the $50/50$ estimators obtain nominal coverage except in the $\alpha = 0.5$ case.}
\label{fig:nu-cov-synth}
\end{figure}

\begin{figure}[!t]
\centering
\begin{subfigure}[t]{0.4\textwidth}
  \includegraphics[width=\linewidth]{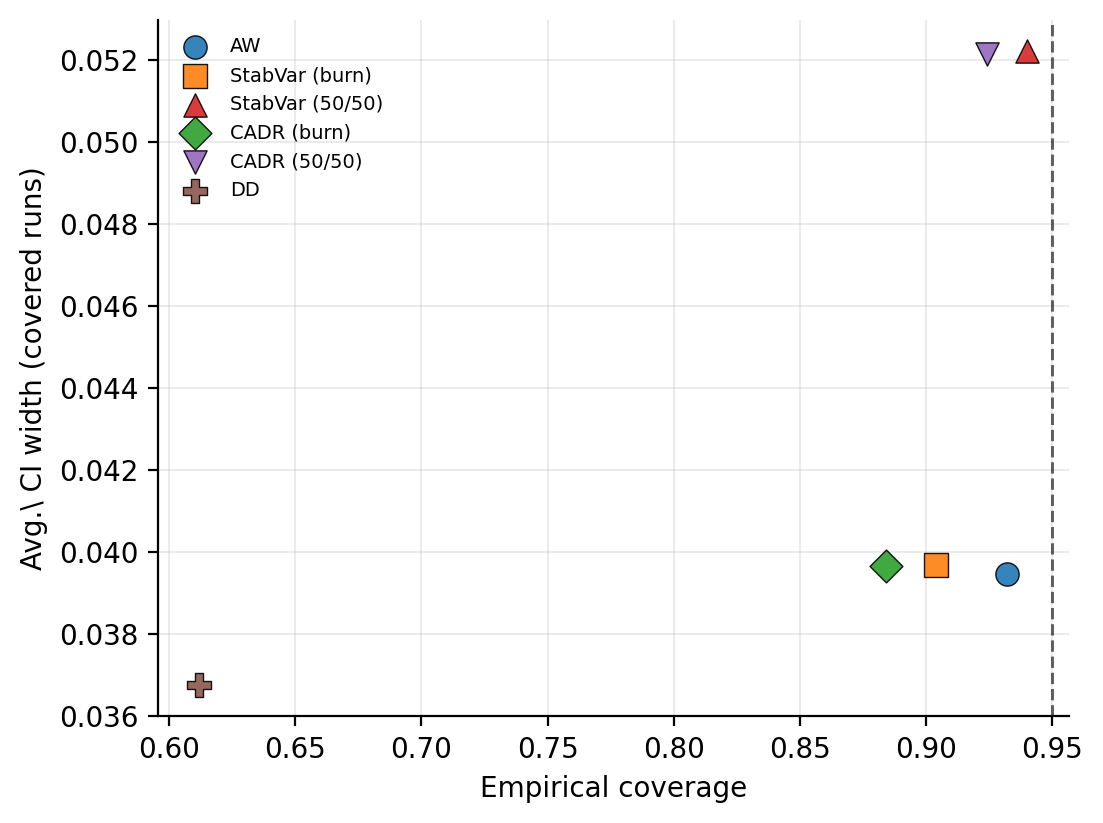}
  \caption{$\alpha = 0$.}
\end{subfigure}
\begin{subfigure}[t]{0.4\textwidth}
  \includegraphics[width=\linewidth]{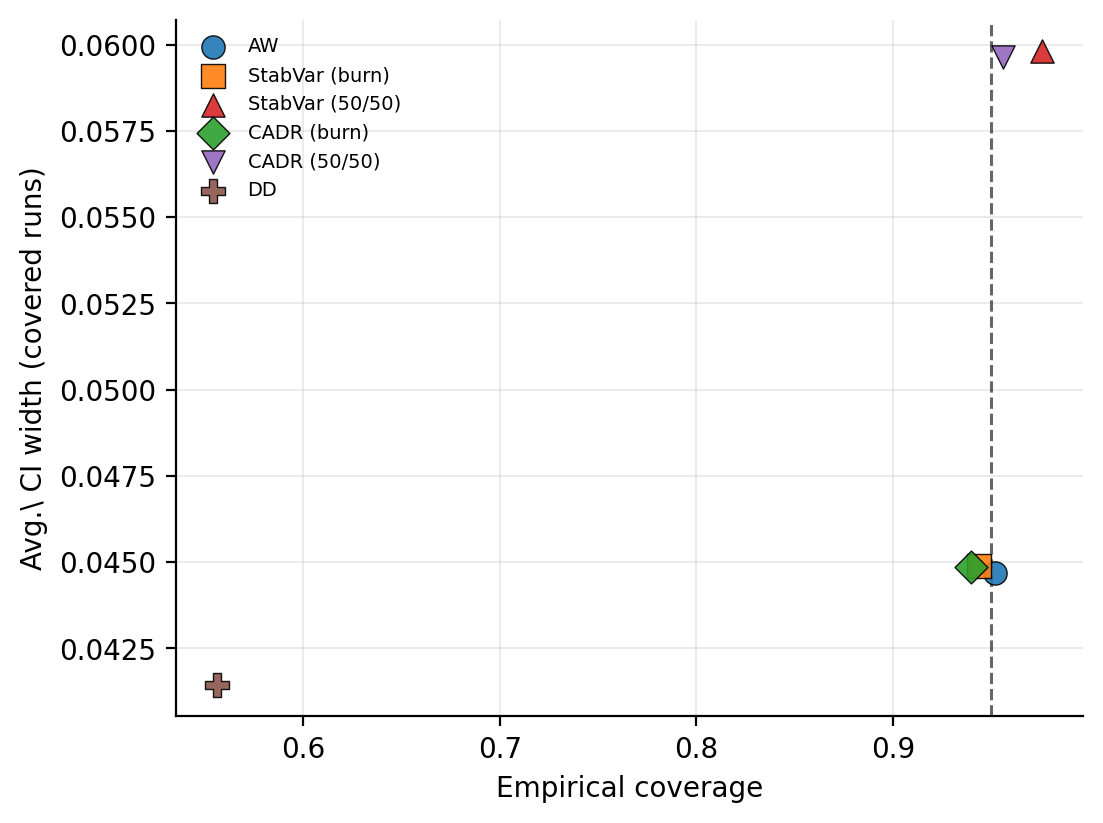}
  \caption{$\alpha = 0.1$.}
\end{subfigure}\hfill

\begin{subfigure}[t]{0.4\textwidth}
  \includegraphics[width=\linewidth]{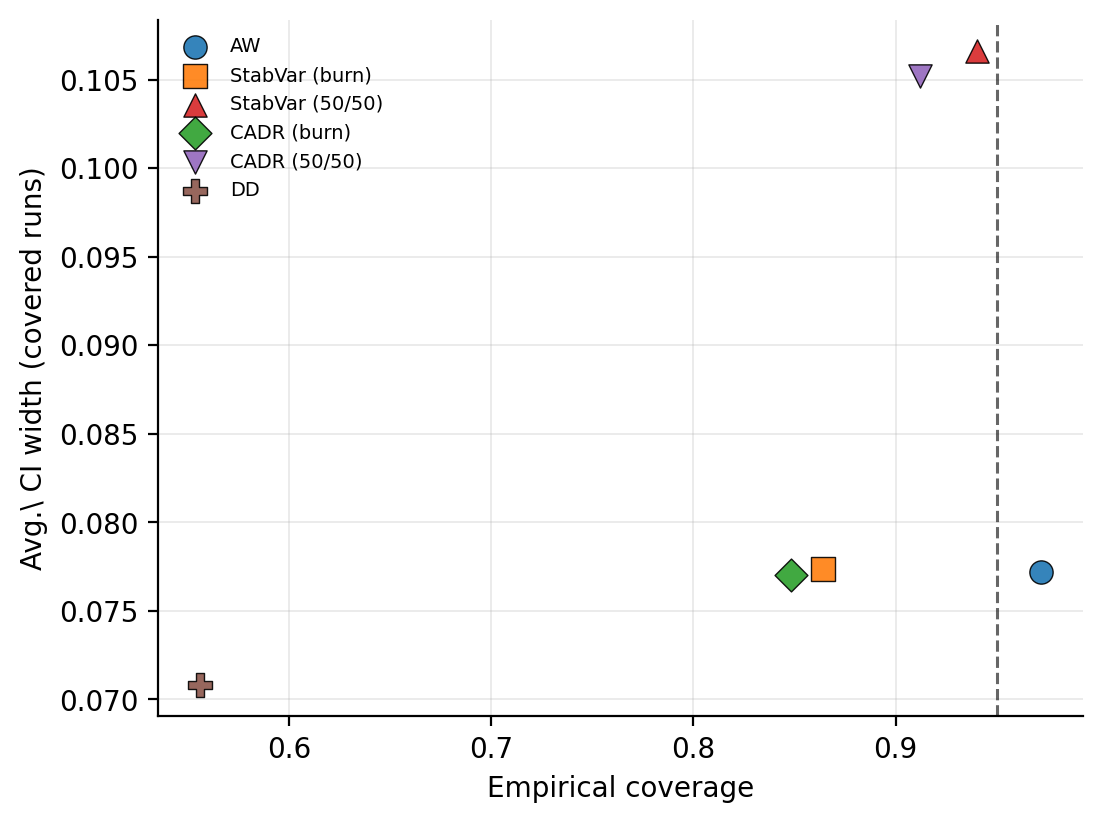}
  \caption{$\alpha = 0.25$.}
\end{subfigure}
\begin{subfigure}[t]{0.4\textwidth}
    \includegraphics[width=\linewidth]{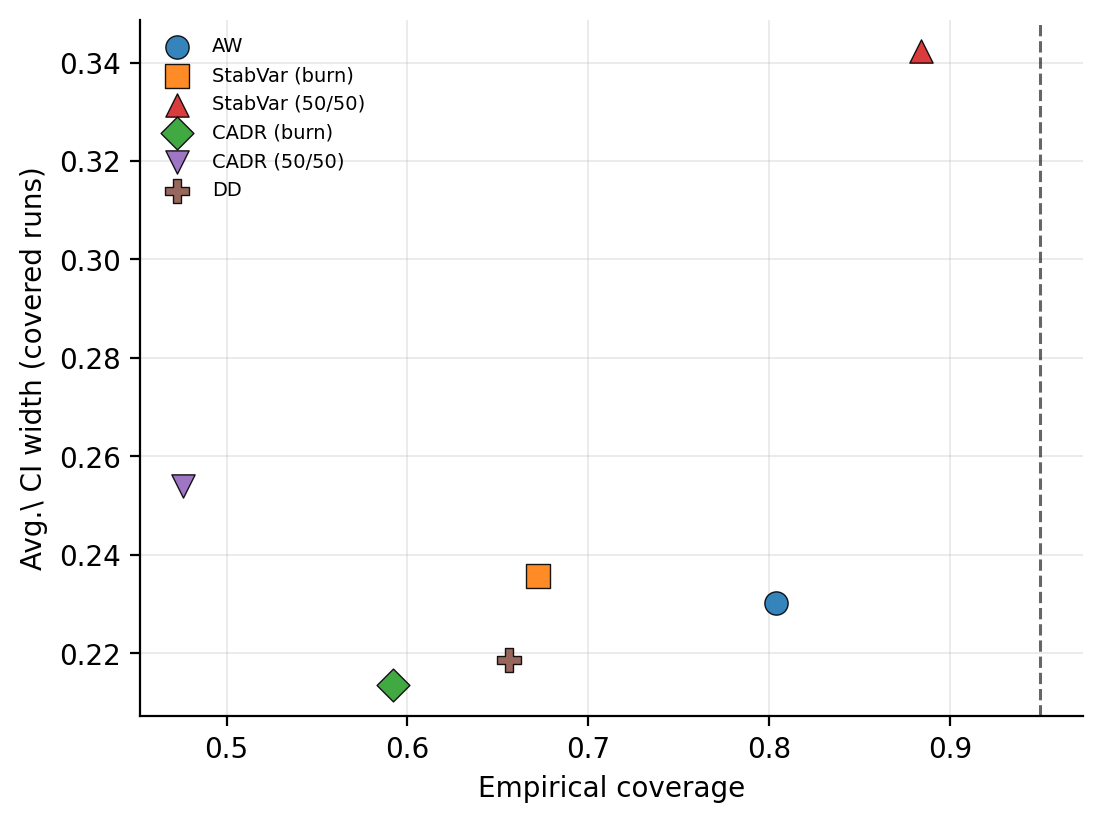}
    \caption{$\alpha = 0.5$.}
\end{subfigure}
\caption{The Pareto frontier of empirical coverage (x-axis) and confidence interval width (y-axis) in the synthetic DGP. We examine this frontier at $n = 15{,}000$ samples and logging exponents
$\alpha \in \{0, 0.1, 0.25, 0.5\}$. Estimators toward the lower right are better calibrated.}
\label{fig:nu-pareto-synth}
\end{figure}

\subsection{Semi-synthetic Experiments}
\label{sec:exp:semi-synthetic}

We now consider a substantially more realistic experiment based on raw user ratings from the Jester
joke-rating dataset of \citet{goldberg2001eigentaste}. The Jester dataset consists of 24,983 rows, each corresponding to a single user considered in the study. Each user was tasked with rating 100 jokes on a scale of -10.0 to 10.0, with the original study retaining the ratings of all units who rated more than 36 jokes. In the sequel, we consider the subset of 14,116 users who successfully rated all 100 jokes. We design a semi-synthetic data generating process on top of these rows based on sampling with replacement, where a user's expected rating of a joke is given by the true ratings in the dataset. We use a small number of jokes as treatments, letting a user's ratings of said jokes represent expected outcomes. A larger subset of jokes serves as covariates. We now describe our DGP in more detail.

\paragraph{DGP, Logging Policies, and Policy Class.}
We treat the five jokes with IDs $\calJ := \{16, 57, 41, 81, 50\}$ as
candidate treatments. We then set $\calA := \calJ \cup \{50'\}$ as the treatment set, where $50'$ is an exact duplicate of arm $50$, i.e., has the same associated (expected) potential outcomes. These jokes are chosen so that no single joke is uniformly preferred over others. We display both marginal and pairwise diagnostics for the selected jokes (i.e., treatments) in Tables~\ref{tab:treatment-summary} and~\ref{tab:treatment-corr} in Appendix~\ref{app:exp:semi-synth}, with the former showing average ratings and win-rates for the chosen jokes, and the latter showing pairwise correlations between joke ratings. We then select 20 of the remaining 95 jokes to serve as covariates. We also detail the exact selection procedure and chosen jokes for covariates in Appendix~\ref{app:exp:semi-synth}.

We generate observations as follows. First, we randomly sample an individual from the 14,116 users retained, and take $X$ to be their ratings of the jokes selected as covariates. We then define their potential outcome associated with treatment $a \in \calA$ to be
\[
Y(a) := \text{``User's raw rating for joke $a$''} + \xi_a, \;\; \text{where } \xi_a \sim \calN(0, 4.0).
\]
We set $Y \equiv Y(A)$, where $A$ is generated from the same logging policy as in the synthetic experiment above, adjusted for the number of interventions. We note that, in the above formulation, although jokes $50$ and $50'$ share the same raw ratings, the added noise is independent.
Again, the policy class considered is all depth 2 policy trees, this time with a minimum of 15\% of the samples landing in any leaf.\footnote{That is, every leaf $L$ has $\P(X \in L) \geq 0.15$.}

\paragraph{Evaluation.}
Because of the larger number of covariates and treatments (20 and 6, respectively) we approximately solve the empirical welfare maximization problem via greedy splitting.  In particular, we leverage the \texttt{PolicyTree} class from the \texttt{EconML} package~\citep{oprescu2019econml}. Our choice to adopt greedy approximations in this section is intentional --- many practitioners will likely use greedy or ML proxies for policy learning in practice. To approximate the optimal value $V^\ast$, we fit the \texttt{PolicyTree} class on the full $14{,}116$-user population without adding any noise.
Thus, although coverage in the semi-synthetic case will not be guaranteed by our theory, the results below suggest that greedy approximations for policy learning can work well in practice.
We consider
$\alpha \in \{0, 0.1, 0.25\}$ and $n \in \{2000, 5000, 8000, 10000, 15000, 20000\}$ with burn-in
periods $b_n$ of $500, 2000, 2500, 3000, 4000$, and $5000$, respectively. Again, we conduct our measurements over $250$
independent runs.

\paragraph{Findings.}
Figures~\ref{fig:nu-cov-semisynth} and \ref{fig:nu-pareto-semisynth} present  analogues of the coverage and Pareto plots shown in the previous subsection. From both plots, we see that \aw{} generally provides the best balance between coverage and confidence interval width. In particular, from these figures, we see that of all estimators that obtain nominal coverage at $n=20,000$ samples (after accounting for Wilson CIs), our estimator obtains the narrowest confidence intervals. The exact coverage of our estimator relative to the 50/50 split estimators based on \zhan{} and \biba{} can be found in Table~\ref{tab:nu-cov-semisynth} in Appendix~\ref{app:exp:semi-synth}. We note that double-dipping (\dd) fails
catastrophically in the present setting: its coverage sits between roughly $0.02$ and $0.10$, with
average estimates of $2.46$--$3.73$ against the truth of $2.197$. Further, this approach shows no sign of
recovering as $n$ grows out to $20{,}000$. We hypothesize that this is likely because cloning the marginally best treatment results in a large space of equally promising policies, so double-dipping is actually capable of overfitting in this setting. Unlike the synthetic study, we do not report  $\alpha = 0.5$ regime here.

\begin{figure}[htbp]
\centering
\includegraphics[width=0.8\linewidth]{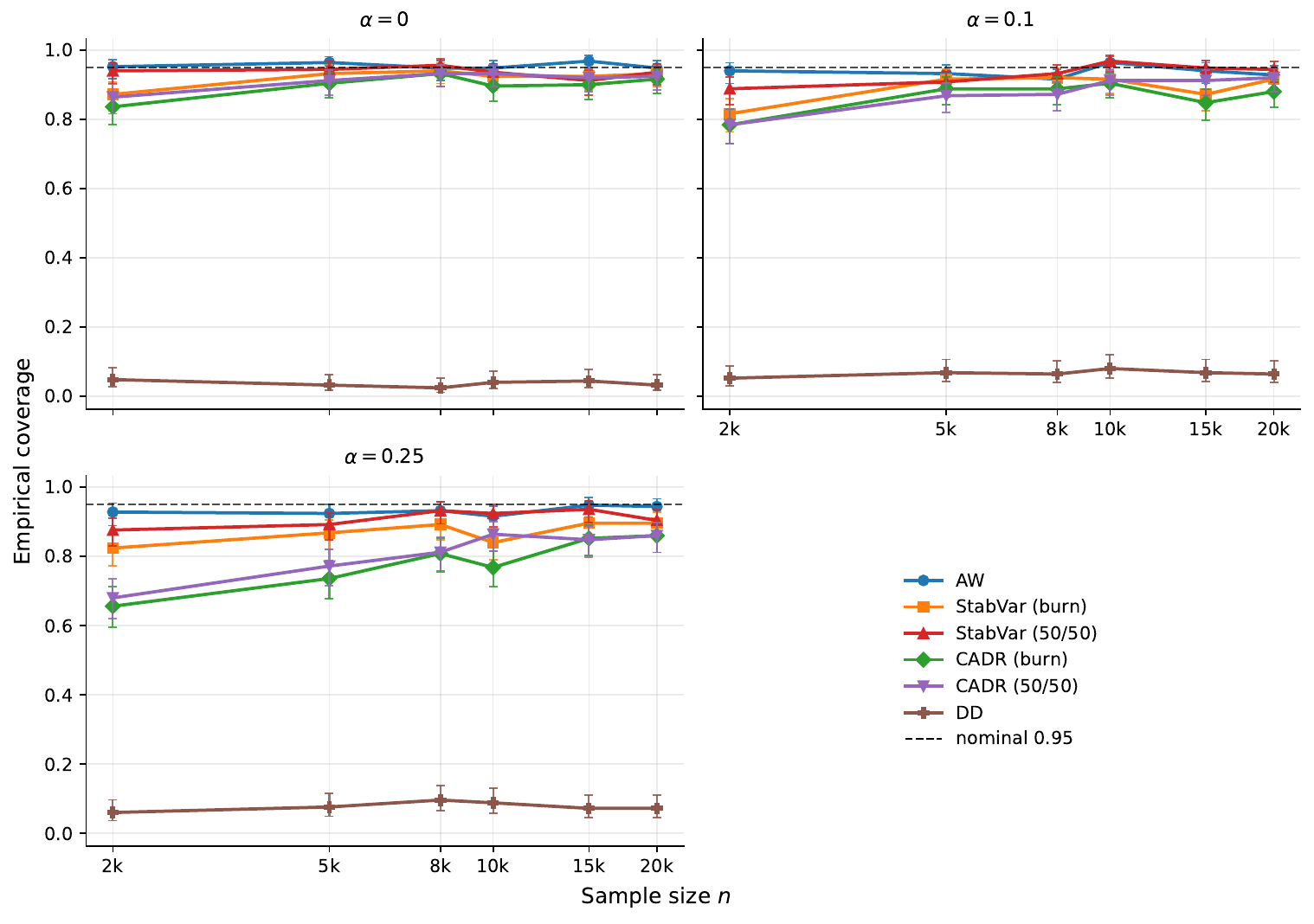}
\caption{Empirical coverage versus sample size in the semi-synthetic DGP for all six
estimators at logging exponents $\alpha \in \{0, 0.1, 0.25\}$. Error bars are Wilson $95\%$
intervals and the dashed line is the nominal $0.95$ coverage target.}
\label{fig:nu-cov-semisynth}
\end{figure}

\begin{figure}[htbp]
\centering
\begin{subfigure}[t]{0.32\textwidth}
  \includegraphics[width=\linewidth]{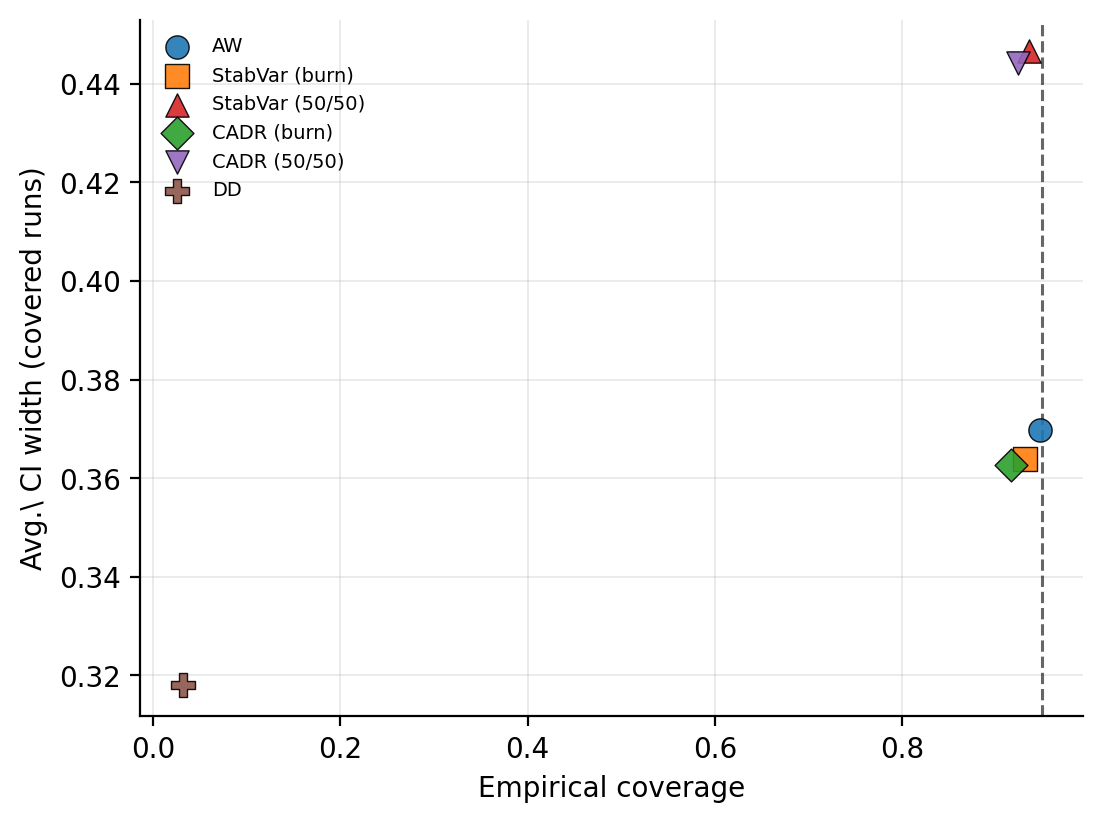}
  \caption{$\alpha = 0$.}
\end{subfigure}\hfill
\begin{subfigure}[t]{0.32\textwidth}
  \includegraphics[width=\linewidth]{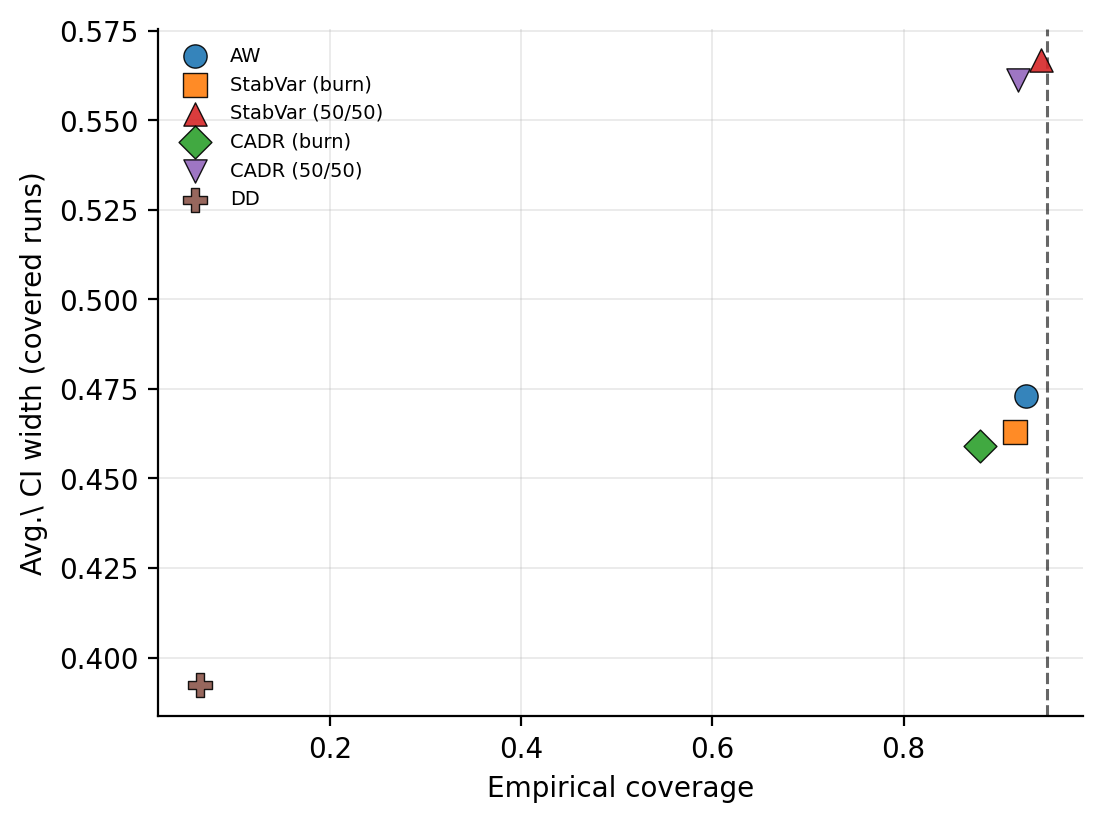}
  \caption{$\alpha = 0.1$.}
\end{subfigure}\hfill
\begin{subfigure}[t]{0.32\textwidth}
  \includegraphics[width=\linewidth]{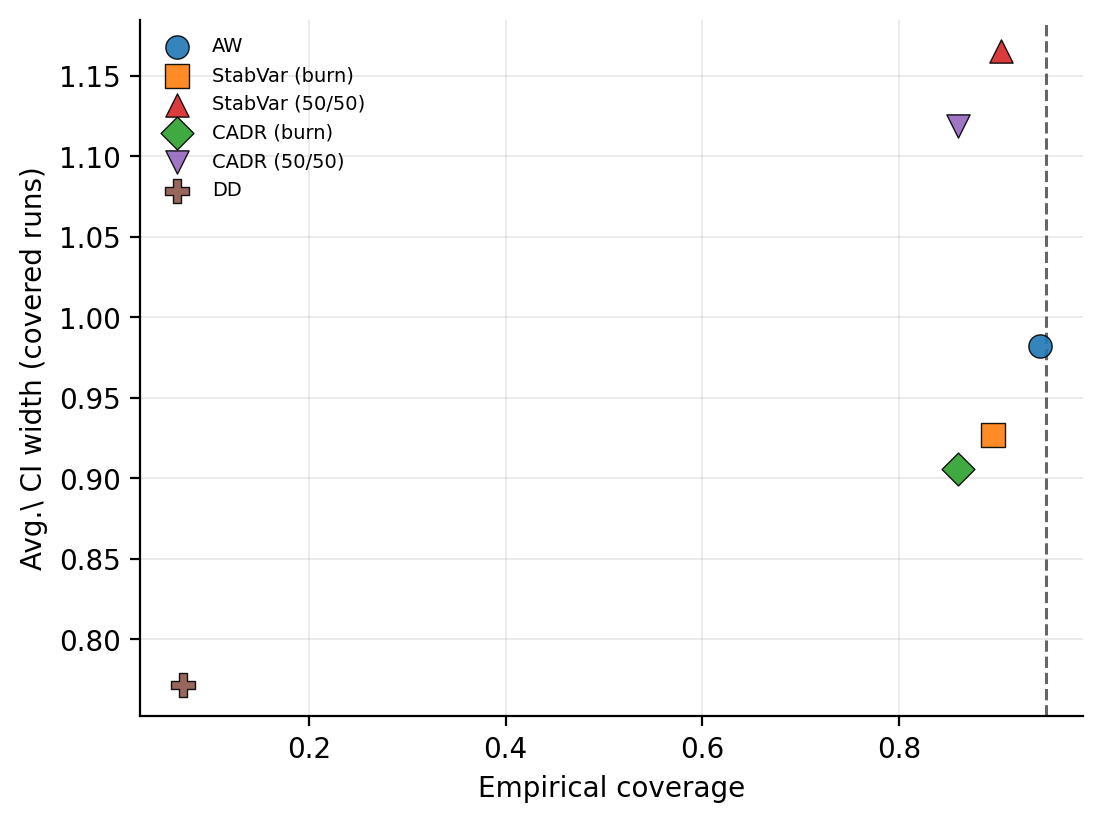}
  \caption{$\alpha = 0.25$.}
\end{subfigure}
\caption{The Pareto frontier of empirical coverage (x-axis) and confidence interval width (y-axis) in the semi-synthetic DGP. Here, we consider $n = 20{,}000$ samples and logging exponents
$\alpha \in \{0, 0.1, 0.25\}$.}
\label{fig:nu-pareto-semisynth}
\end{figure}

\newpage

\subsection{Regret Scaling}
\label{sec:exp:regret}

In Section~\ref{sec:regret}, we proved that the expected regret satisfies $\E[\Reg(\wh{\pi})] = o(n^{(\alpha - 1)/2})$. This was sufficient to argue that the bias incurred via repeated policy fits was negligible, but did not tell us \textit{how much faster} regret decays than the width of the confidence intervals.
Here, we empirically show that the scaled regret $\frac{\Reg(\wh\pi)}{n^{(\alpha-1)/2}}$ tends to zero as $n$ grows large. We will see, however, that the scaled regret decays more slowly for large logging exponents $\alpha$, partially explaining why seemingly all estimators undercover when $\alpha = 0.5$ in the synthetic DGP.

\paragraph{Evaluation.} We consider both DGPs in this section. For $n \in \{2000, 5000,
10000, 15000, 25000\}$ and $\alpha \in \{0.0, 0.1, 0.25, 0.5\}$, we generate $n$ samples from each DGP using the above logging policies. We then solve (either exactly in the synthetic case or approximately in the semi-synthetic case) the weighted empirical welfare maximization on all $n$ samples, per Equation~\eqref{eq:ewm}. For the synthetic DGP, we estimate the true optimal policy value by sampling $300,000$ \textit{fully-observed} samples, computing the estimate \textit{without} inverse propensity weighting. For the semi-synthetic DGP, we just average over the 14,116 samples in the Jester dataset, which are treated as ground truth. We repeat our measurements over $M = 100$ runs for each $(n, \alpha, \texttt{DGP})$ triple.

\paragraph{Findings.} We report our findings in Figure~\ref{fig:scaled_regret}, which displays the decay of scaled regret for both the synthetic and semi-synthetic DGPs.
Under both DGPs, the scaled regret seemingly tends towards zero as the sample size grows. For larger logging exponents (e.g., $\alpha = 0.5$), the rate of convergence seems to be significantly slower. We note that regret even seems to vanish as a function of the sample size in the semi-synthetic DGP, which is interesting as we are not solving the empirical welfare maximization problem exactly.
Note that the convergence of the scaled regret in turn implies that the raw regret itself vanishes, showing that multiplicity of optimal policies does not impact the possibility of policy learning.

\begin{figure}[htbp]
\centering
\begin{subfigure}[t]{0.48\textwidth}
  \includegraphics[width=\linewidth]{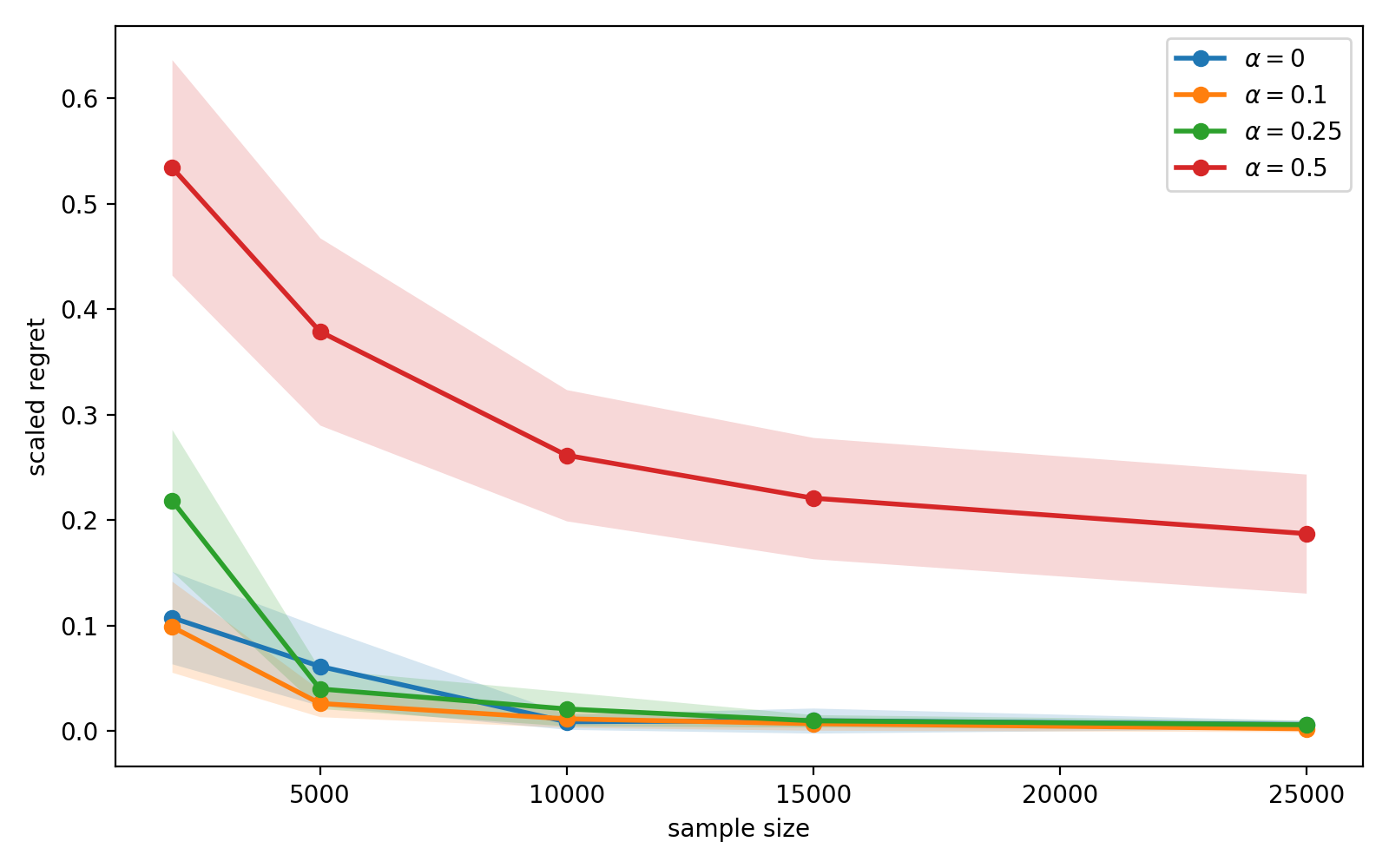}
  \caption{Synthetic DGP.}
\end{subfigure}\hfill
\begin{subfigure}[t]{0.48\textwidth}
  \includegraphics[width=\linewidth]{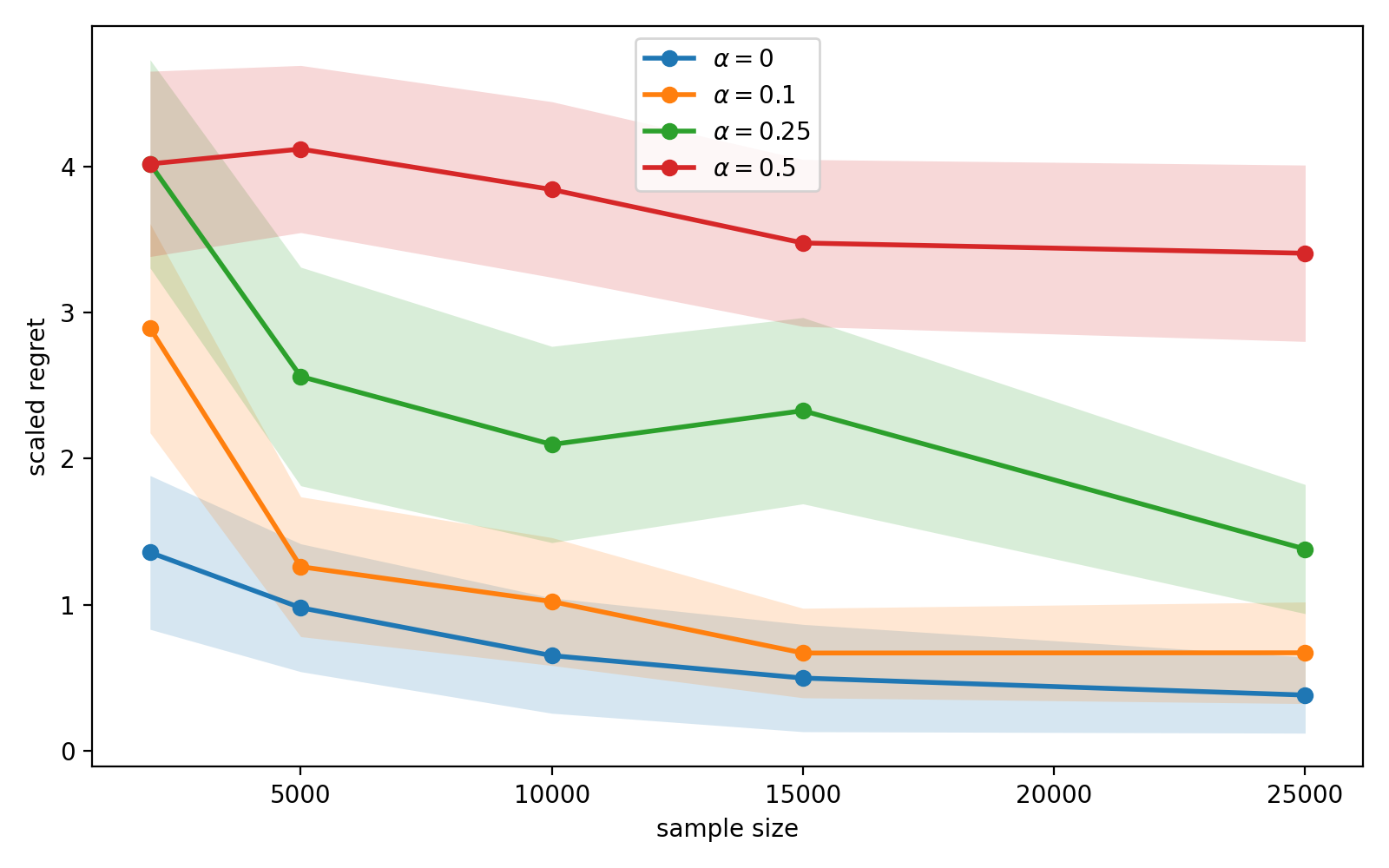}
  \caption{Semi-Synthetic DGP.}
\end{subfigure}
\caption{A visualization of scaled regret against $n$ under various logging exponents $\alpha$. The shaded region gives a $95\%$ pointwise confidence interval over $M = 100$ independent runs. The synthetic panel leverages an exact solver, whereas the semi-synthetic one uses an approximate solution.}
\label{fig:scaled_regret}
\end{figure}

\section{Application: Revisiting Job Search Assistance for Refugees in Jordan}
\label{sec:case_study}

We now apply our estimator in estimating the value of policy improvement in an adaptive field experiment conducted by \citet{caria2024adaptive}. In this experiment, a mix of Syrian refugees (1,663 units) and low-skilled Jordanian citizens (2,107 units) aged 18--45 were chosen to participate in a job search assistance program. These units were \textit{adaptively} assigned to one of four treatments on a rolling basis between February 10th and November 30th of 2019 using a \textit{tempered} Thompson sampling algorithm, formally described below. The four treatments are referred to as \textit{control}, \textit{labeled cash transfer} (or \textit{cash} for short), \textit{information}, and \textit{nudge}, and are described in Table~\ref{tab:jordan_treatments}. As before, we let $\calA$ denote the set of treatments throughout. The outcome of interest in the experiment is a binary indicator of employment status, which was measured at roughly six weeks, two months, and four months following treatment assignment.\footnote{As noted by the authors, overall attrition levels in the experiment were low, being 1.2\% at six weeks, 1.6\% at two months, and 4.2\% at four months.}

\begin{table}[htbp]
    \centering
    \caption{Interventions considered in the job assistance program.}
    \label{tab:jordan_treatments}
    \begin{tabular}{p{0.19\textwidth} p{0.72\textwidth}}
        \toprule
        \textbf{Treatment arm} & \textbf{Description} \\
        \midrule

        Control
        &
        Participants received a small cash infusion of 4 Jordanian dinars (approximately US \$5.60) and a basic informational pamphlet.
        \\
        \midrule

        Cash
        &
        Participants received a larger cash transfer of 65 Jordanian dinars (approximately US \$92).
        \\
        \midrule

        Information
        &
        Participants received more extensive information about preparing for job interviews, including in-person instruction.
        \\
        \midrule

        Nudge
        &
        Participants received a behavioral job-search intervention consisting of a four-week planning calendar, an instructional video, face-to-face guidance, and optional SMS reminders.
        \\

        \bottomrule
    \end{tabular}
\end{table}

One of the quantities estimated in the authors' work is a Bayesian equivalent of the VPI, which measures the gap between the value of the optimal covariate-dependent treatment policy and the policy that assigns all units to control. In more detail, the authors posit a hierarchical Bayesian model for potential outcomes and define a natural analogue of the VPI in terms of this model (we define both in the sequel).
They find that the Bayesian VPI  is 1.7 pp (percentage points) at the six week point, with a 95\% credible set excluding zero. Further, using their Bayesian method, we compute even higher effects at two months (4.4 pp) and four months (2.7 pp) following treatment, also with credible intervals excluding zero.
Our application of their methods points to a high VPI amongst Syrian refugees at the two month mark, with a point estimate of 8.3 pp, indicating reasonable effectiveness of the job search program.

However, implicit in their model (and explicit in their paper) is a strong assumption: the optimal treatment is \textit{almost surely unique} for each covariate stratum. While the authors note in their supplement that such a Bayesian model does not suffer from a winner's curse, this is only true if the assumed model is correct. Under the assumed hierarchical Bayesian model, one can compute credible sets for the effect of each treatment relative to control. The authors do this at the six week mark in Figure A.8 of their supplement, finding that all credible sets contain zero. The same figure shows this holds when conditioning on whether the participant was Jordanian or Syrian as well. Further, Figure A.9 of their supplement shows that credible sets for all expected potential outcomes at the same horizon overlap. Therefore, in this dataset there is not yet evidence that the optimal policy is unique.
Given this, we seek to evaluate if similar benefits for intervening are estimated when
using an approach that is robust to non-unique optima.

\paragraph{Assumed Bayesian Model and Estimand.}
We now briefly recount the assumed Bayesian hierarchical prior and targeting effect estimand from \citet{caria2024adaptive}. First, we introduce notation. We again denote the set of treatments as $\calA$, using the notation $0$ for the control treatment. The considered covariate strata are given by the following Cartesian product:
\[
\calX := \texttt{\{Syrian, Jordanian\}}\times\texttt{\{F, M\}}\times\texttt{\{HS, No HS\}}\times\texttt{\{Never employed, Ever employed\}},
\]
where \texttt{M} and \texttt{F} denote \textit{male} and \textit{female}, respectively, and \texttt{HS} denotes high school education.
Days are indexed by $t \in \calT$, $n_t$ denotes the number of units treated on day $t$, and $i \in [n_t]$ indexes units within a day. Each unit is associated with covariates $X_{i, t}$ and potential employment outcomes $Y_{i, t}(a) \in \{0, 1\}$, which are assumed to be generated by the following hierarchical prior:
\begin{gather*}
(\eta_a, \beta_a) \sim \nu, \;\;\forall a \in \calA \quad \text{where } \nu(\eta, \beta) \propto (\eta + \beta)^{-2.5}\\
\mu_{a, x} \mid (\eta_a, \beta_a) \sim \Beta(\eta_a, \beta_a), \;\; \forall (a, x) \in \calA \times \calX \\
Y_{i, t}(a) \mid \big(X_{i, t} = x, \mu_{a, x}, \eta_a, \beta_a\big) \sim \Bern(\mu_{a, x}) \;\; \forall i \in [n_t], t \in \calT, x \in \calX, a \in \calA.
\end{gather*}
We let $A_{i, t} \in \calA$ denote the realized treatment for said unit, and $Y_{i, t} = Y_{i, t}(A_{i, t})$ denote the realized employment status at the desired followup horizon. We fully describe the tempered Thompson sampling algorithm, which ensures each treatment has at least a 5\% chance of being assigned in each round, in Appendix~\ref{app:exp:case_study}.

Let $\calD := \left\{(X_{i, t}, A_{i, t}, Y_{i, t}) : i \in [n_t], t \in \calT\right\}$ denote the entire dataset of observations. The authors choose their estimand to be the gap between the value of the policy that maximizes per-stratum expected posterior outcomes and the value of the all control policy. Letting $p_x$ denote the fraction of the sample with $X_{i, t} = x$, this estimand is given by
\[
\theta_{0, B} := \sum_{x \in \calX}p_x\left[\max_{a \in \calA}\E\left(\mu_{a, x} \mid \calD\right) - \E\left(\mu_{0, x} \mid \calD\right)\right].
\]
This estimand is different than the frequentist VPI outlined in Equation~\eqref{eq:vpi}. In particular, the frequentist quantity averages over the distribution of covariates and views the per-stratum expected outcomes as \textit{fixed}, not random. The Bayesian analogue instead uses empirical stratum weights and averages expected potential outcomes over the posterior distribution. In our analyses, we will compare our estimates of the \textit{frequentist} VPI to this Bayesian analogue.

\paragraph{Our Measurement Approach.}
We briefly describe our measurement approach. We run our analysis on the released set of 3,709 participants.\footnote{We note the entire experiment had 3,770 participants.} We detail the exact data sources used in Appendix~\ref{app:exp:case_study}. Our analyses at two and four months also drop units with missing outcomes. We apply our estimator with batched refitting as in Algorithm~\ref{alg:online-epoch}. Instead of taking geometric epochs as in our previous experiments, we opt to refit policies at the start of each day. This is particularly natural in the setting of this dataset as on any day multiple units receive treatment. We use a burn-in period that ensures at least 20\% of the data is used for fitting the first policy. Table~\ref{tab:burn20-cutoffs} in Appendix~\ref{app:exp:case_study} displays precisely how many samples were used in constructing the initial policy, the corresponding percentage of overall data used, and the first date of policy evaluation.

We take the policy class $\Pi$ to be the set of all deterministic policies excluding the all control policy, which has finite cardinality due to $\calX$ being finite. The baseline policy $q(a \mid x)$ is taken to be the all control policy, as is done in \citet{caria2024adaptive}. We note that this setting is permissible by Lemma~\ref{lem:suff_conditions}, as $\gamma = 0.2$ ensures that each intervention has probability at least $0.05$ of being assigned in any round, and hence we have $\alpha = 0$.

We measure the overall value of policy improvement at the end of the experiment using the realized outcomes at the six week, two month, and four month marks. We compare our estimates with the Bayesian ones from \citet{caria2024adaptive}, which we compute using the publicly available data and Markov Chain Monte Carlo algorithm (written in pseudo-code in Algorithm~\ref{alg:caria-targeting-short} in Appendix~\ref{app:exp:case_study}). We compare these values in Table~\ref{tab:main-burn20}, and further present estimates broken down by nationality in Table~\ref{tab:subgroup-burn20}.

\paragraph{Our Findings.}
While the Bayesian method leveraged by~\citet{caria2024adaptive} yields a positive VPI at six weeks, two months, and four months following treatment assignment, our full-population confidence intervals for the VPI contain zero at all time horizons.\footnote{We again note that the two and four month estimates are not present in the original paper. We compute them using the authors' methods.} We note that our point estimates, which are detailed in Table~\ref{tab:main-burn20}, are all around 1.0--3.0 pp lower than their Bayesian counterparts. While our approaches are not directly comparable, we believe there are two competing forces at play. First, as noted above, the Bayesian approach assumes unique optimal treatments per stratum, and so would suffer from a winner's curse were said uniqueness  to be violated. Further, due to the fact our approach leverages a policy learning subroutine, it generally yields conservative estimates in finite samples.
We emphasize that our findings \textit{do not preclude} the possibility of a positive VPI.

One interesting finding that our estimates do corroborate is that, at the two month mark, there \textit{is a significant benefit} for targeting amongst Syrians. We estimate that the value of policy improvement is around 7.12 pp, with a lower confidence bound of 1.25 pp. This is only mildly lower than the 8.30 pp estimate from the Bayesian procedure.
Lastly, we note that although the two estimators construct treatment policies in very different ways, the decisions made by the policies are broadly in alignment. We compare the \textit{final} learned policies from our adaptive estimator to the policy that assigns the treatment maximizing posterior expected outcomes in each stratum in Appendix~\ref{app:exp:case_study}, finding that at the two month mark (i.e., when the VPI/effect of targeting was highest), the cash treatment was more often than not the most effective option across many strata.

\begin{table}[ht]
\centering
\caption{Our frequentist \aw{} VPI estimate, and the Bayesian estimates (which have formal properties under the assumption of a unique optimal policy~\citep{caria2024adaptive}) alongside 95\% credible/confidence intervals. Here, the baseline is the ``all control'' policy. Estimates of \aw{} are produced using roughly 20\% of the sample for burn-in.}
\label{tab:main-burn20}
\begin{tabular}{lccc}
\toprule
Estimator & 6 weeks & 2 months & 4 months \\
\midrule
AW & 0.26 [-1.98, 2.51] & 2.09 [-1.53, 5.71] & -0.07 [-2.65, 2.50] \\
Bayes & 1.70 [0.07, 3.33] & 4.42 [1.97, 6.91] & 2.73 [0.59, 4.93] \\
\bottomrule
\end{tabular}

\end{table}

\begin{table}[ht]
\centering
\caption{VPI estimates restricted to the Jordanian and Syrian subpopulations alongside 95\% credible/confidence intervals.}
\label{tab:subgroup-burn20}
\begin{tabular}{llccc}
\toprule
Group & Estimator & 6 weeks & 2 months & 4 months \\
\midrule
Jordanian & AW & -0.05 [-2.99, 2.90] & -1.76 [-6.14, 2.62] & -2.46 [-5.51, 0.60] \\
Jordanian & Bayes & 1.94 [-0.39, 4.27] & 1.83 [-0.92, 4.72] & 1.01 [-2.06, 4.21] \\
Syrian & AW & 0.33 [-2.69, 3.34] & 7.12 [1.25, 12.99] & 3.36 [-0.74, 7.46] \\
Syrian & Bayes & 1.49 [-0.73, 3.87] & 8.30 [4.08, 12.61] & 4.97 [1.46, 8.74] \\
\bottomrule
\end{tabular}

\end{table}

\section{Conclusion}
\label{sec:conclusion}

In this paper, we developed a simple, sequential estimator for the value of policy improvement under adaptive experimentation. Our estimator, which was based on an estimator due to \citet{luedtke2016statistical}, works by repeatedly using prefixes of historical data to train a policy using a weighted empirical welfare maximization objective, then evaluating the learned policy on subsequent samples. We showed it could operate under either continual refitting or batched fitting, in which the EWM policy was only retrained after a geometrically growing number of steps. We  proved that one can construct asymptotically valid confidence intervals for the VPI, and as part of this,  showed that one can obtain marginally faster pointwise expected-regret rates for weighted empirical welfare maximization. In addition,
we evaluated the coverage of our estimator and corresponding confidence interval width across both synthetic and semi-synthetic data generating processes. We found that, when compared to natural sample-splitting and double-dipping baselines, our estimator generally offered the best tradeoff between coverage and confidence interval width. We also note that the convergence of the sample-splitting estimators requires our faster regret rates, and thus their validity can be viewed as an additional contribution of this work.
Finally, we applied our estimator to an adaptive experiment evaluating job-assistance programs in Jordan, obtaining generally more conservative estimates than the existing Bayesian analysis whose formal properties rely on the existence of a unique optimal policy (for which there is little evidence) while corroborating the benefits of personalization among Syrian refugees at the two-month horizon.

There are many natural open directions one could pursue related to the work in this paper.
One interesting direction would be to study the validity of \textit{smoothing-based} estimators for the value of policy improvement.
Smoothing-based estimators, such as those considered in \citet{whitehouse2025inference} and \citet{chen2023inference}, have recently emerged as compelling alternatives to the online estimator of \citet{luedtke2016statistical}, as they are not only more intuitive (replacing a non-smooth objective by a differentiable alternative) and provide canonical tie-breaking rules, but also require only a constant number of policy fits.
However, existing work on these estimators considers the setting of \textit{unrestricted policy classes}, wherein any measurable map of a unit's covariates is permitted as a treatment policy. Extending these estimators to a setting of restricted policy classes is non-trivial, and thus would require significant theoretical development. Further, it is unclear how to leverage such estimators alongside contextual bandit data.

Another natural route would be to develop estimators and confidence intervals directly for the \textit{value of an optimal treatment policy} within a class. Our focus on the value of policy improvement was deliberate: not only did we find the estimand to be more naturally motivated, but the contrast with a randomized baseline also served the technical purpose of stabilizing the variance of each step of IPW off-policy evaluation. Without this added stability, one could still show that our confidence interval covers a \textit{weighted average of the values of learned policies}, but arguing that the bias from policy learning (i.e., accumulated regret) vanishes sufficiently quickly is non-obvious. Likely, the underlying estimator would not need to change, but one would need to develop more refined, self-normalized regret bounds that naturally reflect the accumulated variance of off-policy evaluation rather than just the overall number of samples. It is unclear if such improved rates are possible when there are multiple optimal treatment policies within a class, and we thus leave this for future work. We again emphasize that this added stability from the contrast is not needed in the setting $\alpha = 0$, and thus our analysis also yields confidence intervals for the value of the optimal treatment policy in a class in this setting.

\paragraph{AI Usage Statement}
The authors used generative AI tools to check the correctness of theorem statements, identify and fix errors in proofs, write code for experimentation, produce initial figures and tables, and copyedit. The authors verified all AI-assisted output and remain responsible for the content of this manuscript.

\bibliography{bib}

@article{ponomarev2024lower,
  title={On the lower confidence band for the optimal welfare},
  author={Ponomarev, Kirill and Semenova, Vira},
  journal={arXiv preprint arXiv:2410.07443},
  year={2024}
}

@inproceedings{luedtke2020performance,
  title={Performance guarantees for policy learning},
  author={Luedtke, Alex and Chambaz, Antoine},
  booktitle={Annales de l'IHP Probabilites et statistiques},
  volume={56},
  number={3},
  pages={2162},
  year={2020}
}

@article{luedtke2016statistical,
  title={Statistical inference for the mean outcome under a possibly non-unique optimal treatment strategy},
  author={Luedtke, Alexander R and Van Der Laan, Mark J},
  journal={Annals of statistics},
  volume={44},
  number={2},
  pages={713},
  year={2016}
}

@article{whitehouse2025inference,
  title={Inference on optimal policy values and other irregular functionals via smoothing},
  author={Whitehouse, Justin and Austern, Morgane and Syrgkanis, Vasilis},
  journal={arXiv preprint arXiv:2507.11780},
  year={2025}
}

@article{chen2023inference,
  title={Inference on optimal dynamic policies via softmax approximation},
  author={Chen, Qizhao and Austern, Morgane and Syrgkanis, Vasilis},
  journal={arXiv preprint arXiv:2303.04416},
  year={2023}
}

@article{chernozhukov2013intersection,
  title={Intersection bounds: Estimation and inference},
  author={Chernozhukov, Victor and Lee, Sokbae and Rosen, Adam M},
  journal={Econometrica},
  volume={81},
  number={2},
  pages={667--737},
  year={2013},
  publisher={Wiley Online Library}
}

@article{shi2020breaking,
  title={Breaking the Curse of Nonregularity with Subagging---Inference of the Mean Outcome under Optimal Treatment Regimes},
  author={Shi, Chengchun and Lu, Wenbin and Song, Rui},
  journal={Journal of Machine Learning Research},
  volume={21},
  number={176},
  pages={1--67},
  year={2020}
}

@techreport{van2014probability,
  title={Probability in high dimension},
  author={Van Handel, Ramon},
  year={2014}
}

@article{fang2019inference,
  title={Inference on directionally differentiable functions},
  author={Fang, Zheng and Santos, Andres},
  journal={The Review of Economic Studies},
  volume={86},
  number={1},
  pages={377--412},
  year={2019},
  publisher={Oxford University Press}
}

@article{carcamo2020directional,
  title={Directional differentiability for supremum-type functionals},
  author={C{\'a}rcamo, Javier and Cuevas, Antonio and Rodr{\'\i}guez, Luis-Alberto},
  journal={Bernoulli},
  volume={26},
  number={3},
  pages={2143--2175},
  year={2020},
  publisher={JSTOR}
}

@incollection{van1996weak,
  title={Weak convergence},
  author={Van Der Vaart, Aad W and Wellner, Jon A},
  booktitle={Weak convergence and empirical processes: with applications to statistics},
  pages={16--28},
  year={1996},
  publisher={Springer}
}

@book{kosorok2008introduction,
  title={Introduction to empirical processes and semiparametric inference},
  author={Kosorok, Michael R},
  year={2008},
  publisher={Springer}
}

@misc{chernozhukov2018double,
  title={Double/debiased machine learning for treatment and structural parameters},
  author={Chernozhukov, Victor and Chetverikov, Denis and Demirer, Mert and Duflo, Esther and Hansen, Christian and Newey, Whitney and Robins, James},
  year={2018},
  publisher={Oxford University Press Oxford, UK}
}

@article{chernozhukov2023high,
  title={High-dimensional data bootstrap},
  author={Chernozhukov, Victor and Chetverikov, Denis and Kato, Kengo and Koike, Yuta},
  journal={Annual Review of Statistics and Its Application},
  volume={10},
  number={1},
  pages={427--449},
  year={2023},
  publisher={Annual Reviews}
}

@article{chernozhukov2013gaussian,
  title={Gaussian approximations and multiplier bootstrap for maxima of sums of high-dimensional random vectors},
  author={Chernozhukov, Victor and Chetverikov, Denis and Kato, Kengo},
  journal={The Annals of Statistics},
  pages={2786--2819},
  year={2013},
  publisher={JSTOR}
}

@article{bibaut2020sufficient,
  title={Sufficient and insufficient conditions for the stochastic convergence of {C}es{\`a}ro means},
  author={Bibaut, Aur{\'e}lien F and Luedtke, Alex and van der Laan, Mark J},
  journal={arXiv preprint arXiv:2009.05974},
  year={2020}
}

@article{chambaz2017targeted,
  title={Targeted sequential design for targeted learning inference of the optimal treatment rule and its mean reward},
  author={Chambaz, Antoine and Zheng, Wenjing and van der Laan, Mark J},
  journal={Annals of statistics},
  volume={45},
  number={6},
  pages={2537},
  year={2017}
}

@article{bibaut2021post,
  title={Post-contextual-bandit inference},
  author={Bibaut, Aur{\'e}lien and Dimakopoulou, Maria and Kallus, Nathan and Chambaz, Antoine and van Der Laan, Mark},
  journal={Advances in neural information processing systems},
  volume={34},
  pages={28548--28559},
  year={2021}
}

@article{hadad2021confidence,
  title={Confidence intervals for policy evaluation in adaptive experiments},
  author={Hadad, Vitor and Hirshberg, David A and Zhan, Ruohan and Wager, Stefan and Athey, Susan},
  journal={Proceedings of the national academy of sciences},
  volume={118},
  number={15},
  pages={e2014602118},
  year={2021},
  publisher={National Academy of Sciences}
}

@inproceedings{yao2021power,
  title={Power constrained bandits},
  author={Yao, Jiayu and Brunskill, Emma and Pan, Weiwei and Murphy, Susan and Doshi-Velez, Finale},
  booktitle={Machine Learning for Healthcare Conference},
  pages={209--259},
  year={2021},
  organization={PMLR}
}

@inproceedings{zhan2021off,
  title={Off-policy evaluation via adaptive weighting with data from contextual bandits},
  author={Zhan, Ruohan and Hadad, Vitor and Hirshberg, David A and Athey, Susan},
  booktitle={Proceedings of the 27th ACM SIGKDD Conference on Knowledge Discovery \& Data Mining},
  pages={2125--2135},
  year={2021}
}

@article{zhan2024policy,
  title={Policy learning with adaptively collected data},
  author={Zhan, Ruohan and Ren, Zhimei and Athey, Susan and Zhou, Zhengyuan},
  journal={Management Science},
  volume={70},
  number={8},
  pages={5270--5297},
  year={2024},
  publisher={INFORMS}
}

@book{boucheron2003concentration,
  title={Concentration Inequalities: A Nonasymptotic Theory of Independence},
  author={Boucheron, St{\'e}phane and Lugosi, G{\'a}bor and Massart, Pascal},
  year={2013},
  publisher={Oxford University Press}
}

@article{freedman1975tail,
  title={On tail probabilities for martingales},
  author={Freedman, David A},
  journal={the Annals of Probability},
  pages={100--118},
  year={1975},
  publisher={JSTOR}
}

@article{howard2020time,
  title={Time-uniform Chernoff bounds via nonnegative supermartingales},
  author={Howard, Steven R and Ramdas, Aaditya and McAuliffe, Jon and Sekhon, Jasjeet},
  year={2020}
}

@article{bibaut2021risk,
  title={Risk minimization from adaptively collected data: Guarantees for supervised and policy learning},
  author={Bibaut, Aur{\'e}lien and Kallus, Nathan and Dimakopoulou, Maria and Chambaz, Antoine and van Der Laan, Mark},
  journal={Advances in neural information processing systems},
  volume={34},
  pages={19261--19273},
  year={2021}
}

@article{ganssler1978central,
  title={On central limit theorems for martingale triangular arrays},
  author={G{\"a}nssler, Peter and Strobel, J and Stute, Winfried},
  journal={Acta Mathematica Hungarica},
  volume={31},
  number={3-4},
  pages={205--216},
  year={1978},
  publisher={Akad{\'e}miai Kiad{\'o}, co-published with Springer Science+ Business Media BV~…}
}

@book{hall2014martingale,
  title={Martingale limit theory and its application},
  author={Hall, Peter and Heyde, Christopher C},
  year={2014},
  publisher={Academic press}
}

@article{brown1971martingale,
  title={Martingale central limit theorems},
  author={Brown, Bruce M},
  journal={The Annals of Mathematical Statistics},
  pages={59--66},
  year={1971},
  publisher={JSTOR}
}

@article{daniely2015multiclass,
  title={Multiclass learnability and the ERM principle},
  author={Daniely, Amit and Sabato, Sivan and Ben-David, Shai and Shalev-Shwartz, Shai},
  journal={The Journal of Machine Learning Research},
  volume={16},
  number={1},
  pages={2377--2404},
  year={2015},
  publisher={JMLR. org}
}

@article{natarajan1989learning,
  title={On learning sets and functions},
  author={Natarajan, Balas K},
  journal={Machine Learning},
  volume={4},
  number={1},
  pages={67--97},
  year={1989},
  publisher={Springer}
}

@article{haussler1995generalization,
  title={A generalization of Sauer's lemma},
  author={Haussler, David and Long, Philip M},
  journal={Journal of Combinatorial Theory, Series A},
  volume={71},
  number={2},
  pages={219--240},
  year={1995},
  publisher={Elsevier}
}

@incollection{tewari2017ads,
  title={From ads to interventions: Contextual bandits in mobile health},
  author={Tewari, Ambuj and Murphy, Susan A},
  booktitle={Mobile health: sensors, analytic methods, and applications},
  pages={495--517},
  year={2017},
  publisher={Springer}
}

@inproceedings{li2010contextual,
  title={A contextual-bandit approach to personalized news article recommendation},
  author={Li, Lihong and Chu, Wei and Langford, John and Schapire, Robert E},
  booktitle={Proceedings of the 19th international conference on World wide web},
  pages={661--670},
  year={2010}
}

@article{hirano2009asymptotics,
  title={Asymptotics for statistical treatment rules},
  author={Hirano, Keisuke and Porter, Jack R},
  journal={Econometrica},
  volume={77},
  number={5},
  pages={1683--1701},
  year={2009},
  publisher={Wiley Online Library}
}

@article{kitagawa2018should,
  title={Who should be treated? empirical welfare maximization methods for treatment choice},
  author={Kitagawa, Toru and Tetenov, Aleksey},
  journal={Econometrica},
  volume={86},
  number={2},
  pages={591--616},
  year={2018},
  publisher={Wiley Online Library}
}

@article{manski2004statistical,
  title={Statistical treatment rules for heterogeneous populations},
  author={Manski, Charles F},
  journal={Econometrica},
  volume={72},
  number={4},
  pages={1221--1246},
  year={2004},
  publisher={Wiley Online Library}
}

@article{bottou2013counterfactual,
  title={Counterfactual reasoning and learning systems: The example of computational advertising.},
  author={Bottou, L{\'e}on and Peters, Jonas and Qui{\~n}onero-Candela, Joaquin and Charles, Denis X and Chickering, D Max and Portugaly, Elon and Ray, Dipankar and Simard, Patrice and Snelson, Ed},
  journal={Journal of Machine Learning Research},
  volume={14},
  number={11},
  year={2013}
}

@inproceedings{thomas2015high,
  title={High-confidence off-policy evaluation},
  author={Thomas, Philip and Theocharous, Georgios and Ghavamzadeh, Mohammad},
  booktitle={Proceedings of the AAAI Conference on Artificial Intelligence},
  volume={29},
  number={1},
  year={2015}
}

@article{goldberg2001eigentaste,
  author  = {Goldberg, Ken and Roeder, Theresa and Gupta, Dhruv and Perkins, Chris},
  title   = {Eigentaste: A Constant Time Collaborative Filtering Algorithm},
  journal = {Information Retrieval},
  volume  = {4},
  number  = {2},
  pages   = {133--151},
  year    = {2001},
  doi     = {10.1023/A:1011419012209}
}

@article{sverdrup2020policytree,
  title={policytree: Policy learning via doubly robust empirical welfare maximization over trees},
  author={Sverdrup, Erik and Kanodia, Ayush and Zhou, Zhengyuan and Athey, Susan and Wager, Stefan},
  journal={Journal of Open Source Software},
  volume={5},
  number={50},
  pages={2232},
  year={2020}
}

@inproceedings{oprescu2019econml,
  title={EconML: A machine learning library for estimating heterogeneous treatment effects},
  author={Oprescu, Miruna and Syrgkanis, Vasilis and Battocchi, Keith and Hei, Maggie and Lewis, Greg},
  booktitle={33rd Conference on Neural Information Processing Systems},
  volume={6},
  year={2019},
  organization={Curran Associates, Inc}
}

@article{imai2023experimental,
  title={Experimental evaluation of individualized treatment rules},
  author={Imai, Kosuke and Li, Michael Lingzhi},
  journal={Journal of the American Statistical Association},
  volume={118},
  number={541},
  pages={242--256},
  year={2023},
  publisher={Taylor \& Francis}
}

@misc{sharma2022doordashbandit,
  author       = {Arjun Sharma},
  title        = {Using a Multi-Armed Bandit with Thompson Sampling to Identify Responsive Dashers},
  year         = {2022},
  month        = mar,
  day          = {15},
  howpublished = {\url{https://careersatdoordash.com/blog/using-a-multi-armed-bandit-with-thompson-sampling-to-identify-responsive-dashers/}},
  note         = {DoorDash blog post, accessed June 6, 2026}
}

@article{varatharajah2022contextual,
  title={A contextual-bandit-based approach for informed decision-making in clinical trials},
  author={Varatharajah, Yogatheesan and Berry, Brent},
  journal={Life},
  volume={12},
  number={8},
  pages={1277},
  year={2022},
  publisher={MDPI}
}

@article{villar2015multi,
  title={Multi-armed bandit models for the optimal design of clinical trials: benefits and challenges},
  author={Villar, Sof{\'\i}a S and Bowden, Jack and Wason, James},
  journal={Statistical science: a review journal of the Institute of Mathematical Statistics},
  volume={30},
  number={2},
  pages={199},
  year={2015}
}

@book{lattimore2020bandit,
  title={Bandit algorithms},
  author={Lattimore, Tor and Szepesv{\'a}ri, Csaba},
  year={2020},
  publisher={Cambridge University Press}
}

@article{slivkins2019introduction,
  title={Introduction to multi-armed bandits},
  author={Slivkins, Aleksandrs},
  journal={Foundations and Trends{\textregistered} in Machine Learning},
  volume={12},
  number={1-2},
  pages={1--286},
  year={2019},
  publisher={Emerald Publishing Limited}
}

@article{dudik2011efficient,
  title={Doubly robust policy evaluation and learning},
  author={Dud{\'\i}k, Miroslav and Langford, John and Li, Lihong},
  journal={arXiv preprint arXiv:1103.4601},
  year={2011}
}

@article{laber2014dynamic,
  title={Dynamic treatment regimes: Technical challenges and applications},
  author={Laber, Eric B and Lizotte, Daniel J and Qian, Min and Pelham, William E and Murphy, Susan A},
  journal={Electronic journal of statistics},
  volume={8},
  number={1},
  pages={1225},
  year={2014}
}

@article{goldberg2014comment,
  title={Comment on “Dynamic treatment regimes: Technical challenges and applications”},
  author={Goldberg, Yair and Song, Rui and Zeng, Donglin and Kosorok, Michael R},
  journal={Electronic journal of statistics},
  volume={8},
  number={1},
  pages={1290},
  year={2014}
}

@article{chakraborty2010inference,
  title={Inference for non-regular parameters in optimal dynamic treatment regimes},
  author={Chakraborty, Bibhas and Murphy, Susan and Strecher, Victor},
  journal={Statistical methods in medical research},
  volume={19},
  number={3},
  pages={317--343},
  year={2010},
  publisher={SAGE Publications Sage UK: London, England}
}

@article{le1986asymptotic,
  title={Asymptotic Methods in Statistical Decision Theory},
  author={Le Cam, Lucien},
  journal={Springer Series in Statistics},
  year={1986},
  publisher={Springer New York}
}

@article{dehejia2005program,
  title={Program evaluation as a decision problem},
  author={Dehejia, Rajeev H},
  journal={Journal of Econometrics},
  volume={125},
  number={1-2},
  pages={141--173},
  year={2005},
  publisher={Elsevier}
}

@article{stoye2009minimax,
  title={Minimax regret treatment choice with finite samples},
  author={Stoye, J{\"o}rg},
  journal={Journal of Econometrics},
  volume={151},
  number={1},
  pages={70--81},
  year={2009},
  publisher={Elsevier}
}

@article{girard2025online,
  title={Online Policy Learning via a Self-Normalized Maximal Inequality},
  author={Girard, Samuel and Bibaut, Aur{\'e}lien and Zenati, Houssam},
  journal={arXiv preprint arXiv:2510.15483},
  year={2025}
}

@article{horvitz1952generalization,
  title={A generalization of sampling without replacement from a finite universe},
  author={Horvitz, Daniel G and Thompson, Donovan J},
  journal={Journal of the American statistical Association},
  volume={47},
  number={260},
  pages={663--685},
  year={1952},
  publisher={Taylor \& Francis}
}

@article{rosenbaum1983central,
  title={The central role of the propensity score in observational studies for causal effects},
  author={Rosenbaum, Paul R and Rubin, Donald B},
  journal={Biometrika},
  volume={70},
  number={1},
  pages={41--55},
  year={1983},
  publisher={Oxford University Press}
}

@book{van2011targeted,
  title={Targeted learning: causal inference for observational and experimental data},
  author={Van der Laan, Mark J and Rose, Sherri and others},
  volume={4},
  year={2011},
  publisher={Springer}
}

@article{bang2005doubly,
  title={Doubly robust estimation in missing data and causal inference models},
  author={Bang, Heejung and Robins, James M},
  journal={Biometrics},
  volume={61},
  number={4},
  pages={962--973},
  year={2005},
  publisher={Oxford University Press}
}

@article{robins1995analysis,
  title={Analysis of semiparametric regression models for repeated outcomes in the presence of missing data},
  author={Robins, James M and Rotnitzky, Andrea and Zhao, Lue Ping},
  journal={Journal of the american statistical association},
  volume={90},
  number={429},
  pages={106--121},
  year={1995},
  publisher={Taylor \& Francis}
}

@article{robins1994estimation,
  title={Estimation of regression coefficients when some regressors are not always observed},
  author={Robins, James M and Rotnitzky, Andrea and Zhao, Lue Ping},
  journal={Journal of the American statistical Association},
  volume={89},
  number={427},
  pages={846--866},
  year={1994},
  publisher={Taylor \& Francis}
}

@book{tsiatis2006semiparametric,
  title={Semiparametric theory and missing data},
  author={Tsiatis, Anastasios A},
  volume={4},
  year={2006},
  publisher={Springer}
}

@article{li2026statistical,
  title={A statistical test for the benefits of personalizing interventions},
  author={Li, Zhaoqi and Brunskill, Emma},
  journal={Science},
  volume={393},
  number={6807},
  pages={eaeb9506},
  year={2026},
  publisher={American Association for the Advancement of Science}
}

@inproceedings{li2011unbiased,
  title={Unbiased offline evaluation of contextual-bandit-based news article recommendation algorithms},
  author={Li, Lihong and Chu, Wei and Langford, John and Wang, Xuanhui},
  booktitle={Proceedings of the fourth ACM international conference on Web search and data mining},
  pages={297--306},
  year={2011}
}

@inproceedings{thomas2016data,
  title={Data-efficient off-policy policy evaluation for reinforcement learning},
  author={Thomas, Philip and Brunskill, Emma},
  booktitle={International conference on machine learning},
  pages={2139--2148},
  year={2016},
  organization={PMLR}
}

@inproceedings{jiang2016doubly,
  title={Doubly robust off-policy value evaluation for reinforcement learning},
  author={Jiang, Nan and Li, Lihong},
  booktitle={International conference on machine learning},
  pages={652--661},
  year={2016},
  organization={PMLR}
}

@article{dudik2012sample,
  title={Sample-efficient nonstationary policy evaluation for contextual bandits},
  author={Dud{\'\i}k, Miroslav and Erhan, Dumitru and Langford, John and Li, Lihong},
  journal={arXiv preprint arXiv:1210.4862},
  year={2012}
}

@article{kallus2020double,
  title={Double reinforcement learning for efficient off-policy evaluation in markov decision processes},
  author={Kallus, Nathan and Uehara, Masatoshi},
  journal={Journal of Machine Learning Research},
  volume={21},
  number={167},
  pages={1--63},
  year={2020}
}

@article{shi2022statistical,
  title={Statistical inference of the value function for reinforcement learning in infinite-horizon settings},
  author={Shi, Chengchun and Zhang, Sheng and Lu, Wenbin and Song, Rui},
  journal={Journal of the Royal Statistical Society Series B: Statistical Methodology},
  volume={84},
  number={3},
  pages={765--793},
  year={2022},
  publisher={Oxford University Press}
}

@inproceedings{karampatziakis2021off,
  title={Off-policy confidence sequences},
  author={Karampatziakis, Nikos and Mineiro, Paul and Ramdas, Aaditya},
  booktitle={International Conference on Machine Learning},
  pages={5301--5310},
  year={2021},
  organization={PMLR}
}

@article{waudby2024anytime,
  title={Anytime-valid off-policy inference for contextual bandits},
  author={Waudby-Smith, Ian and Wu, Lili and Ramdas, Aaditya and Karampatziakis, Nikos and Mineiro, Paul},
  journal={ACM/IMS Journal of Data Science},
  volume={1},
  number={3},
  pages={1--42},
  year={2024},
  publisher={ACM New York, NY}
}

@article{howard2021time,
  title={Time-uniform, nonparametric, nonasymptotic confidence sequences},
  author={Howard, Steven R and Ramdas, Aaditya and McAuliffe, Jon and Sekhon, Jasjeet},
  journal={The Annals of Statistics},
  volume={49},
  number={2},
  pages={1055--1080},
  year={2021},
  publisher={JSTOR}
}

@article{athey2021policy,
  title={Policy learning with observational data},
  author={Athey, Susan and Wager, Stefan},
  journal={Econometrica},
  volume={89},
  number={1},
  pages={133--161},
  year={2021},
  publisher={Wiley Online Library}
}

@article{luedtke2021correction,
  title   = {Correction Note: ``Statistical Inference for the Mean Outcome under a Possibly Nonunique Optimal Treatment Rule''},
  author  = {Luedtke, Alex and Bibaut, Aur{\'e}lien and van der Laan, Mark J.},
  journal = {The Annals of Statistics},
  year    = {2021},
  volume  = {49},
  number  = {4},
  pages   = {2429--2430},
  doi     = {10.1214/20-AOS2031},
  url     = {https://doi.org/10.1214/20-AOS2031}
}

@article{audibert2007fast,
  title={Fast learning rates for plug-in classifiers},
  author={Audibert, Jean-Yves and Tsybakov, Alexandre B},
  journal={The Annals of Statistics},
  volume={35},
  number={2},
  pages={608--633},
  year={2007}
}

@article{van2015targeted,
  title={Targeted learning of the mean outcome under an optimal dynamic treatment rule},
  author={van der Laan, Mark J and Luedtke, Alexander R},
  journal={Journal of causal inference},
  volume={3},
  number={1},
  pages={61--95},
  year={2015},
  publisher={De Gruyter}
}

@article{jia2024cramming,
  title={Cramming Contextual Bandits for On-policy Statistical Evaluation},
  author={Jia, Zeyang and Imai, Kosuke and Li, Michael Lingzhi},
  journal={arXiv preprint arXiv:2403.07031},
  year={2024}
}

@article{caria2024adaptive,
  title={An adaptive targeted field experiment: Job search assistance for refugees in Jordan},
  author={Caria, A Stefano and Gordon, Grant and Kasy, Maximilian and Quinn, Simon and Shami, Soha Osman and Teytelboym, Alexander},
  journal={Journal of the European Economic Association},
  volume={22},
  number={2},
  pages={781--836},
  year={2024},
  publisher={Oxford University Press}
}

@article{shen2024doubly,
  title={Doubly robust interval estimation for optimal policy evaluation in online learning},
  author={Shen, Ye and Cai, Hengrui and Song, Rui},
  journal={Journal of the American Statistical Association},
  volume={119},
  number={548},
  pages={2811--2821},
  year={2024},
  publisher={Taylor \& Francis}
}
\bibliographystyle{plainnat}

\appendix
\section{Extended Related Work}
\label{app:related}

In this appendix, we provide an additional related work section that addresses works that are more tangentially related to the primary results of the paper.

\paragraph{Additional Work on Off-Policy Evaluation}

The classical IPW estimates for the value of a fixed treatment policy can be derived from seminal results on propensity weighting by \citet{horvitz1952generalization} and \citet{rosenbaum1983central}. Likewise, the foundations for augmented IPW (or AIPW) style estimates for the value of a fixed policy date back to work on missing data by \citet{robins1994estimation, robins1995analysis, tsiatis2006semiparametric}. Off-policy evaluation is also broadly related to the literature on doubly-robust estimation~\citep{bang2005doubly}, double machine learning/cross-fitting~\citep{chernozhukov2018double}, and targeted maximum likelihood estimation~\citep{van2011targeted}.

Modern approaches for off policy evaluation on bandit data date to works by \citet{li2010contextual, li2011unbiased}, who describe an (inefficient) approach for OPE based on rejection sampling.
The doubly-robust AIPW score for OPE was presented in \citet{dudik2011efficient, dudik2012sample}, who show how the use of a reward model can systematically reduce variance in estimates of a policy's value.
Likewise, OPE has been extensively studied in the reinforcement learning literature, with \citet{thomas2015high}, \citet{jiang2016doubly}, and \citet{thomas2016data} being notable examples. These works largely focus on analyzing variance, bounding mean squared error, and providing high-probability, finite sample confidence intervals. In particular, they do not provide a means for doing statistical inference on the value of a policy via asymptotic normality. Further, they assume bounded importance weights, and thus fail to be applicable in settings where suboptimal actions are played with vanishing frequency.
More recent works due to \citet{kallus2020double} and \citet{shi2022statistical} actually provide analyses that yield asymptotically normal estimates of the value of a policy in RL settings.
While there are many more works on OPE in RL settings, they are beyond the scope of this paper.

\paragraph{Additional Work on Policy Learning}

There is a rich literature on policy learning, with contributions originating from the statistics, econometrics, and machine learning literatures.
An early example from the econometrics literature is due to \citet{manski2004statistical} and studies minimax regret rates for learning a globally-optimal treatment policy given data from an experiment. \citet{dehejia2005program} considers a complementary, Bayesian decision theoretic perspective to policy learning. Subsequent work by \citet{hirano2009asymptotics} studies the policy learning problem proposed by Manski asymptotically through the methodology of Le Cam~\citep{le1986asymptotic} in both parametric and semiparametric settings. \citet{stoye2009minimax} shows that, if covariates are continuously distributed, minimax regret cannot vanish to zero if the policy class considered is left unrestricted. \citet{kitagawa2018should} study policy learning with restricted policy classes through the framework of \textit{empirical welfare maximization}, and show regret rates that depend on the VC-dimension of the underlying policy class. \citet{athey2021policy} then extend these results to the setting of observational data, discussing the requisite nuisance estimation rates needed to guarantee low regret. We note that, while the above works are foundational to the literature on policy learning, they are all focused on a setting where data is i.i.d., and thus the regret bounds contained within are not applicable to the present setting.

\begin{table}[H]
\centering
\caption{An extended comparison of our results with the most closely related papers.}
\label{tab:related-work}
\small
\setlength{\tabcolsep}{3pt}
\renewcommand{\arraystretch}{1.15}

\begin{tabularx}{\textwidth}{
    @{}
    >{\raggedright\arraybackslash}p{0.155\textwidth}
    >{\raggedright\arraybackslash}p{0.145\textwidth}
    >{\raggedright\arraybackslash}p{0.17\textwidth}
    >{\raggedright\arraybackslash}p{0.205\textwidth}
    X
    @{}
}
\toprule
Work
& Setting
& Target
& Result
& Key assumptions
\\
\midrule

\citet{luedtke2016statistical}
& i.i.d.\ data
& Unrestricted optimal value $V^\circ$
& Asymptotically normal estimator
&  Propensity and regression nuisance convergence rates.
\\
\midrule

\citet{li2026statistical}
& i.i.d.\ data
& Best in-class improvement $V^*- \max_a V(a)$
& Asymptotically normal estimator
&  Uniqueness, Nuisance convergence rates.
\\
\midrule
\citet{zhan2021off}
& Contextual bandit data
& Fixed-policy value $V(\pi)$
& Asymptotically normal estimator
&  $\alpha<1/2$
\\
\midrule

\citet{bibaut2021post}
& Contextual bandit data
& Fixed-policy value $V(\pi)$
& Asymptotically normal estimator
& $\alpha\leq1/2$ ($\alpha \leq 1/3$ for variance convergence)
\\
\midrule

\citet{imai2023experimental}
& i.i.d.\ randomized experiment
& Fixed-policy contrast $V(\pi)-V(q)$
& Finite-sample variance and bias bounds for policies
& Binary treatment; complete randomization
\\
\midrule

\citet{bibaut2021risk}
& Contextual bandit data
& Best-in-class policy learning
& High probability regret
$\Theta\left(n^{(\alpha - 1)/2}\right)$
& Sequential bracketing entropy
\\
\midrule

\citet{zhan2024policy}
& Contextual bandit data
& Best-in-class policy learning
& Expected regret $\Theta\left(n^{(\alpha - 1)/2}\right)$
& Finite Natarajan dimension
\\
\midrule

\citet{girard2025online}
& Contextual bandit data
& Best-in-class policy learning
& High probability regret faster than $n^{(\alpha - 1)/2}$
& Sup-norm entropy; stochastic policies;
H\"olderian variance condition implying uniqueness.
\\
\midrule

\citet{jia2024cramming}
& Contextual bandit data
& Value of final logging policy
& Asymptotically normal estimator
& $\delta$-stability of logging policy; $\alpha \leq \min\left\{1/2, \tfrac{\delta}{1 + \delta}\right\}$
\\
\midrule

\textbf{This paper}
& Contextual bandit data
&  Best in-class contrast $V^\ast -V(q)$
& Asymptotically normal estimator; expected regret
$o\!\left(n^{(\alpha -1)/2}\right)$
& Finite disagreement VC dimension; $\alpha<1$
\\

\bottomrule
\end{tabularx}
\end{table}

\section{Proof of Theorem~\ref{thm:clt}}
\label{app:clt_proof}

In this appendix, we present our proof of  Theorem~\ref{thm:clt}. This result guarantees the convergence in distribution of our estimator for the value of policy improvement over a baseline policy, and can be used to construct asymptotically valid confidence intervals for $\theta^\ast$. .

We now describe at a high level how our proof will proceed. To ease notation, define $\wb{\theta}_n := \frac{\sum_{t = 1}^n\wh{\sigma}_t^{-1}\theta(\wh{\pi}_{t - 1})}{\sum_{t = 1}^n \wh{\sigma}_t^{-1}}$ as the weighted average of improvement gaps associated with the (random) sequence of learned policies and $\wh{\Sigma}_n := \sum_{t = 1}^n\wh{\sigma}_t^{-1}$ as the accumulated sum of the predictable inverse standard deviations. Our proof is based on the following decomposition:
\begin{align*}
\frac{\wh{\Sigma}_n}{\sqrt{n}}\left\{\wh{\theta}_n - \theta^\ast\right\} = \underbrace{\frac{\wh{\Sigma}_n}{\sqrt{n}}\left\{\wh{\theta}_n - \wb{\theta}_n\right\}}_{\text{Martingale CLT Term}} - \underbrace{\frac{\sum_{t = 1}^n \wh{\sigma}_t^{-1} \Reg(\wh{\pi}_{t - 1})}{\sqrt{n}}}_{\text{Cumulative Regret Term}},
\end{align*}
which was outlined slightly more verbosely in Equation~\eqref{eq:clt_decomp}.
As the labeling suggests, our goal is to use a martingale central limit theorem for triangular arrays to show that the first, studentized term is asymptotically standard normal. Then, we aim to use our fast regret rates (outlined in Section~\ref{sec:regret} and proved in Appendix~\ref{app:fast-regret} below) to show that the second term is asymptotically negligible, in the sense that it is $o_\P(1)$. We prove each of these two results as propositions in this appendix. We defer the task of showing almost sure convergence of the predictable variance estimates for the next appendix (Appendix~\ref{app:variance}).

To show the normality of the first term, we leverage the following central limit theorem, which appears in \citet{ganssler1978central} and \citet{brown1971martingale}. Other martingale central limit theorems, such as those discussed in \citet{hall2014martingale}, may also be applicable.

\begin{theorem}[Theorem 2 of \citet{brown1971martingale}, Theorem 2 of \citet{ganssler1978central}]
\label{thm:martingale-clt}
For each $n \geq 1$, let $(\zeta_{n, j} : 1 \leq j \leq m_n)$ denote a martingale difference sequence with respect to a filtration $(\calG_{n, j} : 0 \leq j \leq m_n)$. Define the predictable quadratic variation by
\[
\calV_n := \sum_{j = 1}^{m_n}\E\left(\zeta_{n, j}^2 \mid \calG_{n, j -1}\right) = \sum_{j = 1}^{m_n}\Var\left(\zeta_{n, j} \mid \calG_{n, j- 1}\right).
\]
Suppose that the following hold:
\begin{enumerate}
\item The predictable variance converges in probability to one, i.e.\ $\calV_n  - 1 = o_\P(1)$.
\item For any $\epsilon > 0$, the process $(L_n(\epsilon))$ defined by
\[
L_n(\epsilon) := \sum_{j = 1}^{m_n}\E\left(\zeta_{n, j}^2\1\left\{|\zeta_{n, j}| > \epsilon\right\} \mid \calG_{n, j - 1}\right)
\]
converges in probability to zero, i.e.\ $L_n(\epsilon) = o_\P(1)$.
\end{enumerate}
Then, we have $\sum_{j = 1}^{m_n} \zeta_{n, j} \Rightarrow \calN(0, 1)$.
\end{theorem}

\begin{prop}
Under the assumptions of Theorem~\ref{thm:clt}, the randomly centered studentized statistic is asymptotically standard normal, i.e.\
\[
\frac{\wh{\Sigma}_n}{\sqrt{n}}\left\{\wh{\theta}_n - \wb{\theta}_n\right\} \Rightarrow \calN(0, 1).
\]
\end{prop}

\begin{proof}[Proof of Asymptotic Normality]
First, for $n \geq 1$ and $1 \leq t \leq n$, define the difference array
\[
\zeta_{n, t} := \frac{\wt{Y}_t(\wh{\pi}_{t - 1}, q) - \theta(\wh{\pi}_{t - 1})}{\wh{\sigma}_t \sqrt{n}},
\]
where we let $\theta(\pi) := V(\pi) - V(q)$ for any policy $\pi \in \Pi$.

Throughout, we work with respect to the natural filtration $(\calF_t)_{t \geq 0}$ generated by the observations $(Z_n)_{n \geq 0}$. To complete the proof, we need to check three things. First, we need to show that $(\zeta_{n, t})_{1 \leq t \leq n}$ is a martingale difference sequence with respect to $(\calF_t)_{0 \leq t \leq n}$. Then, we need to check the two other conditions of theorem.

Thus, we start by checking that, for each $n \geq 1$, the process $(\zeta_{n,t})_{1 \leq t \leq n}$ is a martingale difference sequence with respect to the natural filtration $(\calF_t)_{1 \leq t \leq n}$. First, observe that integrability is trivial, as the assumption that $\wh{\sigma}_t \geq \kappa_t > 0$ for some $\kappa_t > 0$ implies
\[
\E[|\zeta_{n, t}|] \leq \frac{1}{\kappa_t\sqrt{n}}\E\left|\wt{Y}_t(\wh{\pi}_{t - 1}, q) - \theta(\wh{\pi}_{t - 1})\right| \leq 2\frac{t^{\alpha}}{c\kappa_t \sqrt{n}}.
\]
Next, the martingale property is satisfied, as we have
\begin{align*}
\E\left(\zeta_{n, t} \mid \calF_{t - 1}\right) &= \E\left(\frac{\wt{Y}_t(\wh{\pi}_{t - 1}, q) - \theta(\wh{\pi}_{t - 1})}{\wh{\sigma}_t \sqrt{n}} \mid \calF_{t - 1}\right) \\
&= \frac{\E\left(\wt{Y}_t(\wh{\pi}_{t - 1}, q) \mid \calF_{t - 1}\right) - \theta(\wh{\pi}_{t - 1})}{\wh{\sigma}_t \sqrt{n}} \\
&= 0,
\end{align*}
where  the second to last line follows because both $\wh{\sigma}_t$ and $\wh{\pi}_{t - 1}$ are $\calF_{t - 1}$-measurable and the final line follows as we have
\begin{align*}
\E\left(\wt{Y}_t(\wh{\pi}_{t - 1}, q) \mid \calF_{t - 1}\right) &= \E\left(\wt{Y}_t(\wh{\pi}_{t - 1}) \mid \calF_{t - 1}\right) - \E\left(\wt{Y}_t(q) \mid \calF_{t - 1}\right) \\
&= V(\wh{\pi}_{t - 1}) - V(q) \\
&= \theta(\wh{\pi}_{t - 1}),
\end{align*}
where the second to last line follows because, conditional on $\calF_{t - 1}$, the policy $\wh{\pi}_{t - 1}$ is fixed.
Now, we check that the conditions of the central limit theorem are satisfied. To see that the first condition holds, observe that we have
\[
\calV_n = \sum_{t = 1}^n \Var(\zeta_{n, t} \mid \calF_{t - 1}) = \frac{1}{n}\sum_{t = 1}^n\wh{\sigma}_t^{-2}\Var(\wt{Y}_t(\wh{\pi}_{t - 1}, q) \mid \calF_{t - 1}) = \frac{1}{n}\sum_{t = 1}^n \frac{\sigma_t^2}{\wh{\sigma}_t^2}.
\]
Thus, it suffices to demonstrate that the assumptions of Theorem~\ref{thm:clt} imply that
\[
\frac{1}{n}\sum_{t = 1}^n \frac{\sigma_t^2}{\wh{\sigma}_t^2} - 1 = o_\P(1).
\]

\noindent Let $\Omega$ denote the underlying sample space. By the assumption that
\[
\lim_{t \rightarrow \infty}\frac{\wh{\sigma}^2_t - \sigma_t^2}{\sigma_t^2} = 0 \quad \text{almost surely,}
\]
there exists $\Omega_0 \subset \Omega$ with $\P(\Omega_0) = 1$ such that the following holds: for any $\omega \in \Omega_0$, there is a sample size $N(\omega) \in \N$ such that, for all $t \geq N(\omega)$, we have
\begin{equation}
\label{eq:bdd-ratio-variance}
\frac{1}{2} \leq \frac{\wh{\sigma}_t^2(\omega)}{\sigma_t^2(\omega)} \leq \frac{3}{2},
\end{equation}
where we have indexed the predictable processes $(\wh{\sigma}_t^2)_{t \geq 1}$ and $(\sigma_t^2)_{t \geq 1}$ by the underlying element of the probability space to be explicit. Thus, for any such $\omega \in \Omega_0$ and sufficiently large $n$, we have
\begin{align*}
\frac{1}{n}\sum_{t = 1}^n \frac{\sigma_t^2(\omega)}{\wh{\sigma}_t^2(\omega)} - 1 = \frac{1}{n}\sum_{t = 1}^{N(\omega)}\left\{\frac{\sigma_t^2(\omega)}{\wh{\sigma}_t^2(\omega)} - 1\right\} + \frac{1}{n}\sum_{t = N(\omega) + 1}^n \left\{\frac{\sigma_t^2(\omega)}{\wh{\sigma}_t^2(\omega)}  - 1\right\}.
\end{align*}
The first term on the right hand side trivially vanishes as $n$ grows large, i.e.\ we have
\begin{align*}
\lim_{n \rightarrow \infty}\left|\frac{1}{n}\sum_{t = 1}^{N(\omega)}\left\{\frac{\sigma_t^2(\omega)}{\wh{\sigma}_t^2(\omega)} - 1\right\}\right| & \lesssim  \lim_{n \rightarrow \infty} \frac{1}{n}\sum_{t = 1}^{N(\omega)} \frac{t^\alpha}{\kappa_t^2} + \frac{N(\omega)}{n} \\
&\lesssim \max_{1 \leq t \leq N(\omega)} \kappa_t^{-2} \cdot \lim_{n \rightarrow \infty} \frac{N(\omega)^{\alpha  + 1}}{n} + \lim_{n \rightarrow \infty}\frac{N(\omega)}{n} = 0.
\end{align*}
The second term can be shown to converge almost surely to zero as follows. First, note that
\begin{align*}
\frac{1}{n}\left|\sum_{t = N(\omega) + 1}^n \left\{\frac{\sigma_t^2(\omega)}{\wh{\sigma}_t^2(\omega)} - 1\right\} \right| &\leq \frac{1}{n}\sum_{t = N(\omega) + 1}^n \left|\frac{\sigma_t^2(\omega)}{\wh{\sigma}_t^2(\omega)} - 1 \right| \\
&= \frac{1}{n}\sum_{t = N(\omega) + 1}^n \frac{\left|\sigma_t^2(\omega) - \wh{\sigma}_t^2(\omega)\right|}{\wh{\sigma}_t^2(\omega)} \\
&= \frac{1}{n}\sum_{t = N(\omega) + 1}^n \frac{\left|\wh{\sigma}_t^2(\omega)/\sigma_t^2(\omega) - 1\right|}{\wh{\sigma}_t^2(\omega)/\sigma_t^2(\omega)} \\
&\lesssim \frac{1}{n}\sum_{t = N(\omega) + 1}^n \left|\frac{\wh{\sigma}_t^2(\omega) - \sigma_t^2(\omega)}{\sigma_t^2(\omega)}\right|.
\end{align*}
Since we know $a_t(\omega) :=\frac{\wh{\sigma}_t^2(\omega) - \sigma_t^2(\omega)}{\sigma_t^2(\omega)} \rightarrow 0$ as $t \rightarrow \infty$ for all $\omega \in \Omega_0$, it follows that
\[
\lim_{n \rightarrow \infty}\left|\frac{1}{n}\sum_{t = N(\omega) + 1}^n \left\{\frac{\sigma_t^2(\omega)}{\wh{\sigma}_t^2(\omega)} - 1\right\}\right|  \leq \lim_{n \rightarrow \infty}\frac{1}{n}\sum_{t = N(\omega) + 1}^n \left|a_t(\omega)\right| = 0,
\]
Since if $|a_n| \rightarrow 0$ then $\frac{1}{n}\sum_{t = 1}^n |a_t| \rightarrow 0$ as well. Thus, we see that
\[
\lim_{n \rightarrow \infty}\left\{\frac{1}{n}\sum_{t = 1}^n \frac{\sigma_t^2}{\wh{\sigma}_t^2} - 1\right\} = 0 \quad \text{almost surely,}
\]
and hence the desired convergence in probability holds as well.

Now, we check that the Lindeberg condition holds. Our goal is to show that, for any $\epsilon > 0$, we have
\[
L_n(\epsilon) := \sum_{t = 1}^n \E\left(\zeta_{n, t}^2\1\{|\zeta_{n, t}| > \epsilon\} \mid \calF_{t - 1}\right) = o_\P(1).
\]

First, note that there is some deterministic constant $C > 0$ independent of $n$ or $t$ such that
\[
|\zeta_{n, t}| = \frac{1}{\sqrt{n}\wh{\sigma}_t}\left|\wt{Y}_t(\wh{\pi}_{t - 1}, q) - \theta(\wh{\pi}_{t - 1})\right| \leq \frac{Ct^{\alpha}}{\sqrt{n}\wh{\sigma}_t} =: B_{n, t},
\]
which follows from the first part of Assumption~\ref{ass:estimation} (bounded outcomes) and Assumption~\ref{ass:propensity}. In particular, note that because $\wh{\sigma}_t$ is assumed to be $\calF_{t - 1}$-measurable, $B_{n, t}$ is $\calF_{t - 1}$-measurable as well.

Next, on the event outlined in Equation~\eqref{eq:bdd-ratio-variance}, we note that $\wh{\sigma}_t \geq \sigma_t/\sqrt{2}$. Hence, since Assumption~\ref{ass:ipw_variance_scale} implies that $\sigma_t \geq \sqrt{c_\sigma} t^{\alpha/2}$, on the noted event, we have
\begin{equation}
\label{eq:deterministic_bd_set}
B_{n,t} \leq C'\frac{t^{\alpha/2}}{\sqrt{n}} \leq C'n^{\frac{\alpha - 1}{2}}.
\end{equation}

Now, from our earlier discussion, we know there is a random time $N : \Omega_0 \rightarrow \N$ such that Equation~\eqref{eq:deterministic_bd_set} holds for all $N(\omega) \leq t \leq n$. Hence, explicitly noting dependence on the underlying element $\omega$ of the sample space, on the event $\Omega_0$ (which, again, satisfies $\P(\Omega_0) = 1$), we have for any $n \geq N(\omega)$
\begin{align*}
L_n(\epsilon)(\omega) = \underbrace{\sum_{t = 1}^{N(\omega) - 1}\E\left(\zeta_{n, t}^2\1\left\{|\zeta_{n, t} > \epsilon\right\} \mid \calF_{t - 1}\right)(\omega)}_{=: T_1}  + \underbrace{\sum_{t = N(\omega)}^{n}\E\left(\zeta_{n, t}^2\1\left\{|\zeta_{n, t}> \epsilon\right\} \mid \calF_{t - 1}\right)(\omega)}_{=: T_2}.
\end{align*}
We argue that both terms tend to zero on $\Omega_0$. Analyzing the first term, note that our assumption that $\wh{\sigma}_t \geq \kappa_t > 0$ for a deterministic $\kappa_t$ implies that
\begin{align*}
T_1 & \leq \sum_{t = 1}^{N(\omega)- 1}\E\left(\zeta_{n, t}^2\mid \calF_{t - 1}\right)(\omega) \\
&= \frac{1}{n}\sum_{t = 1}^{N(\omega) - 1}\frac{\Var\left(\wt{Y}(\wh{\pi}_{t - 1}, q) \mid \calF_{t - 1}\right)(\omega)}{\wh{\sigma}_t^2(\omega)} \\
&= \frac{1}{n}\sum_{t = 1}^{N(\omega) - 1}\frac{\sigma_t^2(\omega)}{\wh{\sigma}_t^2(\omega)} \xrightarrow[n \rightarrow \infty]{} 0.
\end{align*}

The second term can similarly be bounded. Observe that, for any $\epsilon > 0$, there exists some deterministic $m > 0$ such that $B_{n, t} \leq C'n^{\frac{\alpha - 1}{2}} \leq \epsilon$. Hence, for any $\omega \in \Omega_0$ and any such $n \geq \max\{m, N(\omega)\}$, we have
\begin{align*}
T_2 &= \sum_{t = N(\omega)}^{n}\E\left(\zeta_{n, t}^2 \1\left\{|\zeta_{n, t}| > \epsilon\right\} \mid \calF_{t - 1}\right)(\omega) \\
&\leq \sum_{t  = N(\omega)}^n \1\left\{B_{n, t}(\omega) > \epsilon\right\}\E\left(\zeta_{n, t}^2 \mid \calF_{t - 1}\right)(\omega) \\
&\leq \sum_{t = N(\omega)}^{n}\1\left\{C'n^{\frac{\alpha - 1}{2}} > \epsilon\right\}\E\left(\zeta_{n, t}^2 \mid \calF_{t - 1}\right)(\omega) \\
&= 0,
\end{align*}

Thus, since we have checked all three conditions, the desired central limit theorem holds.
\end{proof}

We now turn to the more difficult argument, which involves showing that the remainder term consisting of a weighted sum of regrets vanishes. We already know that, for each $t \geq 1$, $\E[\Reg(\wh{\pi}_t)] = o\left(t^{\frac{\alpha - 1}{2}}\right)$. To show the entire sum vanishes at $o_\P(1)$ rates, we need to apply the following result on the convergence of Ces\`aro sums due to \citet{bibaut2020sufficient}.

\begin{lemma}[Proposition 2 of \citet{bibaut2020sufficient}]
\label{lem:expectation-cesaro}
Let $(X_n)$ be a sequence of non-negative random variables such that $\E[X_n]= o\left(n^{-\beta}\right)$ for some $\beta \in (0, 1)$. Then, the sample mean $\wb{X}_n := \frac{1}{n}\sum_{t = 1}^n X_t$ also satisfies $\E[\wb{X}_n] = o\left(n^{-\beta}\right)$ as well.
\end{lemma}

\begin{prop}
Under the assumptions of Theorem~\ref{thm:clt}, the Ces\`aro regret term vanishes in probability, i.e.\ we have
\[
\wh{R}_n := \frac{\sum_{t = 1}^n \wh{\sigma}_t^{-1}\Reg(\wh{\pi}_{t - 1})}{\sqrt{n}} = o_\P(1).
\]
\end{prop}

\begin{proof}
First, we show that the random, cumulative regret vanishes in probability when weighted by the inverse of the true, predictable standard deviation. In more detail, define the random variable
\[
R_n := \frac{\sum_{t = 1}^n \sigma_t^{-1}\Reg(\wh{\pi}_{t - 1})}{\sqrt{n}}.
\]
Per the results of Theorem~\ref{thm:fast-regret}, we know that $\E\left[\Reg(\wh{\pi}_{t - 1})\right] = o\left(t^{\frac{\alpha - 1}{2}}\right)$.  Further, by Assumption~\ref{ass:ipw_variance_scale}, we have $\sigma_t^2 \geq c_\sigma t^{\alpha}$ for some constant $c_\sigma > 0$. Consequently, we have
\begin{align*}
\E[R_n] &= \frac{1}{\sqrt{n}}\sum_{t = 1}^n \E\left[\sigma_t^{-1}\Reg(\wh{\pi}_{t - 1})\right] \\
&\lesssim \frac{1}{\sqrt{n}}\sum_{t = 1}^n t^{-\alpha/2}\E\left[\Reg(\wh{\pi}_{t - 1})\right] \\
&= \frac{1}{\sqrt{n}}\sum_{t = 1}^n t^{-\alpha/2}o\left(t^{\frac{\alpha - 1}{2}}\right) \\
&= \frac{1}{\sqrt{n}}\sum_{t = 1}^n o(t^{-1/2}) \\
&= o(1),
\end{align*}
where the final line follows from Lemma~\ref{lem:expectation-cesaro}. Thus, for any $\epsilon > 0$, we have
\[
\lim_{n \rightarrow \infty}\P(R_n \geq \epsilon) \leq \lim_{n \rightarrow \infty}\frac{\E[R_n]}{\epsilon} = 0,
\]
and so $R_n = o_\P(1)$. We now show that the same conclusion holds when we weight by estimated inverse standard deviations. That is, we argue that $\wh{R}_n = o_\P(1)$. First, for any $N \geq 1$, define the event
\[
\Epsilon_N := \Bigg\{\frac{\wh{\sigma}_t}{\sigma_t} \geq \frac{1}{2} \quad \text{for all } t \geq N\Bigg\}.
\]
Next, note that $\lim_{N \rightarrow \infty}\P(\Epsilon_N) = 1$. Indeed, defining
\[
\Epsilon_\infty := \bigcup_{N = 1}^\infty \Epsilon_N = \bigcup_{N = 1}^\infty \bigcap_{t \geq N}\left\{\frac{\wh{\sigma}_t}{\sigma_t} \geq \frac{1}{2}\right\} = \Bigg\{\frac{\wh{\sigma}_t}{\sigma_t} \geq \frac{1}{2}\text{ eventually}\Bigg\},
\]
we see that we have the containment
\[
\Epsilon_\infty \supset \Bigg\{\lim_{t \rightarrow \infty}\frac{\wh{\sigma}_t^2}{\sigma_t^2} = 1\Bigg\}.
\]
The latter event has probability one by assumption, and thus since $\Epsilon_N \subset \Epsilon_{N + 1}$ for each $N \geq 1$, we have
\[
\lim_{N \rightarrow \infty}\P\left(\Epsilon_N\right) = \P(\Epsilon_\infty) = 1.
\]
On the event $\Epsilon_N$, for any $n \geq N$, we have the following bound:
\begin{align*}
\wh{R}_n &= \frac{1}{\sqrt{n}}\sum_{t = 1}^n\wh{\sigma}_t^{-1}\Reg(\wh{\pi}_{t - 1}) \\
&= \frac{1}{\sqrt{n}}\sum_{t = 1}^{N - 1}\wh{\sigma}_t^{-1}\Reg(\wh{\pi}_{t - 1}) + \frac{1}{\sqrt{n}}\sum_{t = N}^n\wh{\sigma}_t^{-1}\Reg(\wh{\pi}_{t - 1}) \\
&\leq \frac{1}{\sqrt{n}}\sum_{t = 1}^{N - 1} \wh{\sigma}_t^{-1}\Reg(\wh{\pi}_{t - 1}) + 2 R_n,
\end{align*}
where the final inequality follows because $\wh{\sigma}_t \geq \frac{1}{2}\sigma_t$ on $\Epsilon_N$, for $t \geq N$.
Now, fix $\epsilon > 0$ and $\eta > 0$, and take $N \geq 1$ sufficiently large such that $\P(\Epsilon_N^c) \leq \eta$.  Note that we have
\[
\P\left(\wh{R}_n \geq \epsilon\right)  \leq \P(\Epsilon_N^c) + \P\left(\frac{1}{\sqrt{n}}\sum_{t = 1}^N \wh{\sigma}_t^{-1}\Reg(\wh{\pi}_{t - 1}) \geq \frac{\epsilon}{2}\right) + \P\left(2R_n \geq \frac{\epsilon}{2}\right).
\]
Thus, noting that the latter two probabilities vanish as $n \rightarrow \infty$, we have
\[
\lim\sup_{n \rightarrow \infty}\P\left(\wh{R}_n \geq \epsilon\right) \leq \P(\Epsilon_N^c) \leq \eta.
\]
Since $\eta > 0$ was arbitrary, we must have
\[
\lim\sup_{n \rightarrow \infty}\P\left(\wh{R}_n \geq \epsilon\right) = \lim_{n \rightarrow \infty}\P\left(\wh{R}_n \geq \epsilon\right) = 0,
\]
thus completing the proof.
\end{proof}

\section{Proof of Variance Consistency}
\label{app:variance}
In this appendix, we prove Theorem~\ref{thm:variance-convergence}, which ensures the almost sure convergence of a plug-in estimate of the predictable variance when appropriately scaled by the predictable variance itself. We prove this result by separately arguing that both $\wh{\Phi}_{1,t}$ and $\wh{\Phi}_{2, t}$ converge almost surely when appropriately scaled. Since the former of these estimates depends on an adaptively chosen policy and the latter depends on both the adaptively chosen policy and logging policy, we will need to leverage the complexity bounds on the policy and logging classes in order to show convergence.
In this appendix, we actually prove the result under the more general assumption that the contrasts $g_\pi$ and $t^{-\alpha}g_\pi^2/p_t$ live in function classes that possess polynomial uniform covering numbers.

\begin{ass}
\label{ass:covering}
We assume that there are deterministic function classes $\calG_1$ and $\calG_{2, t}$ and constants $G_1, G_2, \frakv_1, \frakv_2 > 0$ such that
\begin{enumerate}
    \item $\|g_1\|_\infty \leq G_1$ and $\|g_2\|_\infty \leq G_2$ for all $g_1 \in \calG_1$ and $g_2 \in \calG_{2, t}$.
    \item $g_\pi \in \calG_1$ for all $\pi \in \Pi$, and $t^{-\alpha}\frac{g_\pi(a \mid x)^2}{p_t(a \mid x)} \in \calG_{2, t}$ almost surely for all $\pi \in \Pi$.
    \item for $b \in \{1, 2\}$, for all  distributions $Q$ on $\calA \times \calX$, $0 < \epsilon \leq 1$,  and all $t \geq 1$
    \begin{equation}
    \label{eq:uniform-covering}
    \log N\big(G_b \epsilon, \calG_{b, t}, L^2(Q)\big) \lesssim \frakv_b \log\left(\frac{A}{\epsilon}\right),
    \end{equation}
    where $A \geq e$ is a constant and $\calG_{1, t} \equiv \calG_1$.
\end{enumerate}
\end{ass}
When Assumption~\ref{ass:disagreement} holds, Proposition~\ref{prop:fn_class_pseudo} actually implies that the function class $\calG_1$ given by
\[
\calG_1 := \Big\{(a, x) \mapsto g_\pi(a \mid x) : \pi \in \Pi\Big\}
\]
has finite pseudo-dimension, and hence polynomial covering numbers by Lemma~\ref{lem:pseudo-dimension-covering}. Further, the same proposition and lemma combination yields that under Assumption~\ref{ass:logging}, the class $\calG_{2, t}$ given by
\[
\calG_{2, t} := \Big\{(a, x) \mapsto g_\pi(a \mid x)^2 r(a \mid x) : r \in \calP_t, \pi \in \Pi\Big\}
\]
has finite pseudo-dimension and again polynomial uniform covering numbers. We opt to work under Assumption~\ref{ass:covering} as it simplifies the proof.

We also emphasize that, as stated, Assumption~\ref{ass:covering} does not leverage any sort of sequential complexity on the logging policy classes. This is in contrast to the sequential bracketing numbers considered in \citet{bibaut2021post} and the sequential Rademacher complexity based argument leveraged in \citet{zhan2024policy}.

We dedicate the remainder of this section to the proof of Theorem~\ref{thm:variance-convergence}. Our proof proceeds roughly in two stages. First, we leverage high-probability tail bounds for collections of martingales indexed by infinite but well-behaved function classes to show that, for any finite sample size $n$, the worst case deviation of the estimates $\wh{\Phi}_{1, t}$ and $\wh{\Phi}_{2, t}$ from their means is not too large. Then, we use the first Borel-Cantelli lemma to yield almost sure convergence.
We start by stating the aforementioned lemma for completeness.

\begin{lemma}[First Borel-Cantelli Lemma]
\label{lem:borel-cantelli}
Let $(E_n)_{n \geq 1}$ be a sequence of events, and define the event that infinitely many $E_n$ occur as
\[
\Big\{E_n \text{ i.o. }\Big\} = \bigcap_{n \geq 1}\bigcup_{m \geq n}E_m.
\]
Then, if $\sum_{m = 1}^\infty \P(E_m) < \infty$,
\[
\P\Big(E_n \text{ i.o. }\Big) = 0.
\]

\end{lemma}

We now provide our proof of almost sure variance convergence. In what follows, we define $\mu_\ell(a, x) := \E\left(Y^\ell \mid X = x, A = a\right)$.

\begin{proof}[Proof of Theorem~\ref{thm:variance-convergence}]
In the sequel, we consider the processes $\big(G_{\ell, n}(g) : g \in \calG_{\ell, n + 1}\big)$ given by
\[
G_{\ell, n}(g) := \frac{1}{n}\sum_{t = 1}^n \left\{\frac{g(A_t \mid X_t)}{p_t(A_t \mid X_t)}Y_t^\ell - \E_X\left[\sum_{a \in \calA}g(a \mid X)\mu_\ell(a, X)\right] \right\}
\]
where $\calG_{\ell, n + 1} = \calG_1$ when $\ell = 1$ and $\calG_{\ell, n + 1} = \calG_{2, n + 1}$ when $\ell = 2$. This process exactly falls into the category of processes outlined in Equation~\eqref{eq:martingale} with weights $w_t = 1$ (equivalently, $\beta =0$), and hence Theorem~\ref{thm:high-prob-chain} yields the existence of constants $C, c > 0$ such that, for any $\ell \in \{1, 2\}$, $n \geq 2$, and $x > 0$
\[
\P\left(\sup_{g \in \calG_{\ell, n + 1}}\left|G_{\ell, n}(g)\right| \geq r_{\ell, n}(x)\right) \leq Ce^{-x} + C\exp\left\{-cn^{1 - \alpha}\right\}.
\]
where we let
\[
r_{\ell, n}(x) := C'G_\ell n^{\frac{\alpha - 1}{2}}(\sqrt{\frakv_\ell} + \sqrt{x}) + C'G_\ell n^{\alpha - 1}(\frakv_\ell\log(n) + x)\log(n)
\]
for some appropriately large constant $C' > 0$.
For $\ell \in \{1, 2\}$, define the sequence of events $(\Epsilon_{\ell, n})_{n \geq 1}$ as follows:
\[
\Epsilon_{\ell, n} := \left\{\sup_{g\in \calG_{\ell, n + 1}}\left|G_{\ell, n}(g)\right| < r_{\ell, n}(x_n)\right\},
\]
where we have selected $x_n = 4\log(n)$ (which is greater than zero for $n \geq 2$). By this choice of $x_n$, we have
\[
\P(\Epsilon_{\ell, n}^c) \leq Cn^{-4} + C\exp\left\{-cn^{1 - \alpha}\right\}.
\]
Summing over all $n \geq 1$  yields
\[
\sum_{n = 1}^\infty \P(\Epsilon_{\ell, n}^c) \lesssim \P(\Epsilon^c_{\ell, 1}) + C\sum_{n = 2}^\infty n^{-4} + C\sum_{n = 2}^\infty \exp\left\{-cn^{1 - \alpha}\right\} < \infty,
\]
which follows because both sums are finite. Thus, Lemma~\ref{lem:borel-cantelli} yields that
\[
\P\left(\Epsilon_{\ell, n}^c \text{ i.o. } \right) = \P\Bigg(\bigcap_{n \geq 1}\bigcup_{m \geq n} \Epsilon_{\ell, m}^c\Bigg)= 0.
\]
Letting $(\Omega, \calF, \P)$ denote the underlying probability space and letting
\[
\Omega_{\ell, 0} := \bigcup_{n \geq 1}\bigcap_{m \geq n}\Epsilon_{\ell, m} = \left(\bigcap_{n \geq 1}\bigcup_{m \geq n}\Epsilon_{\ell, m}^c\right)^c
\]
denote the probability one event that all $\Epsilon_{\ell, n}$ occur for sufficiently large $n$, and set $\Omega_0 := \Omega_{1, 0} \cap \Omega_{2, 0}$, which still has probability one. For any element  $\omega \in \Omega_0$ and $\ell \in \{1, 2\}$, there exists a sample size $N(\omega) \in \N$ satisfying
\[
\sup_{g \in \calG_{\ell, n + 1}}\left|G_{\ell, n}(g)(\omega)\right| < r_{\ell, n}(4\log n) \quad \text{ for all } n \geq N(\omega).
\]
Noting that
\begin{align*}
\lim_{n \rightarrow \infty} r_{\ell, n}(4\log n) &\lesssim \lim_{n \rightarrow \infty}|\calA|G_\ell n^{\frac{\alpha - 1}{2}}\left(\sqrt{\frakv_\ell} + \sqrt{4\log(n)}\right) \\
&\qquad + \lim_{n \rightarrow \infty}|\calA|G_\ell n^{\alpha - 1}\{\frakv_\ell \log(n) + 4\log(n)\}\log(n) \\
&= 0
\end{align*}
it thus follows that, for all $\omega \in \Omega_{0}$ and $\ell \in \{1, 2\}$, we have
\[
\lim_{n \rightarrow \infty}\sup_{g \in \calG_{\ell, n + 1}}\left|G_{\ell,n}(g)(\omega)\right| = 0,
\]
i.e.\ $\sup_{g \in \calG_{\ell, n + 1}}\left|G_{\ell, n}(g)\right| \xrightarrow[n \rightarrow \infty]{a.s.} 0$.

We now show that this uniform convergence implies the almost sure convergence result claimed in Equation~\eqref{eq:variance-rate}. First, define the random quantities
\[
\Phi_{1, n} :=\E\left(\sum_{a \in \calA}g_n(a \mid X_n) \mu_1(a, X_n) \mid \calF_{n- 1}\right)
\]
and
\[
\Phi_{2, n} := \E\left(\sum_{a \in \calA}\frac{g_n(a \mid X_n)^2}{p_n(a \mid X_n)}\mu_2(a,  X_n) \mid \calF_{n - 1}\right),
\]
and recall that we have the identity $\sigma_n^2 = \Phi_{2, n} - \Phi_{1, n}^2$. Further, since we have assumed that $g_{n}(a \mid x) := \1\left\{\wh{\pi}_{n - 1}(x) = a\right\} - q(a \mid x) \in \calG_1$ for each $n \geq 1$, we have that $\wh{\Phi}_{1, n} - \Phi_{1, n} = G_{1, n - 1}(g_n)$. Thus, the results shown above implies, almost surely,
\[
\lim_{n \rightarrow \infty}\left|\wh{\Phi}_{1, n} - \Phi_{1, n}\right| \leq \lim_{n \rightarrow \infty} \sup_{g \in \calG_1}\left|G_{1, n - 1}(g)\right| = 0.
\]
Since we have assumed that outcomes are bounded, we know that $\sup_{n \geq 1}\left|\Phi_{1, n}\right| \leq C$ for some constant $C > 0$. Defining $D_n := \wh{\Phi}_{1, n} - \Phi_{1, n}$, we thus have
\begin{align*}
\left|\wh{\Phi}_{1, n}^2 - \Phi_{1, n}^2\right| = \left|D_n(2 \Phi_{1, n} + D_n)\right| \leq 2C |D_n| + D_n^2 \xrightarrow[n \rightarrow \infty]{a.s.} 0,
\end{align*}
and hence $\left|\wh{\Phi}_{1, n}^2 - \Phi_{1, n}^2\right| \xrightarrow[n \rightarrow \infty]{a.s.} 0$. Consequently, it also follows that $\tfrac{\left|\wh{\Phi}_{1, n}^2 - \Phi_{1, n}^2\right|}{n^{\alpha}} \xrightarrow[n \rightarrow \infty]{a.s.} 0$, since $0 < n^{-\alpha} \leq 1$.

Next, note that by the assumption we have $h_n(a \mid x) := n^{-\alpha}\frac{g_n(a \mid x)^2}{p_n(a \mid x)} \in \calG_{2, n}$. Since we have $n^{-\alpha}\left(\wh{\Phi}_{2, n} - \Phi_{2, n}\right) = G_{2, n -1}(h_n)$, for each $n \geq 1$, we have almost surely that
\[
\lim_{n \rightarrow \infty} n^{-\alpha}\left|\wh{\Phi}_{2, n} - \Phi_{2, n}\right| \leq \lim_{n \rightarrow \infty}\sup_{g \in \calG_{2, n}}\left|G_{2, n - 1}(g)\right| = 0.
\]
Thus, we have that
\[
n^{-\alpha}\left|\wt{\sigma}_n^2 - \sigma_n^2\right| \leq n^{-\alpha}\left|\wh{\Phi}_{2, n} - \Phi_{2, n}\right| + n^{-\alpha}\left|\wh{\Phi}_{1, n}^2 - \Phi_{1, n}^2\right| \xrightarrow[n \rightarrow \infty]{a.s.} 0.
\]

Recalling that $\sigma_n^2 := \sigma_n^2(\wh{\pi}_{n - 1}) \geq c_\sigma n^\alpha$ by Assumption~\ref{ass:ipw_variance_scale} yields that
\[
\frac{|\wt{\sigma}_n^2 - \sigma_n^2|}{\sigma_n^2} \leq \frac{1}{c_\sigma}n^{-\alpha}\left|\wt{\sigma}_n^2 - \sigma_n^2\right| \xrightarrow[n \rightarrow \infty]{a.s.} 0,
\]
which is the desired result.
This completes the proof.
\end{proof}

\section{Proof of Fast Regret Rates}
\label{app:fast-regret}
The main goal of this appendix is to prove Theorem~\ref{thm:fast-regret}, which shows that the expected regret incurred from performing weighted empirical welfare maximization scales as $o\left(n^{(\alpha - 1)/2}\right)$. Our proof of this result proceeds in two stages. First, we use localized bounds on the expected supremum of a collection of martingales to yield the target regret convergence rate \textit{in probability}. Then, we use unlocalized, high-probability bounds for the same collection of martingales to establish uniform integrability.\footnote{We say a collection of random variables $(X_\alpha)_{\alpha \in A}$ is uniformly integrable if, for any $\delta > 0$, there exists a constant $K_\delta > 0$ such that
\[
\sup_{\alpha \in A}\E\left[|X_\alpha|\mathbbm{1}\left\{|X_\alpha| > K_\delta\right\}\right] \leq \delta.
\]} The exact statements and derivations of the bounds leveraged below can be found in Appendix~\ref{app:martingale-bounds}.
Note that we do not concern ourselves with measurability of stated suprema in the sequel. The results could naturally be formulated in terms of outer expectations and measures, or we could posit a separability assumption on the indexing policy class $\Pi$. We choose not to do this for the sake of streamlining the section.
Before presenting the proof, we prove two brief technical devices.

\begin{lemma}
\label{lem:conv-prob-decrease}
Suppose $(U_n)_{n \geq 1}$ is a sequence of real-valued random variables satisfying $U_n \xrightarrow[n\rightarrow \infty]{\P} 0$. Further, suppose that $(a_n)_{n \geq 1}$ is a deterministic, positive sequence such that $\lim_{n \rightarrow \infty}a_n = 0$. Then, there exists a sequence $(\delta_n)_{n \geq 1}$ satisfying (i) $\delta_n > a_n$, (ii) $\delta_n \downarrow 0$, and (iii) $\lim_{n \rightarrow \infty}\P(U_n \geq \delta_n) = 0$.
\end{lemma}

\begin{proof}
We start by constructing a sequence $(\epsilon_n)_{n \geq 1}$ satisfying points (ii) and (iii) from the lemma statement. Consider the sequence of indices $(N_k)_{k \geq 0}$ inductively defined by $N_0 = 1$ and
\[
N_k := \min\bigg\{n > N_{k - 1} : \P\left(U_m \geq 2^{-k}\right) \leq 2^{-k} \text{ for all } m \geq n\bigg\}.
\]
By assumption, $U_n \xrightarrow[n \rightarrow \infty]{\P} 0$, and hence
\[
\lim_{n \rightarrow \infty}\P\left(U_n \geq 2^{-k}\right) = 0 \quad \text{for all } k\geq 1.
\]
Hence, each $N_k$ is finite. Now, define the function $k(n)$ as the unique index $k$ such that $n \in [N_{k}, N_{k + 1})$. We define the sequence $(\epsilon_n)_{n \geq 1}$ by $\epsilon_n := 2^{-k(n)}$. Observe that, by construction, $\lim_{n \rightarrow \infty}\epsilon_n = 0$. Further, we have
\[
\lim_{n \rightarrow \infty}\P\left(U_n \geq \epsilon_n\right) \leq \lim_{n \rightarrow \infty}2^{-k(n)} = 0.
\]
Hence, the sequence $(\epsilon_n)_{n \geq 1}$ satisfies the last two points. Now, define $\wb{a}_n := \sup_{m \geq n}a_m$, and set $\delta_n := \epsilon_n \lor 2\wb{a}_n$. Then we trivially have $\delta_n > a_n$, $\delta_n \downarrow 0$, and lastly
\[
\P\left(U_n \geq \delta_n\right) \leq \P\left(U_n \geq \epsilon_n\right).
\]
Hence, the result follows.
\end{proof}

The next result is standard, and so no proof is provided.

\begin{prop}[Vitali's Theorem]
\label{prop:vitali}
Suppose $(X_n)_{n \geq 1}$ are a sequence of random variables and $X$ is a random variable in a measure space $(\Omega, \calF, P)$ such that $X_n \in L^p(P)$ for each $n$ and $X \in L^p(P)$ for some $p \geq 1$. The following are equivalent.
\begin{enumerate}
    \item $\|X_n - X\|_{L^p(P)} \xrightarrow[n \rightarrow \infty]{} 0$, i.e.\ $(X_n)_{n \geq 1}$ converge in $L^p$ to $X$.
    \item $X_n \xrightarrow[n \rightarrow \infty]{\P} X$ and $(|X_n|^p)_{n \geq 1}$ is uniformly integrable.
\end{enumerate}
\end{prop}

The final technical lemma ensures that, for any $\epsilon > 0$, the set of policies that disagree from any optimal policy on an $\epsilon$-fraction of the population have regret that is well-separated from zero.

\begin{lemma}
\label{lem:distance}
Consider the function $\psi : \R_{\geq 0} \rightarrow \R_{\geq 0}$ given by
\[
\psi(\epsilon) := \inf_{\substack{\pi \in \Pi \\ d(\pi, \Pi^\ast) \geq \epsilon}}\left\{V^\ast  - V(\pi)\right\}.
\]
Under Assumptions~\ref{ass:observations},  \ref{ass:disagreement}, the first part of Assumption~\ref{ass:estimation}, and the assumption that $\Pi$ is closed under the disagreement metric, we have $\psi(\epsilon) > 0$ for all $\epsilon > 0$.
\end{lemma}

\begin{proof}
First, note that Proposition~\ref{prop:vc-bound-pi} implies that $\log N(\epsilon, \Pi, d) < \infty$ for any $\epsilon  > 0$, and hence $\Pi$ is totally bounded under the disagreement metric. Next, since we have assumed that $\Pi$ is closed, it follows by total boundedness that $\Pi$ is also compact. Noting that $d(\cdot,  \Pi^\ast) : \Pi \rightarrow \R_{\geq 0}$ is a continuous map, it follows that the set
\[
\Pi_\epsilon := \Big\{\pi \in \Pi : d(\pi, \Pi^\ast) \geq \epsilon\Big\}
\]
is also closed, and hence compact. Since the value mapping $V : \Pi \rightarrow \R$ is continuous under the above assumptions, the infimum used in defining $\psi(\epsilon)$ is attained by some policy $\pi_\epsilon \in \Pi_\epsilon$. But, by definition of $\Pi_\epsilon$, $\Reg(\pi_\epsilon) > 0$, proving the claim.
\end{proof}

With the above technical devices, alongside the martingale maximal inequalities shown in the next appendix, we can prove our fast regret theorem. In what follows, recall that, for any $\gamma$, we have defined $S_{n, \gamma} := \sum_{t = 1}^n t^{\gamma}$. Further, note that $S_{n, \gamma} \asymp n^{1 + \gamma}$ for any $\gamma \in (-1, 0]$.

\begin{proof}[Proof of Theorem~\ref{thm:fast-regret}]
Our proof proceeds in a few steps, so we first provide an outline of our argument. Our first course of action is to prove the weaker result that regret vanishes in probability as the sample size grows, i.e.\ that $\Reg(\wh{\pi}_n) = o_\P(1)$. Then, we show that under our assumptions, convergence of the regret implies that the policy learned through weighted empirical risk minimization localizes around $\Pi^\ast$, the set of optimal policies. Lastly, we show that we can leverage this localization to show that regret actually scales as $\Reg(\wh{\pi}_n) = o_\P(n^{(\alpha - 1)/2})$.

\paragraph{Showing Vanishing Regret}

First, we establish that the regret vanishes as the overall sample size grows. In what follows, for shorthand, we define the empirical policy value $\wh{V}_n : \Pi \rightarrow \R$ by
\[
\wh{V}_n(\pi) := \frac{1}{S_{n, -\alpha}}\sum_{t = 1}^n t^{-\alpha}\wt{Y}_t(\pi).
\]
Observe the following basic inequality:
\begin{align*}
\Reg(\wh{\pi}_n) &= V^\ast - V(\wh{\pi}_n) \\
&= \left(V^\ast - \wh{V}_n(\pi^\ast)\right) -\left(V(\wh{\pi}_n) - \wh{V}_n(\wh{\pi}_n)\right) + \underbrace{\left(\wh{V}_n(\pi^\ast) - \wh{V}_n(\wh{\pi}_n)\right)}_{\leq 0} \\
&\leq \left(V^\ast - \wh{V}_n(\pi^\ast)\right) -\left(V(\wh{\pi}_n) - \wh{V}_n(\wh{\pi}_n)\right) \\
&= \frac{M_n(\wh{\pi}_n, \pi^\ast)}{S_{n, -\alpha}} \\
&\leq \frac{1}{S_{n, -\alpha}}\sup_{d(\pi, \rho) \leq 1}\left|M_n(\pi, \rho)\right|,
\end{align*}
where, for any $\pi, \rho \in \Pi$, $M_n(\pi, \rho)$ is as in Appendix~\ref{app:martingale-bounds} in Equation~\eqref{eq:diff-martingale}.

Now, Theorem~\ref{thm:local-sup-bound} implies that
\begin{align*}
\E\left[\frac{1}{S_{n, -\alpha}}\sup_{d(\pi, \rho) \leq 1}\left|M_n(\pi, \rho)\right|\right] \lesssim  n^{\frac{\alpha - 1}{2}}\sqrt{\frakv\log\left(A\right)}  + n^{\alpha - 1}\frakv\log\left(A\right)\log(n)^2.
\end{align*}
where again $\frakv$ denotes the VC dimension of the disagreement class $\calD$, and $A \geq e$ is some constant. It is thus apparent that for any $\epsilon >0$, we have
\begin{align*}
&\lim_{n \rightarrow \infty}\P\Big(\Reg(\wh{\pi}_n) \geq \epsilon\Big) \leq \lim_{n \rightarrow \infty}\frac{1}{\epsilon}\E\left[\frac{1}{S_{n, -\alpha}}\sup_{d(\pi, \rho) \leq 1}\left|M_n(\pi, \rho)\right|\right] \\
&\quad \leq \lim_{n \rightarrow \infty} \frac{1}{\epsilon}n^{\frac{\alpha - 1}{2}}\sqrt{\frakv\log(A)} + \lim_{n \rightarrow \infty} \frac{1}{\epsilon}n^{\alpha - 1}\frakv\log(A)\log(n)^2\\
&\quad = 0.
\end{align*}
In other words, we have shown that $\Reg(\wh{\pi}_n) = o_\P(1)$.

\paragraph{Localization of the Learned Policy}
Now, we use the fact that $\Reg(\wh{\pi}_n) = o_\P(1)$ to show that $d(\wh{\pi}_n, \Pi^\ast) = o_\P(1)$, where we recall that
\[
d(\pi, S) := \inf_{\rho \in S}d(\pi, \rho) = \inf_{\rho \in S}\P\left(\pi(X) \neq \rho(X)\right)^{1/2}
\]
for any $\pi \in \Pi$ and $S \subset \Pi$. Observe the following inclusion for any $\epsilon > 0$:
\[
\Big\{d(\wh{\pi}_n, \Pi^\ast) \geq \epsilon  \Big\} \subset \Big\{\Reg(\wh{\pi}_n) \geq \psi(\epsilon)\Big\},
\]
where the function $\psi$ is as in Lemma~\ref{lem:distance}. Since $\psi(\epsilon) > 0$ by assumption, this implies that
\begin{align*}
\lim_{n \rightarrow \infty}\P\Big(d(\wh{\pi}_n, \Pi^\ast) \geq \epsilon \Big) \leq \lim_{n \rightarrow \infty}\P\Big(\Reg(\wh{\pi}_n) \geq \psi(\epsilon)\Big) = 0.
\end{align*}

\paragraph{Showing Faster Regret}

Now we conclude the proof. First, since we know $d(\wh{\pi}_n, \Pi^\ast) \xrightarrow[n \rightarrow \infty]{\P} 0$, Lemma~\ref{lem:conv-prob-decrease} implies the existence of a  sequence $(\delta_n)_{n \geq 1}$ satisfying $\delta_n \downarrow 0$, $\lim_{n \rightarrow \infty}\P(d(\wh{\pi}_n, \Pi^\ast) \geq \delta_n) = 0$, and decaying sufficiently slowly as to satisfy
\begin{equation}
\label{eq:limit-rate}
\lim_{n \rightarrow \infty} n^{\frac{\alpha - 1}{2}}\log(A/\delta_n)\log(n)^2 = 0.
\end{equation}
More explicitly, set $\kappa := \frac{1 - \alpha}{4}$ and set $a_n := e^{-n^\kappa}$. Then, the specific sequence $\delta_n$ outlined in the proof of Lemma~\ref{lem:conv-prob-decrease} will satisfy:
\[
n^{\frac{\alpha - 1}{2}}\log(A/\delta_n)\log(n)^2 \leq n^{\frac{\alpha - 1}{2}}\left\{\log(A) + n^\kappa\right\}\log(n)^2 \xrightarrow[n \rightarrow \infty]{} 0.
\]
Using this shrinking sequence, we can define the ``good'' event $G_n$ by
\[
G_n := \Big\{d(\wh{\pi}_n, \Pi^\ast) \leq  \delta_n\Big\}.
\]
We see that, for any $\epsilon > 0$, we have
\begin{align*}
\P\left(\sqrt{S_{n, -\alpha}}\Reg(\wh{\pi}_n) \geq \epsilon\right) &= \P\left(\left\{\sqrt{S_{n, -\alpha}}\Reg(\wh{\pi}_n) \geq \epsilon\right\} \cap G_n^c\right) + \P\left(\left\{\sqrt{S_{n, -\alpha}}\Reg(\wh{\pi}_n) \geq \epsilon\right\} \cap G_n\right) \\
&\leq \P(G_n^c) + \P\left(\left\{\sqrt{S_{n, -\alpha}}\Reg(\wh{\pi}_n) \geq \epsilon\right\} \cap G_n\right).
\end{align*}
We clearly have $\P(G_n^c) \rightarrow 0$ by our assumption on the sequence $(\delta_n)$. To show the result, we just need to show that the latter probability tends towards zero. To see this behavior, first note that
\begin{align*}
&\P\left(G_n \cap \left\{\sqrt{S_{n, - \alpha}} \Reg(\wh{\pi}_n) \geq \epsilon\right\}\right) = \P\left(\sqrt{S_{n, -\alpha}}\Reg(\wh{\pi}_n)\1_{G_n} \geq \epsilon\right) \\
&\qquad \leq \frac{1}{\epsilon}\E\left[\1_{G_n} \sqrt{S_{n, -\alpha}}\Reg(\wh{\pi}_n)\right] \\
&\qquad \leq \frac{1}{\epsilon}\sqrt{S_{n, -\alpha}}\E\left[\frac{1}{S_{n, -\alpha}} \sup_{d(\pi, \rho) \leq \delta_n}\left|M_n(\pi, \rho)\right|\right] \\
&\qquad \lesssim \frac{1}{\epsilon}\left\{\delta_n\sqrt{\frakv\log\left(\frac{A}{\delta_n}\right)} + n^{\frac{\alpha - 1}{2}}\frakv\log\left(\frac{A}{\delta_n}\right)\log(n)^2\right\}.
\end{align*}

\noindent By our assumption on the rate of decay of $(\delta_n)$, it follows that
\[
\lim_{n \rightarrow \infty}\P\left(G_n \cap\left\{\sqrt{S_{n, -\alpha}}\Reg(\wh{\pi}_n) \geq \epsilon\right\}\right) = 0,
\]
which thus implies
\[
\lim_{n \rightarrow \infty}\P\Big(\sqrt{S_{n, -\alpha}}\Reg(\wh{\pi}_n) \geq \epsilon\Big) = 0.
\]
Since $\epsilon > 0$ was arbitrary, this implies that $\sqrt{S_{n, -\alpha}}\Reg(\wh{\pi}_n) = o_\P(1)$, i.e.\ that
\[
\Reg(\wh{\pi}_n) = o_\P(S_{n, -\alpha}^{-1/2}) = o_\P\left(n^{\frac{\alpha - 1}{2}}\right),
\]
which is the first desired result.

We now strengthen the result by showing $\E\left[\Reg(\wh{\pi}_n)\right] = o\left(n^{\frac{\alpha - 1}{2}}\right)$, or equivalently that $\E\left[\sqrt{S_{n, -\alpha}}\Reg(\wh{\pi}_n)\right] = o(1)$. By Vitali's Theorem (Proposition~\ref{prop:vitali} above), since we have shown $\sqrt{S_{n, -\alpha}}\Reg(\wh{\pi}_n) = o_\P(1)$, it suffices to argue that the sequence $\big(\sqrt{S_{n, -\alpha}}\Reg(\wh{\pi}_n)\big)_{n \geq 1}$ is uniformly integrable. We do precisely this in Theorem~\ref{thm:uniform-integrability} below. Thus, we have shown the result.

\end{proof}

\subsection{Establishing Uniform Integrability}

The one missing piece in the proof of Theorem~\ref{thm:fast-regret} above is showing that the process $\sqrt{S_{n, -\alpha}}\Reg(\wh{\pi}_n)$ is uniformly integrable. We now establish this in the following theorem.

\begin{theorem}
\label{thm:uniform-integrability}
Assume the same setup as Theorem~\ref{thm:fast-regret}.
Then, for all $q \in [1, 2]$, we have
\[
\sup_{n \geq 1}\E\left[S_{n, -\alpha}^{q/2}\Reg(\wh{\pi}_n)^q\right] < \infty.
\]
Consequently, the process $\big(\sqrt{S_{n, -\alpha}}\Reg(\wh{\pi}_n)\big)$ is uniformly integrable.
\end{theorem}

We note that the uniform integrability follows from the uniform expectation bound via Poussin's lemma, which we state below.

\begin{prop}[Poussin's Lemma]
\label{prop:pous}
A sequence of integrable random variables $(X_n)_{n \geq 1}$ is uniformly integrable if and only if there is an increasing convex function $\Phi : \R_{\geq 0} \rightarrow \R_{\geq 0}$ such that
\[
\lim_{t \rightarrow \infty}\frac{\Phi(t)}{t} = \infty \quad \text{and} \quad \sup_n \E\left[\Phi(|X_n|)\right] < \infty.
\]

\end{prop}

We now present our proof of uniform integrability.

\begin{proof}[Proof of Theorem~\ref{thm:uniform-integrability}]
Throughout the proof, we assume $n \geq 2$ to apply our concentration inequalities from Appendix~\ref{app:martingale-bounds}. Note that the finiteness of the moment for the case $n = 1$ is trivial, as outcomes are assumed to be bounded.

First, observe that the following basic inequality on regret holds:
\begin{align*}
\Reg(\wh{\pi}_n) = V^\ast - V(\wh{\pi}_n) \leq 2\sup_{\pi \in \Pi} \left|\wh{V}_n(\pi) - V(\pi)\right| = 2\sup_{\pi \in \Pi}\left|G_{n}(\pi)\right|,
\end{align*}
where we have defined, for any $\pi \in \Pi$, the process $(G_n(\pi))_{n \geq 1}$ by
\[
G_n(\pi) := \frac{1}{S_{n, -\alpha}}\sum_{t = 1}^nw_t\left\{\frac{\1\{A_t = \pi(X_t)\}}{p_t(A_t \mid X_t)}Y_t - V(\pi)\right\},
\]
where $w_t = t^{-\alpha}$.

Note that by the basic inequality outlined above, we clearly have
\[
\sup_{n \geq 1}\E\left[S_{n, -\alpha}^{q/2}\Reg(\wh{\pi}_n)^q\right] \lesssim \sup_{n \geq 1}\E\left[S_{n, -\alpha}^{q/2}\left(\sup_{\pi \in \Pi}\left|G_n(\pi)\right|\right)^q\right],
\]
so it suffices to show that the right hand side is finite. We aim to establish this via an application of the layer cake theorem for each $n \geq 1$. In the sequel, we define the random variable $U_n := \sup_{\pi \in \Pi}\left|G_n(\pi)\right|$ to ease notation. First, note that for any $n \geq 1$ and any deterministic $D_n$, we have the upper bound
\begin{align*}
S_{n, -\alpha}^{q/2} \E U_n^q &= S_{n, -\alpha}^{q/2}\E\left[(U_n - D_n + D_n)^q\right] \\
&\leq S_{n, -\alpha}^{q/2}2^{q - 1}\left\{\E\left[(U_n - D_n)_+^q\right] + D_n^q\right\} \\
&\equiv 2^{q - 1}\left\{S_{n,- \alpha}^{q/2}\E\left[R_n^q\right] + S_{n,-\alpha}^{q/2}D_n^q\right\},
\end{align*}
where we have set $R_n := (U_n - D_n)_+ \equiv \max\{U_n - D_n, 0\}$. We now show that there is a choice of $D_n$ such that both terms on the right hand side are uniformly bounded over $n \geq 1$. First, since $R_n \geq 0$ by construction, we have by the layer cake theorem:
\begin{equation}
\label{eq:layer-cake}
\E\left[R_n^q\right] = \int_0^\infty \P\left(R_n > t\right)q t^{q-1} dt.
\end{equation}
We now use our high-probability bound established in Theorem~\ref{thm:high-prob-chain} to control the tail probability.
First, defining the class of functions
\[
\calG := \Big\{ g_\pi(a \mid x) := \1\{a = \pi(x)\} : \pi \in \Pi\Big\},
\]
we can equivalently view $\left(G_n(\pi)\right)_{\pi \in \Pi}$ as a process indexed by the function class $\calG$, i.e.\ as a process $(G_n(g))_{g \in \calG}$. Further, it is clear that $G_n(g)$ takes the form of the process defined in Equation~\eqref{eq:martingale} in Appendix~\ref{app:martingale-bounds} with $\ell = 1$. Moreover, each function $g \in \calG$ is bounded between $0$ and $1$, and the assumption that the disagreement class $\calD$ has finite VC-dimension (Assumption~\ref{ass:disagreement}) implies that $\calG$ has finite pseudo-dimension, per Proposition~\ref{prop:fn_class_pseudo}.\footnote{Note that this follows from taking $q(a \mid x) \equiv 0$ in the proposition. Further, note that Assumption~\ref{ass:logging} is not needed for $\calG$ (denoted $\calG_1$ in the proposition) to have finite pseudo-dimension.} Hence, by Lemma~\ref{lem:pseudo-dimension-covering}, there is some constant $\frakv > 0$ and another constant $A > e$ (potentially different from the constants in previous subsection) such that the uniform covering number of $\calG$ is bounded as
\[
\sup_{Q}\log N\left(\epsilon, \calG, L^2(Q)\right) \lesssim  \frakv \log\left(\frac{A}{\epsilon}\right) \quad \text{for all } 0 < \epsilon \leq 1.
\]
Thus, since the conditions of Theorem~\ref{thm:high-prob-chain} are satisfied, we have for any $x > 0$
\[
\P\left(U_n \geq Cn^{\frac{\alpha - 1}{2}}\left(1 + \sqrt{x}\right) + Cn^{\alpha - 1}\log(n)\left(\log(n) + x\right)\right) \leq Ce^{-x} + C\exp\left\{-cn^{1 - \alpha}\right\},
\]
where $C, c > 0$ are appropriately large constants that potentially depend on $\frakv, \alpha,$ and $\calA$, but not on $n$ or $x$. We now convert this tail bound into one of the form outlined in the right hand side of Equation~\eqref{eq:layer-cake}. Grouping terms, we have that
\[
\P\left(U_n \geq D_n + C\left[a_n\sqrt{x} + b_n x\right]\right) \leq C\exp\{-x\} + C\exp\left\{-cn^{1 - \alpha}\right\},
\]
where we have defined $a_n := n^{\frac{\alpha - 1}{2}}$, $b_n := n^{\alpha - 1}\log(n)$, and $D_n := Cn^{\frac{\alpha - 1}{2}} + Cn^{\alpha - 1}\log(n)^2$. With $R_n$ as defined above, we have
\begin{align*}
\P\left(R_n \geq C\left[a_n\sqrt{x} + b_n x\right]\right) &= \P\left(U_n \geq D_n + C\left[a_n\sqrt{x} + b_n x\right]\right) \\
&\leq C\exp\{-x\} + C\exp\left\{-cn^{1 - \alpha}\right\}
\end{align*}
for any $x > 0$. Now, let $t > 0$ be arbitrary, and set $x_t > 0$ as
\[
x_t := \eta \min\left\{\frac{t^2}{a_n^2}, \frac{t}{b_n}\right\},
\]
where $\eta > 0$ is any constant satisfying $C(\eta + \sqrt{\eta}) \leq 1$. For the given choice of $x_t$, we have
\begin{align*}
C\left[a_n\sqrt{x_t} + b_n x_t\right] &= C\left[a_n\sqrt{\eta \min\left\{ \frac{t^2}{a_n^2}, \frac{t}{b_n}\right\}} + b_n \eta \min\left\{\frac{t^2}{a_n^2}, \frac{t}{b_n}\right\}\right] \\
&\leq C\left[a_n \sqrt{\eta} \sqrt{\frac{t^2}{a_n^2}} + b_n \eta  \frac{t}{b_n}\right] \\
&= C(\sqrt{\eta} + \eta)t \\
&\leq t.
\end{align*}
Consequently, we have
\begin{align*}
\P\left(R_n \geq t\right) &\leq \P\left(R_n \geq C\left[a_n\sqrt{x_t} + b_n x_t\right]\right) \\
&\leq C\exp\left\{-\eta\min\left\{\frac{t^2}{a_n^2}, \frac{t}{b_n}\right\}\right\} + C\exp\left\{-cn^{1 - \alpha}\right\}.
\end{align*}
We now return to the expression in Equation~\eqref{eq:layer-cake}. Note that we have the following almost sure upper bound
\begin{align*}
R_n &= (U_n - D_n)_+ \leq U_n = \frac{1}{S_{n, -\alpha}}\sup_{\pi \in \Pi}\left|\sum_{t = 1}^n w_t\left\{\frac{g_\pi(A_t \mid X_t)}{p_t(A_t \mid X_t)}Y_t - V(\pi)\right\}\right| \\
&\lesssim n^{\alpha - 1}\sup_{\pi \in \Pi}\sum_{t = 1}^n t^{-\alpha}\left\{\left|\frac{g_\pi(A_t \mid X_t)}{p_t(A_t \mid X_t)}Y_t\right| + \left|V(\pi)\right|\right\} \\
&\lesssim n^{\alpha - 1} \cdot n = n^{\alpha},
\end{align*}
where the last line follows from the fact that $V(\pi), Y_t \in [0, 1]$ almost surely and $p_t(A_t \mid X_t) \geq \lb{c}t^{-\alpha}$ for some absolute constant $\lb{c} > 0$.
Thus, letting $B > 0$ denote the constant such that $R_n \leq Bn^\alpha$ almost surely, we have
\begin{align*}
\E\left[R_n^q\right] &= \int_0^{\infty}\P\left(R_n > t\right)q t^{q-1} dt = \int_0^{Bn^\alpha} \P(R_n > t)qt^{q - 1}dt \\
&\leq \int_0^{Bn^\alpha} \left(C\exp\left\{-\eta\min\left\{\frac{t^2}{a_n^2}, \frac{t}{b_n}\right\}\right\} + C\exp\left\{-cn^{1 - \alpha}\right\}\right)q t^{q-1} dt \\
&\leq Cq\underbrace{\int_0^\infty \exp\left\{-\eta\min\left\{\frac{t^2}{a_n^2}, \frac{t}{b_n}\right\}\right\} t^{q-1} dt}_{=: T_1} + Cq \underbrace{\int_0^{Bn^\alpha} \exp\left\{-cn^{1 - \alpha}\right\} t^{q-1} dt}_{=: T_2}.
\end{align*}

We now bound $T_1$ and $T_2$ separately. First, we see that we have the bound
\begin{align*}
T_1 &\leq \int_0^\infty \exp\left\{-\eta\frac{t^2}{a_n^2}\right\}t^{q - 1}dt + \int_0^\infty\exp\left\{-\eta \frac{t}{b_n}\right\} t^{q - 1}dt \\
&\leq \frac{a_n^q}{2 \eta^{q / 2}} \int_0^\infty u^{q/2 - 1}e^{-u}du+ \frac{b_n^q}{\eta^q}\int_0^\infty v^{q - 1}e^{-v}dv \\
&= \frac{a_n^q}{2 \eta^{q/2}} \Gamma(q/ 2) + \frac{b_n^q}{\eta^q}\Gamma(q) \\
&\lesssim n^{\frac{q(\alpha - 1)}{2}} + n^{q(\alpha - 1)}\log(n)^{q},
\end{align*}
where the second inequality follows from performing the change of variables $u := \frac{\eta t^2}{a_n^2}$ in the first integral and $v := \frac{\eta t}{b_n}$ in the second integral, where $\Gamma(q) := \int_0^\infty s^{q - 1}e^{-s}ds$ is the Gamma function. We now bound $T_2$.
\begin{align*}
T_2 = \exp\left\{-cn^{1 - \alpha}\right\}\int_0^{B n^\alpha} t^{q - 1}dt \asymp n^{\alpha q} \exp\left\{-cn^{1 - \alpha}\right\}.
\end{align*}
Thus, we have shown that
\[
\E\left[R_n^q\right] \lesssim n^{\frac{q(\alpha - 1)}{2}} + n^{q(\alpha - 1)}\log(n)^q + n^{\alpha q}\exp\left\{-cn^{1 - \alpha}\right\}.
\]
Noting that $S_{n, -\alpha}^{q/2} \asymp n^{-\frac{q(\alpha - 1)}{2}}$, we see that
\begin{align*}
S_{n, -\alpha}^{q/2}\E\left[U_n^q\right] &\lesssim S_{n, -\alpha}^{q/2}\E\left[R_n^q\right] + S_{n, -\alpha}^{q/2}D_n^q \\
& \lesssim S_{n, -\alpha}^{q/2}\Big[n^{\frac{q(\alpha - 1)}{2}} + n^{q(\alpha - 1)}\log(n)^q  \\
&\qquad  + n^{\alpha q}\exp\left\{-cn^{1 - \alpha}\right\} + n^{\frac{q(\alpha - 1)}{2}}  + n^{q(\alpha - 1)}\log(n)^{2q}\Big] \\
&\asymp n^{\frac{q(1 - \alpha)}{2}}\Big[n^{\frac{q(\alpha - 1)}{2}} + n^{q(\alpha - 1)}\log(n)^q + n^{\alpha q}\exp\left\{-cn^{1 - \alpha}\right\} \\
&\qquad  + n^{\frac{q(\alpha - 1)}{2}}  + n^{q(\alpha - 1)}\log(n)^{2q}\Big]  \\
&\asymp 1 + n^{\frac{q(\alpha - 1)}{2}}\log(n)^{2q} + n^{\frac{q(1 + \alpha)}{2}}\exp\left\{-cn^{1 - \alpha}\right\}.
\end{align*}
That is, there is a constant $D > 0$ independent of $n$ such that
\[
\E\left[S_{n, -\alpha}^{q/2}\Reg(\wh{\pi}_n)^q\right] \leq D\left[1 + n^{\frac{q(\alpha - 1)}{2}}\log(n)^{2q} + n^{\frac{q(\alpha + 1)}{2}}\exp\left\{-cn^{1 - \alpha}\right\}\right].
\]
Taking a limit superior over $n$ gives
\[
\limsup_{n \rightarrow \infty}\E\left[S_{n, -\alpha}^{q/2}\Reg(\wh{\pi}_n)^q\right] \leq D,
\]
which implies $\sup_{n \geq 1}\E\left[S_{n, -\alpha}^{q/2}\Reg(\wh{\pi}_n)^q\right] < \infty$.

The uniform integrability of the process $\big(\sqrt{S_{n, -\alpha}}\Reg(\wh{\pi}_n)\big)_{n\geq 1}$ now follows immediately by Proposition~\ref{prop:pous} applied with the mapping $\Phi(t) := t^{q}$ for any $q > 1$. Thus, the result holds.

\end{proof}

\section{Combinatorial Dimension, Covering Numbers, and Equivalence of Assumptions}
\label{app:combinatorial-dimension}

In this appendix, we discuss the various notions of combinatorial dimension that we leverage throughout both the main body of this paper and the other appendices. We start by drawing connections between our main complexity class assumption (\ref{ass:disagreement}) and the notion of Natarajan dimension from the multi-class classification literature~\citep{daniely2015multiclass, natarajan1989learning}. Next, we use a generalization of Sauer's lemma to multiclass setting (see \citet{haussler1995generalization}) to bound the uniform covering numbers of the policy class $\Pi$ under the natural disagreement metric. Lastly, we define pseudo-dimension and discuss how composing our policy class with controlled complexity logging policy classes results in classes with finite pseudo-dimension, and hence finite uniform covering numbers. These uniform covering numbers will allow us to show the almost sure consistency of our plug-in variance estimates (Appendix~\ref{app:variance}) by using the uniform martingale bounds discussed in Appendix~\ref{app:martingale-bounds}.

We start by discussing the equivalence of our main assumption on the VC dimension of the disagreement class associated with $\Pi$ (Assumption~\ref{ass:disagreement}) with the assumption of finite Natarajan dimension. In particular, we show that these two notions are equivalent up to logarithmic factor depending on the number of treatments/classes $|\calA|$.  We start by defining Natarajan  shattering and Natarajan dimension.

\begin{definition}
Let $\calG \subset \big\{g : \calX \rightarrow \calA\big\}$ be a collection of functions and let $S \subset \calX$ be a set of points. We say that $\calG$ \textit{Natarajan shatters} $S$ if there exist two functions $f_0, f_1 : S \rightarrow \calA$ such that (i) $f_0(x) \neq f_1(x)$ for all $x \in S$ and (ii) for every $T \subset S$, there is $g \in \calG$ satisfying
\begin{equation}
\label{eq:natarajan-shatter}
\forall x \in T, g(x) = f_0(x) \quad \text{and} \quad \forall x \in S \setminus T, g(x) = f_1(x).
\end{equation}
We say the \textit{Natarajan dimension} of the class $\calG$ is the size of the largest set Natarajan shattered by $\calG$.
\end{definition}

One important property of classes with finite Natarajan dimension is that the number of labels they assign to a collection of points only grows polynomially in the number of samples. The following serves as an analogue of Sauer's lemma for VC classes.

\begin{lemma}[Generalized Sauer's Lemma, Corollary 5 of \citet{haussler1995generalization}]
\label{lem:gen-sauer}
Suppose $\calG \subset \calA^\calX$ is a class of functions with finite Natarajan dimension. Then, for any set $S \subset \calX$ of cardinality $n \geq 1$, letting
\[
\calG(S) := \bigg\{\big(g(x_1), \dots, g(x_n)\big) : g \in \calG\bigg\} \subset \calA^n,
\]
we have
\[
\big|\calG(S)\big| \leq \sum_{i = 0}^{d_N(\calG)} \binom{n}{i}\binom{|\calA|}{2}^i \leq \left(\frac{|\calA|^2en}{2d_N(\calG)}\right)^{d_N(\calG)},
\]
with the final inequality holding only if $n \geq d_N(\calG)$.
\end{lemma}

We note that, in the setting of binary functions (so $|\calA| = 2$) one can obtain the slightly sharper rate of $|\calG(S)| \leq \left(\frac{en}{d_{\vc}(\calG)}\right)^{d_{\vc}(\calG)}$ instead~\citep{van2014probability}, which is just the standard rate obtained via Sauer's lemma. We leverage the general result above in controlling uniform covering numbers in this appendix, and leverage the standard version of Sauer's lemma in the proof of Lemma~\ref{lem:vc-disagreement} in Appendix~\ref{app:martingale-bounds}.

The following proposition formally relates the VC dimension of the disagreement class to the Natarajan dimension of the underlying policy class. In particular, it ensures that the two are equivalent up to only logarithmic factors in the number of actions/labels.

\begin{prop}
\label{prop:vc-to-natarajan}
Let $\Pi$ be a policy class and $\calD$ the corresponding disagreement class. Then, we have
\[
d_N(\Pi) \leq d_{\vc}(\calD) \lesssim \log(e|\calA|)d_N(\Pi).
\]
\end{prop}

\begin{proof}
We first argue that $d_N(\Pi) \leq d_{\vc}(\calD)$. Let $S \subset \calX$ be Natarajan shattered by $\Pi$ with witness functions $f_0, f_1 : S \rightarrow \calA$. For any $T \subset S$, we let $\pi_T \in \Pi$ denote an arbitrary policy satisfying Equation~\eqref{eq:natarajan-shatter}. We examine the policy $\pi_\emptyset : \calX \rightarrow \calA$, which satisfies $\pi_\emptyset(x) = f_1(x)$ for all $x \in S$. Enumerate $S$ as $x_1, \dots, x_n$, and let $b_1, \dots, b_n \in \{0, 1\}$ give any binary labeling to be assigned to the elements. Let $T_b := \{x_i \in S : b_i = 1\}$ and setting $g_b(x) := \1\left\{\pi_{T_b}(x) \neq \pi_\emptyset(x)\right\}$, observe that we have $g(x_i) = b_i$ for all $i \in [n]$, and hence $\calD$ shatters $S$ (in the sense of VC dimension). Thus, $d_N(\Pi) \leq d_{\vc}(\calD)$.

Now, we prove the upper bound on $d_{\vc}(\calD)$ in terms of $d_N(\Pi)$. Let $S \subset \calX$ be shattered by $\calD$, set $n := |S|$, and let $d := d_N(\Pi)$. First, note that if $d = 0$, then $d_{\vc}(\calD) =0$ as well. Now suppose $d \geq 1$. If $n < d$, then we trivially have
\[
n \leq d \leq 16\log(e|\calA|) d.
\]
Hence, suppose that $n \geq d$ (as a shattered set of this size will always exist by the first inequality). Consider the mapping $\iota : \Pi(S) \times \Pi(S) \rightarrow \calD(S)$ given by
\[
\iota(v, w) := \big(\1\{v_1 \neq w_1\}, \dots, \1\{v_n \neq w_n\}\big).
\]
Observe that the above map is clearly a surjection, i.e., for each bit vector $b \in \calD(S)$, there exists at least one tuple $(v, w) \in \Pi(S) \times \Pi(S)$ satisfying $\iota(v, w) = b$. This implies that we have $|\Pi(S)|^2 \geq |\calD(S)|$. Using this alongside generalized Sauer's lemma (Lemma~\ref{lem:gen-sauer}) yields
\begin{align*}
2^n = |\calD(S)| \leq |\Pi(S)|^2 \leq \left(\frac{|\calA|^2 e n}{2d}\right)^{2d},
\end{align*}
where the first equality follows from the assumption of shattering.
Taking logs on both sides and defining $r := \frac{n}{d}$, we see that
\[
r \log(2) \leq 2\log\left(r|\calA|^2 e\right) \leq 2\log(r) + 4\log(e|\calA|).
\]
We claim now that we must have
\[
r \leq 16 \log(e|\calA|).
\]
Suppose the above inequality doesn't hold. Since $\log(e |\calA|) \geq 1$, we see that $r \geq 16$. Hence, we must simultaneously have $2 \log r  \leq \frac{\log(2)}{2} r$ and $4\log(e|\calA|) \leq \tfrac{r}{4}$. This implies that
\[
r \log(2) \leq \left(\frac{1}{4} + \frac{\log(2)}{2}\right) r < r\log(2),
\]
which follows since $\log(2)/2 > 1/4$. This is a contradiction, and hence we have shown
\[
n \leq 16\log(|\calA|e) d.
\]
Taking the maximum over all such $n$ admitting shattered sets furnishes the claimed result.
\end{proof}

The next proposition bounds the covering numbers of $\Pi$ uniformly under the disagreement metric $d_Q(\pi, \rho) := Q(\pi(X) \neq \rho(X))^{1/2}$. We prove it using results from \citet{daniely2015multiclass} alongside Proposition~\ref{prop:vc-to-natarajan}.

\begin{prop}
\label{prop:vc-bound-pi}
Let $\Pi$ be a policy class satisfying Assumption~\ref{ass:disagreement}, and let $\frakv := d_{\vc}(\calD)$ denote the VC dimension of the disagreement class. Let $Q$ be any probability distribution on $\calX$. Then, for any $0 < \epsilon \leq 1$, we have
\[
\log N\left(\epsilon, \Pi, d_Q\right) \lesssim \frakv\log\left(\frac{A \sqrt{|\calA|}}{\epsilon}\right),
\]
where $A \geq e$ is some constant.
\end{prop}

\begin{proof}
First, define the function class $\calC_\Pi \subset \{0, 1\}^{\calA \times \calX}$ as
\[
\calC_\Pi := \bigg\{c_\pi(a \mid x) := \1\{\pi(x) = a\} : \pi \in \Pi\bigg\}.
\]
We claim that $d_{\vc}(\calC_{\Pi}) \leq \frakv$. To see this, suppose $\calC_\Pi$ shatters a set $S = \left\{(x_1, a_1), \dots, (x_n, a_n)\right\}$.\footnote{Note that this requires $x_1, \dots, x_n$ necessarily to be unique. In particular, if $(x_i, a_i), (x_j, a_j) \in S$ with $x_i = x_j$ but $a_i \neq a_j$, this would require the existence of a policy $\pi \in \Pi$ with $\pi(x_i) = a_i$ and $\pi(x_i) = a_j$, a contradiction.}  Then, there is some policy $\pi_S \in \Pi$ satisfying $\pi_S(x_i) = a_i$ for all $i \in [n]$. Further, for any $T \subset S$, there is some $\pi_{T}$ such that $\pi_{T}(x_i) = a_i$ if and only if $(x_i, a_i) \in T$. Thus, we see that, for any $I \subset [n]$, letting $T_{I^c} := \left\{(x_i, a_i) \in S : i \notin I\right\}$, there is some policy $\pi_{T_{I^c}}$ such that:
\[
\1\left\{\pi_{T_{I^c}}(x_i) \neq \pi_S(x_i)\right\} = 1 \Longleftrightarrow i \in I.
\]
Hence, $\calD$ shatters $\left\{x_i : (x_i, a_i) \in S\right\}$.

Since $g_\pi(a \mid x) \in \{0, 1\}$ for each $a, x \in \calA \times \calX$, we know that $d_{\vc}(\calC_\Pi) = d_P(\calC_\Pi)$, where $d_P$ denotes pseudo-dimension, defined in Definition~\ref{def:pseudo-dimension} below.
Thus, from Lemma~\ref{lem:pseudo-dimension-covering} below, it follows that
\[
\sup_{\wb{Q}}\log N\big(\epsilon, \calC_\Pi, L^2(\wb{Q})\big) \lesssim \frakv\log\left(\frac{A}{\epsilon}\right)
\]
for some absolute constant $A \geq e$, where the supremum is over all distributions on $\calA \times \calX$. Now, let $Q$ be some distribution on $\calX$, and consider the product distribution $\wb{Q} := \calU(\calA) \otimes Q$, where $\calU(S)$ denotes the uniform distribution on a set $S$. It thus follows for policies $\pi, \rho \in \Pi$, the $L^2(\wb{Q})$ distance between $g_\pi$ and $g_\rho$ satisfies
\begin{align*}
\left\|g_\pi - g_\rho\right\|_{L^2(\wb{Q})}^2 &= \frac{1}{|\calA|}\sum_{a \in \calA}\E_{Q}\left(g_\pi(a, X) - g_\rho(a, X)\right)^2 \\
&= \frac{1}{|\calA|}\sum_{a \in \calA}Q\left(\1\big\{\pi(X) = a\big\} \neq \1\big\{\rho(X) = a\big\}\right) \\
&= \frac{2}{|\calA|}Q\left(\pi(X) \neq \rho(X)\right),
\end{align*}
where the final equality follows since $\1\left\{\pi(X) \neq \rho(X)\right\} = 1$ if and only if there are exactly two treatments $a, b \in \calA$ such that $\1\left\{\pi(X) = c\right\} \neq \1\left\{\rho(X) = c\right\}$ for $c \in \{a, b\}$. Thus, we have the following relationship for our covering numbers
\[
\log N\big(\epsilon, \Pi, d_Q\big) \leq \log N\big(\sqrt{2/|\calA|}\epsilon, \calC_\Pi, L^2(\wb{Q})\big) \lesssim \frakv \log\left(\frac{A\sqrt{|\calA|}}{\epsilon}\right)
\]
for some constant $A \geq e$. Since $Q$ was arbitrary, we have shown the desired uniform covering bound.
\end{proof}

To conclude this appendix, we comment on the assumption (Assumption~\ref{ass:covering}) which assumes that the class $\calG_1$ and $\calG_{2, t}$ containing $g_\pi(a \mid x) := \1\{\pi(x) = a\} - q(a \mid x)$ and $t^{-\alpha}\frac{g_\pi(a \mid x)^2}{p_t(a \mid x)}$ have uniform covering numbers of polynomial growth. We directly leverage this assumption as opposed to Assumption~\ref{ass:disagreement} and \ref{ass:logging} in our convergence analysis of our variance estimator due to convenience. We now show that this assumption follows from the two more primitive assumptions noted above. We start by formally defining the notion of pseudo-dimension.
\begin{definition}
\label{def:pseudo-dimension}
Let $\calG \subset \R^{\calX}$ be a class of functions. Define the corresponding subgraph class of function $\calG^\#$ by
\[
\calG^\# := \bigg\{ (x, t) \mapsto \1\{t < g(x)\} : g \in \calG\bigg\}.
\]
Then, the pseudo-dimension of $\calG$ is defined as $d_P(\calG) := d_{\vc}(\calG^\#)$.
\end{definition}
We now aim to show that the two function classes
\begin{align*}
\calG_1 &:= \bigg\{ g_\pi(a \mid x) := \1\{\pi(x) = a\} - q(a \mid x) : \pi \in \Pi\bigg\} \\
\calG_{2, t} &:= \bigg\{ g_\pi(a \mid x)^2r(a \mid x) : \pi \in \Pi, r \in \calP_t\bigg\},
\end{align*}
have finite pseudo-dimension irrespective of $t$ under Assumption~\ref{ass:disagreement} and \ref{ass:logging}, hence justifying Assumption~\ref{ass:covering} in Appendix~\ref{app:clt_proof}.
The following proposition yields a uniform bound (in $t \geq 1$) on the pseudo-dimension of $\calG_1$ and $\calG_{2, t}$. We can then couple this bound with  Lemma~\ref{lem:pseudo-dimension-covering} below to yield polynomial uniform covering numbers.

\begin{prop}
\label{prop:fn_class_pseudo}
Suppose Assumptions~\ref{ass:disagreement} and \ref{ass:logging} hold, let $\frakv := d_{\vc}(\calD)$ and let $\frakp$ satisfy $\sup_{t \geq 1}d_P(\calP_t) \leq \frakp$.  Then, we have $d_P(\calG_1) \leq \frakv$ and $\sup_{t \geq 1}d_P(\calG_{2, t}) \leq 5(\frakv + \frakp)$, where $\calG_1$ and $\calG_{2, t}$ are as defined above.

\end{prop}

\begin{proof}
First, we argue the bound on $d_P(\calG_1)$. Again, consider the class of functions $\calC_\Pi$ as defined in the proof of Proposition~\ref{prop:vc-bound-pi}. Then, since $\calC_\Pi$ is a binary class, pseudo-dimension and VC dimension coincide, and $d_{P}(\calC_\Pi) = d_{\vc}(\calC_\Pi) \leq \frakv$. Further, through standard closure properties of classes of finite pseudo-dimension (see Lemma 2.6.18 of \citet{van1996weak}), we have that $d_P(\calG_1) = d_P(\calC_\Pi - q) = d_P(\calC_\Pi)$, since $q(a \mid x)$ is a fixed function. This yields the first claim.

Next, we argue the bound on $d_P(\calG_{2, t})$. Suppose $\calG_{2, t}^\#$ shatters a collection of points
\[
S = \left\{(x_1, a_1, t_1) ,\dots, (x_n, a_n, t_n)\right\},
\]
where $t_1, \dots, t_n$ denote the thresholds associated with Definition~\ref{def:pseudo-dimension}. Any function in $\calG_{2, t}$ is given by
\[
h_{\pi, r}(a, x) = \big(c_\pi(a, x) - q(a \mid x)\big)^2r(a \mid x)
\]
for some $c_\pi(a, x) = \1\{\pi(x) = a\} \in \calC_\Pi$ and $r \in \calP_t$. By our assumption of shattering, we know that we must have
\[
2^n = \left|\left\{\big(\1\left\{h_{\pi, r}(a_i, x_i) > t_i\right\}\big)_{i = 1}^n : \pi \in \Pi, r \in \calP_t\right\}\right|
\]
Since $d_{\vc}(\calC_\Pi) \leq \frakv$ and since $q$ is fixed, Sauer's lemma implies that there are at most $\sum_{j = 0}^{\min\{\frakv, n\}}\binom{n}{j}$ distinct (non-negative) values that the vector $\big((c_\pi(a_i, x_i) - q(a_i \mid x_i))^2\big)_{i = 1}^n$ can take. Further, letting $w_i^\pi := (c_\pi(a_i, x_i) - q(a_i \mid x_i))^2$ for each $i \in [n]$ and $\pi \in \Pi$, we see that, for any fixed $\pi$, there are at most $\sum_{k = 0}^{\min\{\frakp, n\}}\binom{n}{k}$ values that $\big(\1\{w_i^\pi r(a_i \mid x_i) > t_i\}\big)_{i = 1}^n$ can take.\footnote{This is again by Sauer's lemma. Note this follows from rearranging and noting $\1\{w_i^\pi r(a_i \mid x_i) > t_i\} = 1$ if and only if $\1\{r(a_i \mid x_i) > t_i/w_i^\pi\}$, so long as $w_i^\pi > 0$. Having some indices with $w_i^\pi = 0$ only serves to decrease the number of possible labels.} Assume that $n \geq \max\{\frakv, \frakp\}$ (Else, the claim trivially holds). Then, we must certainly have
\begin{equation}
\label{eq:product_sauer}
2^n \leq \left\{\sum_{j = 0}^{\min\{\frakv, n\}}\binom{n}{j}\right\}\left\{\sum_{k = 0}^{\min\{\frakp, n\}}\binom{n}{k}\right\} \leq \left(\frac{e n}{\frakv}\right)^\frakv\left(\frac{e n}{\frakp}\right)^\frakp = \frac{(en)^{\frakv + \frakp}}{\frakv^\frakv \frakp^\frakp}.
\end{equation}

We now argue that we must have $n \leq 5(\frakv + \frakp)$, which proves the result. Define $\lambda := \frac{\frakv}{\frakv + \frakp}$, which thus implies $1 - \lambda = \frac{\frakp}{\frakv + \frakp}$. Using this notation, observe that we have
\[
\frakv^\frakv \frakp^\frakp = (\frakv + \frakp)^{\frakv + \frakp}\left\{\lambda^\lambda(1 - \lambda)^{1- \lambda}\right\}^{\frakv + \frakp} \geq \left(\frac{\frakv + \frakp}{2}\right)^{\frakv + \frakp},
\]
which follows since $\lambda^\lambda(1 - \lambda)^{1 - \lambda}$ is minimized when $\lambda = 1/2$. Hence, plugging this in to Equation~\eqref{eq:product_sauer}, we see that
\[
2^n \leq \left(\frac{2en}{\frakv + \frakp}\right)^{\frakv + \frakp}.
\]

We now conclude in much the same manner as the proof of Proposition~\ref{prop:vc-to-natarajan}. Setting $r := \frac{n}{\frakv + \frakp}$ and taking $(\frakv + \frakp)$th roots, we have $2^r \leq 2er$. If we assume that $n \geq 5(\frakv + \frakp)$, then we certainly have $r \geq 5$. Noting that
\[
2^5 = 32 > 10 e = 2e(5)
\]
and that the mapping $2^r/r$ is increasing for $r \geq 5$ leads to a contradiction, thus completing the proof.

\end{proof}

We lastly state a well-known result that shows that classes with finite pseudo-dimension (and thus, also finite VC-dimension in the binary setting) admit well-behaved uniform covering numbers.

\begin{lemma}[Theorem 7.16/Problem 7.12 \citet{van2014probability}]
\label{lem:pseudo-dimension-covering}
Let $\calG \subset \R^{\calA \times \calX}$ be a function class with (i) $\sup_{g\in\calG}\|g\|_\infty \leq G$ for some $G > 0$ and (ii) $d_P(\calG) = \frakv$. Then, for any $0 < \epsilon \leq 1$, one has
\[
\sup_{Q}\log N\left(G\epsilon, \calG, L^2(Q)\right) \lesssim \frakv \log\left(\frac{A}{\epsilon}\right)
\]
where the supremum is over all probability measures $Q$ on $\calA \times \calX$ and $A \geq e$ is some constant.

\end{lemma}

\section{Uniform Martingale Bounds}
\label{app:martingale-bounds}
In this appendix, we prove two sets of bounds for collections of martingales indexed by a set of functions with well-behaved combinatorial properties. The first bound we present is a high-probability bound on the supremum of the collection of martingales.
The second bound is a \textit{localized bound} on the expected supremum of the collection of martingales. Before presenting these results, we compare our proof techniques to related work in order to emphasize what is novel in our setting.

We start by discussing our finite sample, high-probability chaining bounds.
Our proof of this bound, roughly, combines two ideas used previously in the literature on off-policy evaluation/policy learning under adaptively collected data. First, following a decomposition leveraged in \citet{zhan2024policy}, we separate each martingale in the collection into a process that is a \textit{martingale conditional on covariates $X_1, \dots, X_n$} and an \textit{empirical process involving the i.i.d. covariates}. The latter term can be controlled with largely standard chaining and symmetrization arguments for i.i.d.\ random variables~\citep{van2014probability}. To control the former, we use a finite-depth chaining argument akin the importance-sampling maximal inequality from \citet{bibaut2021risk}. A key difference in our analysis is how we handle the remainders from finite-depth truncation. Because \citet{bibaut2021risk} leverage bracketing-type assumptions to control the complexity of function classes, the remainder from chaining is naturally easier to control as any function can be bracketed between two elements of the cover whose difference has small $L^2$ norm.
Controlling the remainder under our VC-like assumption, on the other hand, presents additional challenges, as no two elements of our cover may uniformly bound the target function above and below. Our analysis works around this difference by leveraging the aforementioned decomposition due to \citet{zhan2024policy} alongside coverings under appropriately weighted empirical distributions. We further need to argue using a martingale argument that the average inverse propensity weight is not far off from its conditional expectation, thus allowing us to use additional high-probability martingale concentration inequalities to control the remainder. We detail this to a greater extent in our proof in the sequel.

We can control the growth of the first term using a chaining argument with early truncation under an appropriately-weighted empirical metric. The latter term can be controlled using more standard chaining techniques based on symmetrization. We use this high-probability bound in two main ways throughout the text. First, we use it to establish \textit{almost sure} convergence of a natural plug-in estimate for the predictable covariance. Second, we use it to show uniform integrability of an appropriately scaled version of the regret of the learned policy.

 Like before, we prove our bound using a chaining argument with early truncation. However, we are able to exploit the assumed VC-structure of the policy disagreement class in order to avoid decomposing our process into a sum of two simpler processes. This localized bound is leveraged in showing our faster regret rates for weighted empirical risk minimization.

\subsection{High-Probability Chaining Tail Bound}

In this section, we aim to prove a high-probability bound on the growth of a family of martingales indexed by a set of functions with well-behaved uniform covering numbers. In more detail, we will aim to control the growth of the following process
\begin{equation}
\label{eq:martingale}
G_{\ell, n}(g) := \frac{1}{S_{n, -\beta}}\sum_{t = 1}^n w_t\left\{\frac{g(A_t \mid X_t) Y^\ell_t}{p_t(A_t \mid X_t)} - \E_X\left[\sum_{a \in \calA}g(a \mid X_t)\mu_\ell(a, X_t)\right]\right\}
\end{equation}
uniformly over a class of functions $\calG$, where $w_t = t^{-\beta}$, $\beta \in \{0, \alpha\}$, $\alpha$ is as in Assumption~\ref{ass:propensity}, and $S_{n, \gamma}:= \sum_{t = 1}^n t^\gamma$ for any $\gamma \in \R$. Throughout the paper and this appendix, we again view the cardinality $|\calA|$ of the action set $\calA$ as fixed. Thus, we will omit the dependence of any bounds in the sequel on this quantity. All constants that appear hereafter should thus be interpreted as potentially depending on $|\calA|$. In our analyses, we \textit{do} track the dependence of our bounds on the supremum norm/envelope of the class of functions, as this determines how rescaling functions  may impact convergence rates for considered processes. We start by stating our main theorem of this subsection.

\begin{theorem}[High-Probability Chaining Tail Bound]
\label{thm:high-prob-chain}
Let $Z_1, \dots, Z_n$ satisfy Assumptions~\ref{ass:observations} and the first part of Assumption~\ref{ass:estimation}. Further, let $p_1, \dots, p_n$ satisfy Assumption~\ref{ass:propensity}. Let $\calG \subset \Big\{g : \calA \times \calX \rightarrow [-G, G]\Big\}$ satisfying the following uniform covering number bound: for any $\epsilon \in (0, 1]$, we have
\[
\sup_{Q}\log N\big(G\epsilon, \calG, L^2(Q)\big) \lesssim \frakv \log\left(\frac{A}{\epsilon}\right),
\]
where $A \geq e$ is a constant and the supremum is over all distributions on $\calA \times \calX$. For $\ell \in \{1, 2\}$, $n \geq 1$, and $g \in \calG$, let $G_{\ell, n}(g)$ be the process defined in Equation~\eqref{eq:martingale}.
Then, there are absolute constants $C, c > 0$ such that, for any $x > 0$ and $n \geq 2$, we have with probability at least $1 - Ce^{-x} - C\exp\{-cn^{1 - \alpha}\}$,
\[
\sup_{g \in \calG}\left|G_{\ell, n}(g)\right| \lesssim Gn^{\frac{\alpha - 1}{2}}\left(\sqrt{\frakv} + \sqrt{x}\right) + Gn^{\alpha - 1}\left\{\frakv\log(n) + x\right\}\log(n).
\]
\end{theorem}

As noted at the start of this section, our proof follows the strategy of \citet{zhan2024policy}, and involves decomposing the process $\big(G_{\ell, n}(g)\big)_{g \in \calG}$ into two components --- one \textit{martingale component} $\big(M_{\ell, n}(g)\big)_{g \in \calG}$, which consists of dependent data, and an \textit{empirical process} component $\big(E_{\ell, n}(g)\big)_{g \in \calG}$, which consists of i.i.d.\ observations. In more detail, we write
\begin{equation}
\label{eq:martingale-decomp}
G_{\ell, n}(g) = M_{\ell, n}(g) + E_{\ell, n}(g) \quad \text{for all } g \in \calG,
\end{equation}
where we define
\begin{equation}
\label{eq:martingale-term}
M_{\ell, n}(g) := \frac{1}{S_{n, -\beta}}\sum_{t = 1}^n w_t\left\{\frac{g(A_t \mid X_t)Y_t^\ell}{p_t(A_t \mid X_t)} - \sum_{a \in \calA}g(a \mid X_t)\mu_{\ell}(a, X_t)\right\}
\end{equation}
and
\begin{equation}
\label{eq:empirical-term}
E_{\ell, n}(g) := \frac{1}{S_{n, -\beta}}\sum_{t = 1}^n w_t\left\{\sum_{a \in \calA} g(a \mid X_t)\mu_{\ell}(a, X_t) - \E_{X}\left[\sum_{a \in \calA}g(a \mid X) \mu_\ell(a , X)\right]\right\}.
\end{equation}
We will be able to control the growth of $(E_{\ell, n}(g))$ using largely off-the-shelf machinery from the theory of empirical processes --- see \citet{van2014probability}, \citet{kosorok2008introduction}, or \citet{van1996weak}, for instance. Due to the dependence of $M_{\ell, n}(g)$ on adaptive treatment assignments and decreasing propensities, we will need to carefully perform chaining under an appropriately-weighted empirical metric in order to show convergence. We will also need to leverage martingale concentration inequalities in the place of standard concentration inequalities for i.i.d.\ random variables. We will use the following high-probability tail bound for sub-Gamma process, which follows from a general bound in \citet{van2014probability}.
\begin{lemma}[Finite Maxima Tail Bound]
\label{lem:gamma-concentration}
Let $U_1, \dots, U_n$ denote finitely-many random variables satisfying $\log\E\left[e^{\lambda U_i}\right] \leq \psi_{\nu, b}(\lambda)$ for all $\lambda \in (-b^{-1}, b^{-1})$, where $\psi_{\nu, b}(\lambda) := \frac{\nu\lambda^2}{2(1 - b|\lambda|)}$ is the sub-Gamma CGF with scale $b$ and variance $\nu$. Then, we have
\[
\P\left(\max_{1 \leq i \leq n}\left|U_i\right| \geq \sqrt{2\nu\{\log(2n) + x\}} + b\{\log(2n) + x\}\right) \leq 2e^{-x}.
\]
\end{lemma}

Note that the above lemma does not make any assumption on the dependence structure of the random variables $U_1, \dots, U_n$, and thus will be applicable in our setting. We now state the first proposition of this section, which controls the growth of the martingale term $M_{\ell, n}(g)$. We also leverage the following version of Freedman's inequality~\citep{freedman1975tail, howard2020time}.
\begin{lemma}[Freedman's Inequality]
\label{lem:freedman}
Let $(M_t)_{0 \leq t \leq n}$ be a martingale with respect to some filtration $(\calF_t)_{0 \leq t \leq n}$ whose increments $\Delta M_t := M_t - M_{t - 1}$ satisfy $|\Delta M_t| \leq b$ and whose predictable quadratic variation $V_n := \sum_{t = 1}^n\Var(\Delta M_t \mid \calF_{t- 1})$ satisfies $V_n \leq v$ almost surely. Then, for any $x > 0$, we have
\[
\P\left(|M_n| \geq x\right) \leq 2\exp\left\{-\frac{x^2}{2\left(v  + bx/3\right)}\right\}.
\]
\end{lemma}

We now state our first proposition of this section, which gives high-probability control over the supremum of the collection of martingales $\big(M_{\ell, n}(g)\big)_{g \in \calG}$.

\begin{prop}
\label{prop:martingale}
For any $n \geq 1$, $\ell \in \{1, 2\}$, $\beta \in \{0, \alpha\}$, consider the process $\big(M_{\ell, n}(g)\big)_{g \in \calG}$ defined above.
Further, suppose the assumptions of Theorem~\ref{thm:high-prob-chain} hold. Then, there are absolute constants $C, c> 0$ such that, for any $x > 0$ and $n \geq 2$, with probability at least $1 - Ce^{-x} - C\exp\left\{-cn^{1 - \alpha}\right\}$, we have
\[
\sup_{g \in \calG}\left|M_{\ell, n}(g)\right| \lesssim n^{\frac{\alpha - 1}{2}}G\left(\sqrt{\frakv} + \sqrt{x}\right) + n^{\alpha - 1}G\left\{\frakv\log(n) + x\right\}\log(n).
\]
\end{prop}

\begin{proof}
We begin by providing a brief overview of our proof structure. First, we show that, for any $g \in \calG$, the process $\big(M_{\ell,t}(g)\big)_{1 \leq t \leq n}$ is a martingale with respect to an enlarged filtration that contains both past observations $Z_1, \dots, Z_t$ up to time $t$ and \textit{all covariates} $X_1, \dots, X_n$. Then, we aim to derive a high-probability bound for the supremum $\sup_{g \in \calG}\left|M_{\ell, n}(g)\right|$ conditional on $X_1, \dots, X_n$ by using a finite-depth chaining argument with respect to a natural, weighted metric. Lastly, we integrate/marginalize the bound over all realizations of $X_1, \dots, X_n$ in order to obtain the claimed, high-probability bound.

We now present notation that will be used in the proof. We let $(\calF_t)_{0 \leq t \leq n}$ denote the natural filtration associated with $Z_1, \dots, Z_n$, let $\X$ denote the $\sigma$-algebra generated by the i.i.d.\ covariates $X_1, \dots, X_n$, and let $(\calH_t)_{0 \leq t \leq n}$ be defined by $\calH_t := \calF_t \vee \X$.\footnote{Here, for two $\sigma$-algebras $\calA, \calB$, $\calA \vee \calB$ denotes their ``join,'' which is the smallest $\sigma$-algebra containing both $\calA$ and $\calB$.} Further, to ease notation, we define the (unweighted) increments $\xi_{\ell, t}(g)$ of $M_{\ell, n}(g)$ as
\[
\xi_{\ell, t}(g) := \frac{g(A_t \mid X_t)Y_t^\ell}{p_t(A_t \mid X_t)} - \sum_{a \in \calA}g(a \mid X_t)\mu_\ell(a \mid X_t).
\]
Thus, with this notation, we can write $M_{\ell, n}(g) = S_{n, -\beta}^{-1}\sum_{t = 1}^nw_t \xi_{\ell, t}(g)$ for each $g \in \calG$.
Next, define the IPW weighted empirical distribution on $\calA \times \calX$ as follows:
\[
\wb{\Q}_n := \frac{1}{|\calA|S_{n, \alpha - 2\beta}}\sum_{t = 1}^n t^{\alpha - 2\beta} \sum_{a \in \calA}\delta_{(a, X_t)},
\]
where $\delta_{(a, X_t)}$ denotes the point-mass on $(a, X_t) \in \calA \times \calX$. We note that since $\alpha - 2 \beta > -1$, we have $S_{n, \alpha - 2\beta} := \sum_{t = 1}^n t^{\alpha - 2\beta} \asymp n^{\alpha - 2\beta + 1}$. In the sequel, we will consider the $L^2(\Q_n)$-norm associated with $\wb{\Q}_n$ in the sequel, which is given by
\[
d_n(f, g) \equiv \left\|f - g\right\|_n := \left(\frac{1}{|\calA|S_{n, \alpha - 2\beta}} \sum_{t = 1}^n t^{\alpha - 2\beta} \sum_{a \in \calA}\left\{f(a, X_t) - g(a, X_t)\right\}^2\right)^{1/2}.
\]

\paragraph{Showing $M_{\ell, n}(g)$ is a martingale.}
We start by showing that, for any $g \in \calG$, $M_{\ell, n}(g)$ is a martingale with respect to the extended filtration $(\calH_t)_{0 \leq t \leq n}$. First, by our assumptions, $Y_t$ and $g(A_t \mid X_t)$ are bounded, and $p_t(A_t \mid X_t) \gtrsim t^{-\alpha}$. Hence, the increments $|\xi_{\ell, t}(g)|$ are bounded almost surely and so $(M_{\ell, t}(g))_{1 \leq t \leq n}$ is integrable. Second, $\xi_{\ell, t}(g)$ is clearly $\calH_t$-measurable, and so $(M_{\ell, t}(g))_{1 \leq t \leq n}$ is adapted to $\calH_t$. What remains is to check the martingale property, and to verify this it suffices to show that $\E\left(\xi_{\ell, t}(g) \mid \calH_{t - 1}\right) = 0$. We have
\begin{align*}
&\E\left(\xi_{\ell, t}(g) \mid \calH_{t - 1}\right) = \E\left(\frac{g(A_t \mid X_t)Y_t^\ell}{p_t(A_t \mid X_t)} \mid \calH_{t - 1}\right) - \sum_{a \in \calA}g(a \mid X_t)\mu_\ell(a, X_t) \\
&\qquad = \E\left(\frac{g(A_t \mid X_t)Y_t^\ell}{p_t(A_t \mid X_t)} \mid \calF_{t - 1}, X_t\right) - \sum_{a \in \calA}g(a \mid X_t)\mu_\ell(a, X_t) \\
&\qquad = \E\left(\sum_{a \in \calA}\frac{\1\{A_t = a\}g(a \mid X_t)}{p_t(a \mid X_t)}\mu_\ell(a, X_t) \mid X_t, \calF_{t - 1}\right) - \sum_{a \in \calA}g(a \mid X_t)\mu_\ell(a, X_t) \\
&\qquad = \sum_{a \in \calA}g(a \mid X_t)\mu_\ell(a, X_t) - \sum_{a \in \calA}g(a \mid X_t)\mu_\ell(a, X_t) \\
&\qquad =0,
\end{align*}
Thus, the martingale property is satisfied, and so $(M_{\ell, t}(g))_{1 \leq t \leq n}$ is a martingale with respect to $(\calH_s)$ for any fixed $g \in \calG$. Note in the above the first equality follows from the $\calH_s$-measurability of the sum, the second follows since $X_{t + 1}, \dots,  X_n$ are independent of $Z_1, \dots, Z_t$, and the second to last equality follows from inverse propensity weighting alongside Assumption~\ref{ass:observations}.

\paragraph{Bounded Increments and Quadratic Variation of $M_{\ell, n}(g)$.}

Our goal in the sequel is to perform a chaining argument and to use Freedman's inequality to bound the growth of the supremum of the differences in each segment of the chain. To apply Freedman's inequality, we must first bound the  increments and predictable quadratic variation of $M_{\ell, n}(g)$. First, note that
\[
\left|\xi_{\ell, t}(g)\right| \leq \left|\frac{g(A_t \mid X_t)Y_t^\ell}{p_t(A_t \mid X_t)}\right| + \sum_{a \in \calA}\left|g(a \mid X_t)\right| \lesssim Gt^{\alpha}|\calA|.
\]
Thus, the increments of $M_{\ell, n}(g)$ satisfy
\[
\left|\Delta M_{\ell, t}\right| = \left|\frac{w_t \xi_{\ell, t}(g)}{S_{n, -\beta}}\right| \lesssim \frac{G|\calA|t^{\alpha - \beta}}{n^{-\beta + 1}} \leq G|\calA| n^{\alpha - 1},
\]
which follows from the observation that $S_{n, -\beta} \asymp n^{-\beta + 1}$ and hence $t^{\alpha - \beta}/S_{n, -\beta} \leq n^{\alpha - \beta} / S_{n, -\beta} \asymp n^{-1}$. Next, we see that we have the upper bound
\begin{align*}
\Var\left(w_t\xi_{\ell, t}(g) \mid \calH_{t - 1}\right) &\lesssim t^{-2\beta}\E\left(\xi_{\ell, t}(g)^2 \mid \calH_{t - 1}\right) \\
&\leq t^{-2\beta}\E\left(\frac{g(A_t \mid X_t)^2Y_t^{2\ell}}{p_t(A_t \mid X_t)^2} \mid \calH_{t - 1}\right) \\
&= t^{-2\beta}\sum_{a \in \calA}\frac{g(a \mid X_t)^2 \mu_{2\ell}(a, X_t)}{p_t(a \mid X_t)} \\
&\lesssim t^{\alpha - 2\beta}\sum_{a \in \calA}g(a \mid X_t)^2.
\end{align*}
Thus, the predictable quadratic variation satisfies
\begin{align*}
\frac{1}{S_{n, -\beta}^2}\sum_{t = 1}^n \Var\left(w_t \xi_{\ell, t}(g) \mid \calH_{t - 1}\right) &= \frac{1}{S_{n, -\beta}^2}\sum_{t = 1}^n \Var\left(w_t\xi_{\ell, t}(g) \mid \calH_{t- 1}\right) \\
&\lesssim n^{2\beta - 2}\sum_{t = 1}^n t^{\alpha - 2\beta} \sum_{a \in \calA}g(a \mid X_t)^2 \\
&\leq n^{2 \beta - 2}|\calA|S_{n, \alpha - 2\beta} \|g\|_n^2 \\
&\asymp n^{\alpha - 1}|\calA|\|g\|_n^2.
\end{align*}
which follows from the definition of the empirical norm $\|\cdot\|_n$ and from the earlier noted fact that $S_{n, \alpha - 2\beta} \asymp n^{\alpha - 2\beta + 1}$. Thus, we have shown that $M_{\ell, n}(g)$ is a martingale for each $g \in \calG$, and also characterized both the boundedness properties and the quadratic variation.

\paragraph{Setting up Chaining.}

Now, for $j \geq 0$, let $\delta_j := 2^{-j}G$, where again $G \geq \|g\|_\infty$ for each $g \in \calG$. Let $J_n := \min\left\{j \geq 1 : 2^{-j} \leq n^{-L}\right\}$, where $L$ is a fixed constant to be chosen later. For each such $j$, let $K_j$ denote a minimal $\delta_j$-covering of $\calG$ under $d_n \equiv \|\cdot\|_n$. By the assumption on the growth rate of uniform covering numbers,  we know that
\[
\log|K_j| \lesssim \frakv \log\left(\frac{A}{2^{-j}}\right) = \frakv\log(A 2^j)
\]
for some constant $A \geq e$. For any $g \in \calG$, we let $g_j \in K_j$ denote a projection of $g$ onto $K_j$, i.e.\ a point satisfying $d_n(g, g_j) \leq \delta_j$. Observe that a standard chaining decomposition up to level $J_n$ yields
\[
g = g_0 + \sum_{j = 1}^{J_n}\big\{g_j - g_{j - 1}\big\} + \big\{g - g_{J_n}\}.
\]
Noting that $d(g_j, g_{j - 1}) \leq d(g_j, g) + d(g_{j - 1}, g) \leq 3\delta_j$ and letting
\[
\calR_j := \Big\{g - g' : g \in K_j, g' \in K_{j - 1}, d_n(g, g') \leq 3 \delta_j \Big\}
\]
 yields via the triangle inequality
\[
\sup_{g \in \calG}\left|M_{\ell, n}(g)\right| \leq \underbrace{\sup_{g \in K_0}\left|M_{\ell, n}(g)\right|}_{=: T_1} + \underbrace{\sum_{j = 1}^{J_n} \sup_{h \in \calR_j}\left|M_{\ell,n}(h)\right|}_{=: T_2} + \underbrace{\sup_{\|h\|_n \leq \delta_{J_n}}\left|M_{\ell,n}(h)\right|}_{=: T_3}.
\]
We now show high-probability bounds on each of $T_1$, $T_2$, and $T_3$. An important instantiation of Lemma~\ref{lem:gamma-concentration} we will use in the sequel is the following. Let $\calG_0$ be a  finite collection of functions that may depend on $X_1, \dots, X_n$ satisfying $|\calG_0| \leq m$ for some fixed $m$, $\|g\|_n \leq \delta$, and $\|g\|_\infty \leq b$ for some $\delta, b> 0$ all $g \in \calG_0$. Then, using our earlier-derived bounds on the increments and predictable quadratic variation of $(M_{\ell, t}(g))_{1 \leq t \leq n}$, we see that the following holds: for any $x > 0$, with probability at least $1 - 2e^{-x}$
\begin{equation}
\label{eq:high-prob-finite}
\max_{g \in \calG_0}\left|M_{\ell,n}(g)\right| \lesssim n^{\frac{\alpha - 1}{2}}\delta\sqrt{\log(m) + x} + n^{\alpha - 1} b|\calA|\left\{\log(m) + x\right\}.
\end{equation}
We note that, since $\|\cdot\|_n$ is a random metric dependent on the covariates $X_1, \dots, X_n$, Lemma~\ref{lem:gamma-concentration} implies the above holds conditionally on $\X$. However, since the right hand side does not depend on $X_1, \dots, X_n$, we can simply integrate over these variables to get the stated marginal bound.
\paragraph{Bounding the Coarse Term.}
First, we focus on bounding $T_1$. By construction, we know $|K_0| = 1$, so bounding $T_1$ just requires an application Equation~\eqref{eq:high-prob-finite} over a set with a single element. Thus, we have with probability  at least $1 - 2e^{-x}$,
\[
T_1 \lesssim n^{\frac{\alpha - 1}{2}}G\sqrt{x} + n^{\alpha - 1} G|\calA| x.
\]

\paragraph{Bounding the Chaining Term.}

We apply Equation~\eqref{eq:high-prob-finite} over each link of the chain and then union bound. For each $j \in 1, 2, \dots, J_n$, set $x_j := x + 2\log(j + 1)$. Then, for each such $j$, we have with probability at least $1 - 2e^{-x_j} = 1 - \frac{2}{(j + 1)^2}e^{-x}$
\begin{align*}
\sup_{h \in \calR_j}\left|M_{\ell, n}(h)\right| &\lesssim \delta_jn^{\frac{\alpha - 1}{2}}\sqrt{\log\left(|K_j||K_{j - 1}|\right) + x + 2\log(j + 1)} \label{tag:bound}\tag{A} \\
&\qquad + n^{\alpha - 1}G|\calA|\left\{\log(|K_j||K_{j - 1}|)  + x + 2\log(j  + 1)\right\} \\
&\lesssim 2^{-j}G n^{\frac{\alpha - 1}{2}}\sqrt{\frakv\log(A 2^j) + x + 2\log(j + 1)} \\
&\qquad + n^{\alpha - 1}G|\calA|\left\{\frakv\log(A2^j) + x + 2\log(j + 1)\right\} \\
&\lesssim 2^{-j}G n^{\frac{\alpha - 1}{2}}\sqrt{\frakv\log(A 2^j) + x} + n^{\alpha - 1} G |\calA|\left(\frakv\log(A 2^j) + x\right),
\end{align*}
where the second line follows from the assumption of logarithmic uniform covering numbers. The final inequality follows because $2 \log(j + 1) \leq  \log(A 2^j)$, since $A \geq e$  and $j \geq 1$.

Summing over indices $j = 1, 2, \dots, J_n$, we see that we have
\begin{align*}
\sum_{j = 1}^{J_n}\sup_{h \in \calR_j}\left|M_{\ell, n}(h)\right| &\lesssim G n^{\frac{\alpha - 1}{2}}\sum_{j = 1}^{J_n}2^{-j}\sqrt{\frakv \log(A2^j) + x} \\
&\qquad + G|\calA|n^{\alpha - 1}\sum_{j = 1}^{J_n}\left\{\frakv \log(A2^j) + x\right\}.
\end{align*}
The first sum can be bounded as follows:
\begin{align*}
\sum_{j = 1}^{J_n}2^{-j}\sqrt{\frakv \log(A 2^j) + x} &\leq \sum_{j = 1}^{J_n}2^{-j}\sqrt{\frakv \log(A 2^j)} + \sqrt{x}\sum_{j = 1}^{J_n}2^{-j} \\
&\lesssim \int_0^1 \sqrt{\frakv\log\left(\frac{A}{\epsilon}\right)}d\epsilon + \sqrt{x} \\
&\lesssim \sqrt{\frakv\log(A)} + \sqrt{x},
\end{align*}
where the last inequality follows because $\int_0^1\sqrt{\frakv\log(A/\epsilon)}d\epsilon \leq \sqrt{2\frakv\log(A)}$.

Now note the summation over the second term can be bounded above as follows:
\begin{align*}
\sum_{j = 1}^{J_n}\left\{\frakv \log(A 2^j) + x\right\} &\lesssim J_n \frakv \log\left(A 2^{J_n}\right) + J_n x \\
&\lesssim \left\{L\log(n) \frakv\log(A) + x\right\} L \log(n),
\end{align*}
which follows as $J_n$ was chosen to be the smallest $j \geq 1$ such that $\delta_j \leq Gn^{-L}$. Since by construction $L$ is a constant only depending on $\alpha$ (which does not scale with $n$), we can omit dependence on multiplicative factors of $L$ in the sequel.

Via a union bound, we see that the probability that the Inequality~\ref{tag:bound} fails for any $1 \leq j \leq {J_n}$ is upper bounded by
\[
\sum_{j = 1}^{J_n} 2e^{-x_j} \leq \sum_{j = 1}^\infty 2e^{-x_j} = \sum_{j = 1}^\infty \frac{2}{(j + 1)^2}e^{-x} \lesssim C e^{-x}
\]
for some absolute constant $C > 0$. Thus, with probability at least $1 - Ce^{-x}$
\[
\sum_{j = 1}^{J_n}\sup_{h \in \calR_j}\left|M_{\ell, n}(h)\right| \lesssim G n^{\frac{\alpha - 1}{2}} \left(\sqrt{\frakv}  + \sqrt{x}\right) + G |\calA|n^{\alpha - 1}\left\{\frakv\log(n) + x\right\}\log(n),
\]
which follows from the above analysis and omitting multiplicative constants.

\paragraph{Bounding the Remainder.}

Observe that we can bound $T_3$ as follows:
\begin{align*}
T_3 &\leq \underbrace{\frac{1}{S_{n, -\beta}}\sup_{\|h\|_n \leq \delta_{J_n}} \sum_{t = 1}^nw_t\frac{|h(A_t \mid X_t)|Y_t^\ell}{p_t(A_t \mid X_t)}}_{=: T_3^{(1)}} \\
&\qquad + \underbrace{\frac{1}{S_{n, -\beta}}\sup_{\|h\|_n \leq \delta_{J_n}}\sum_{t = 1}^nw_t\sum_{a \in \calA}|h(a, X_t)|\mu_\ell(a, X_t)}_{=: T_3^{(2)}}.
\end{align*}
Using the assumption that $|Y_t| \leq  1$ almost surely, the first term can be bounded as
\begin{align*}
T_3^{(1)} &\leq \frac{1}{S_{n, -\beta}}\sup_{\|h\|_n \leq \delta_{J_n}} \sum_{t = 1}^n w_t \frac{\left|h(A_t \mid X_t)\right|}{p_t(A_t \mid X_t)} \\
&= \frac{1}{S_{n, -\beta}}\sup_{\|h\|_n \leq \delta_{J_n}}\sum_{t = 1}^n \frac{w_t|h(A_t \mid X_t)|}{\sqrt{p_t(A_t \mid X_t)}}\cdot \frac{1}{\sqrt{p_t(A_t \mid X_t)}} \\
&\leq \frac{1}{S_{n, -\beta}}\left(\sup_{\|h\|_n \leq \delta_{J_n}} \sum_{t = 1}^n\frac{w_t^2 h(A_t \mid X_t)^2}{p_t(A_t \mid X_t)}\right)^{1/2}\left(\sum_{t = 1}^n \frac{1}{p_t(A_t \mid X_t)}\right)^{1/2},
\end{align*}
where the final inequality follows from Cauchy-Schwarz.
First, note that, for any $h$ satisfying the assumed normed bound, we have
\begin{align*}
\sum_{t  = 1}^n \frac{w_t^2 h(A_t \mid X_t)^2}{p_t(A_t \mid X_t)} &\lesssim \sum_{t = 1}^n t^{\alpha - 2\beta} h(A_t \mid X_t)^2 \\
&\leq \sum_{t = 1}^nt^{\alpha - 2\beta}\sum_{a \in \calA}h(a \mid X_t)^2 \\
&\leq |\calA|S_{n, \alpha - 2\beta}\|h\|_n^2 \\
&\lesssim G^2|\calA|n^{\alpha - 2\beta + 1 -2L},
\end{align*}
where the final inequality follows from the assumed norm bound on the function $h$.
Consequently, we have
\[
\left(\sup_{\|h\|_n \leq \delta_{J_n}} \sum_{t = 1}^n\frac{w_t^2h(A_t \mid X_t)^2}{p_t(A_t \mid X_t)}\right)^{1/2} \lesssim \sqrt{|\calA|}n^{\frac{\alpha + 1}{2} - \beta - L}.
\]
Next, we claim that the process $(S_t)_{1 \leq  t \leq n}$ defined as
\[
S_t := \sum_{s = 1}^t \left\{\frac{1}{p_s(A_s \mid X_s)} - |\calA|\right\}
\]
is a martingale with respect to $(\calH_t)_{1 \leq t \leq n}$. Integrability of the process is straightforward, as is the fact that $S_t$ is $\calH_t$-measurable. Further, noting that
\[
\E\left(\frac{1}{p_t(A_t \mid X_t)} \mid \calH_{t - 1}\right) = \sum_{a \in \calA}p_t(a \mid X_t) \cdot \frac{1}{p_t(a \mid X_t)} = |\calA|
\]
shows that the martingale property holds. Thus, noting that
\begin{align*}
\sum_{t = 1}^n \Var\left(\frac{1}{p_t(A_t \mid X_t)} \mid \calH_{t - 1}\right) &\leq \sum_{t = 1}^n \E\left(\frac{1}{p_t(A_t \mid X_t)^2} \mid \calH_{t -1}\right) \\
&= \sum_{t = 1}^n\sum_{a \in \calA}\frac{1}{p_t(a \mid X_t)} \\
&\lesssim  \sum_{t= 1}^n |\calA|t^{\alpha} \lesssim n^{\alpha + 1}|\calA|
\end{align*}
and that
\[
\left|\frac{1}{p_t(A_t \mid X_t)}\right| \lesssim t^\alpha \lesssim n^{\alpha},
\]
we can apply Freedman's inequality (Lemma~\ref{lem:freedman}) for martingales with bounded predictable quadratic variation and bounded increments to control the magnitude of $S_n$. In particular, our above analysis yields the existence of a constant $c > 0$ such that
\[
\P\left(S_n \geq n\right) \leq \exp\left\{-\frac{cn^2}{ n^{1 + \alpha}}\right\} = \exp\{-c n^{1 - \alpha}\},
\]
which follows because, using the notation of Lemma~\ref{lem:freedman}, we have $V \lesssim n^{\alpha + 1}$, $x = n$, and $b = n^\alpha$.
Thus, rearranging, we have with probability at least $1 - e^{-cn^{1 - \alpha}}$ that
\[
\sum_{t  = 1}^n \frac{1}{p_t(A_t \mid X_t)} \lesssim (|\calA| + 1)n \lesssim |\calA| n,
\]
 Consequently,
\begin{align*}
T_3^{(1)} &\lesssim \frac{1}{S_{n, -\beta}}\left(\sup_{\|h\|_n \leq \delta_{J_n}} \sum_{t = 1}^n\frac{w_t^2 h(A_t \mid X_t)^2}{p_t(A_t \mid X_t)}\right)^{1/2}\left(\sum_{t = 1}^n \frac{1}{p_t(A_t \mid X_t)}\right)^{1/2} \\
&\leq \frac{1}{S_{n, -\beta}}G|\calA|n^{\frac{\alpha + 1}{2}  - \beta - L}n^{1/2} \\
&\lesssim G|\calA|n^{\frac{\alpha}{2} - L}.
\end{align*}
We can similarly bound $T_3^{(2)}$. In particular, again using the fact that $|Y_t| \leq 1$ almost surely (which in turn implies $|\mu_\ell(a, X_t)| \leq 1$ almost surely for any $a \in \calA$ and $\ell \geq 1$), we have
\begin{align*}
T_3^{(2)} &\leq \frac{1}{S_{n, -\beta}}\sup_{\|h\|_n \leq \delta_{J_n}}\sum_{t = 1}^n\sum_{a \in \calA}w_t|h(a \mid X_t)| \\
&= \frac{1}{S_{n, -\beta}}\sup_{\|h\|_n \leq \delta_J}\sum_{t = 1}^n \sum_{a \in \calA}\left(w_t t^{\alpha/2}\left|h(a \mid X_t)\right|\right)t^{-\alpha/2} \\
&\leq \frac{1}{S_{n, -\beta}}\left(\sup_{\|h\|_n \leq \delta_J}\sum_{t = 1}^n \sum_{a \in \calA}w_t^2t^\alpha h(a \mid X_t)^2\right)^{1/2}\left(\sum_{t = 1}^n t^{-\alpha}\right)^{1/2} \\
&\lesssim \frac{1}{S_{n, -\beta}}G\sqrt{|\calA|S_{n, \alpha - 2\beta}}\delta_J n^{\frac{1 - \alpha}{2}} \\
&\lesssim \sqrt{|\calA|}Gn^{\beta - 1}n^{\frac{1 + \alpha}{2} - \beta}n^{-L}n^{\frac{1 - \alpha}{2}} \\
&\leq \sqrt{|\calA|}Gn^{-L},
\end{align*}
where the first line follows from the triangle inequality, the second from multiplying and dividing by $t^{\alpha/2}$, and the third from Cauchy-Schwarz. The fourth inequality follows from the definition of the $\|\cdot\|_n$ norm and the assumed norm bound on $h$. Thus, for some absolute constant $c > 0$, we have that, with probability at least $1 - \exp\{-cn^{1 - \alpha}\}$ for some absolute $c > 0$, we have
\[
T_3 \lesssim G\sqrt{|\calA|}n^{-L} + G|\calA|n^{\frac{\alpha}{2} - L}.
\]
If we take $L := \alpha/2 + 2$, we then get that
\[
T_3 \lesssim G|\calA|n^{-2}.
\]
Putting these bounds together completes the proof.

\end{proof}

\begin{prop}[McDiarmid's Inequality, Theorem 3.11 of \citet{van2014probability}]
\label{prop:mcdiarmid}
Let $X_1, \dots, X_n \in \calX$ be i.i.d., and let $f : \calX^n \rightarrow \R$ be a function. For any $x_1, \dots, x_n \in \calX$, define $D_i f(x)$ by
\[
D_i f(x) := \sup_{z \in \calX}f(x_1, \dots, x_{i - 1}, z, x_{i + 1}, \dots, x_n) - \inf_{z \in \calX}f(x_1, \dots, x_{i - 1}, z, x_{i + 1}, \dots, x_n).
\]
Then, the random variable $f(X_1, \dots, X_n)$ is sub-Gaussian with variance proxy $\sigma^2 = \frac{1}{4}\sum_{i = 1}^n \left\|D_i f\right\|_{\infty}^2$. Hence, we have
\[
\P\Big(f(X_1, \dots, X_n) - \E f(X_1, \dots, X_n) > x\Big) \leq \exp\left\{-\frac{2x^2}{\sum_{i = 1}^n \left\|D_if\right\|_\infty^2}\right\}.
\]

\end{prop}

\begin{definition}[Definition 5.20 of \citet{van2014probability}]
\label{def:sub-gaussian-process}
Suppose $(\calG, d)$ is a separable metric space. We say a stochastic process $(M(g))_{g \in \calG}$ is \textit{sub-Gaussian} with variance proxy $\sigma^2$ if, for any $\lambda \in \R$ and $g, h \in \calG$, we have
\[
\log \E\left[\exp\left\{\lambda (M(g) - M(h))\right\}\right] \leq \frac{\lambda^2\sigma^2d(g, h)^2}{2}
\]
and we also have $\E[M(g)] = 0$ for all $g \in \calG$.
\end{definition}

\begin{lemma}
\label{lem:sub-gaussian-entropy}
Let $(\calH, d)$ be a semi-normed space of functions, $\calG \subset \calH$, and let $M : \calH \rightarrow \R$ be odd, i.e., $M(-g) = -M(g)$. Suppose $(M(g))_{g \in \calH}$ is a sub-Gaussian process with variance proxy $\sigma^2$. Then, we have
\[
\E\left[\sup_{g \in \calG}\left|M(g)\right|\right] \lesssim \sigma \int_0^{\diam(\calG, d)} \sqrt{\log 2N\big(\epsilon, \calG, d\big)}d\epsilon,
\]
where $\diam(\calG, d) := \sup_{g \in \calG}d(0, g).$
\end{lemma}

The above follows directly from Corollary 5.25 of \citet{van2014probability}, which is only stated for expected supremum $\E\left[\sup_{g \in \calG}M(g)\right]$, not the supremum of the absolute values of $M(g)$. Since we assume symmetry of $M$ and $\calG$, we trivially have
\[
\E\left[\sup_{g\in\calG}|M(g)|\right] = \E\left[\sup_{g \in \calG^\pm}M(g)\right] \lesssim \sigma\int_0^{\diam(\calG^\pm, d)}\sqrt{\log N(\epsilon, \calG^\pm, d)}d\epsilon,
\]
where $\calG^\pm := \calG \cup - \calG$. The bound follows from noting that $\diam(\calG^{\pm}, d) = \diam(\calG, d)$ and $N(\epsilon, \calG^\pm, d) \leq 2 N(\epsilon, \calG, d)$.

\begin{prop}[Controlling $E_{\ell, n}$]
\label{prop:emp-process}
For any $n \geq 2$, $\ell \in \{1, 2\}$, $\beta \in \{0, \alpha\}$, consider the process $\big(E_{\ell, n}(g)\big)_{g \in \calG}$ defined above.
Further, assume the conditions of Theorem~\ref{thm:high-prob-chain} hold, and let $x > 0$ be arbitrary. Then, with probability $\geq 1 - 2 e^{-x}$, we have
\[
\sup_{g \in \calG}\left|E_{\ell, n}(g)\right| \lesssim G\kappa_{n, \beta}\big(\sqrt{\frakv} + \sqrt{x}\big)
\]
where we define
\[
\kappa_{n, \beta} := \begin{cases}
n^{-1/2} &\text{for } \beta \in [0, 1/2) \\
\sqrt{\log(n)/n} &\text{for } \beta = 1/2, \text{ and } \\
n^{-(1 - \beta)} &\text{for } \beta \in (1/2, 1).
\end{cases}
\]

\end{prop}

\begin{proof}

First, for $\ell \geq 1$, define the function class $\calU_\ell$ by
\[
\calU_{\ell} := \Big\{u_g(x) := \sum_{a \in \calA}g(a \mid x)\mu_\ell(a, x) : g \in \calG\Big\}
\]
For any probability distribution $Q$ on $\calX$, let $\wb{Q} := Q \otimes \rho$, where $\rho$ denotes the uniform distribution on $\calA$. Observe that we have the inequality
\begin{align*}
&\left\|u_g(X) - u_{g'}( X)\right\|_{L^2(Q)} = \left(\int_{\calX}\left\{\sum_{a \in \calA} \left(g(a \mid x) - g'(a \mid x)\right)\mu_\ell(a, x)\right\}^2Q(dx)\right)^{1/2} \\
&\qquad = |\calA|\left(\int_{\calX}\left\{\frac{1}{|\calA|}\sum_{a \in \calA} \left(g(a \mid x) - g'(a \mid x)\right)\mu_\ell(a, x)\right\}^2Q(dx)\right)^{1/2} \\
&\qquad \leq |\calA|\left(\int_{\calX}\frac{1}{|\calA|}\sum_{a \in \calA} \big(g(a \mid x) - g'(a \mid x)\big)^2\mu_\ell(a, x)^2 Q(dx)\right)^{1/2} \\
&\qquad \leq |\calA|\left(\int_{\calX}\frac{1}{|\calA|}\sum_{a \in \calA} \big(g(a \mid x) - g'(a \mid x)\big)^2 Q(dx)\right)^{1/2} \\
&\qquad \leq |\calA|\left(\int_{\calX \times \calA}\big(g(a \mid x) - g'(a \mid x)\big)^2 \wb{Q}(dx \times da)\right)^{1/2},
\end{align*}
where the first inequality follows from Jensen's inequality.
Thus, the covering numbers of $\calU_\ell$ obey the following relation: for any $0 < \epsilon \leq 1$, letting $D_{\calU} := |\calA|G$ denote an upper bound on the $\ell_\infty$ norm of elements of $\calU_{\ell}$, we have
\begin{align*}
\log N\big(D_{\calU}\epsilon, \calU_\ell, L^2(Q)\big) &\leq \log N\left(\frac{D_{\calU}\epsilon}{|\calA|}, \calG, L^2(\wb{Q})\right) \\
&= \log N\left( G \epsilon, \calG, L^2(\wb{Q})\right) \\
&\lesssim \frakv \log\left(\frac{A}{\epsilon}\right).
\end{align*}

Now, define the symmetrized process $R_{\ell, n}(u)$ indexed by $\calU_\ell$ as follows:
\[
R_{\ell, n}(u) := \frac{1}{S_{n, -\beta}}\sum_{t = 1}^n w_t\epsilon_t u(X_t),
\]
where $\epsilon_1, \dots, \epsilon_n$ are i.i.d.\ Rademacher random variables. Further, we define the empirical $L^2$ metric $d_n : \calU_\ell \times \calU_\ell \rightarrow \R_{\geq 0}$ by
\[
d_n(u_g, u_{g'}) \equiv \|u_g - u_{g'}\|_n := \left(\frac{1}{S_{n, -2\beta}}\sum_{t = 1}^n t^{-2\beta}\left\{u_g(X_t) - u_{g'}(X_t)\right\}^2\right)^{1/2},
\]
where we recall that $S_{n, \gamma} := \sum_{t = 1}^n t^\gamma$. As in the proof of the previous proposition, this is just an $L^2$ norm on $\calU_\ell$, but this time under the weighted empirical measure $\Q_n := \frac{1}{S_{n, -2\beta}}\sum_{t = 1}^n t^{-2\beta}\delta_{X_t}$. Note that $S_{n, \gamma}$ scales as follows for all $\gamma \in \R$.
\begin{equation}
\label{eq:S_gamma_scale}
S_{n, \gamma} \asymp \begin{cases}
    n^{\gamma + 1} &\text{For all } \gamma > -1, \\
    \log(n) &\text{For } \gamma = -1, \text{ and} \\
    \zeta(-\gamma) &\text{For }\gamma < -1,
\end{cases}
\end{equation}
where $\zeta(s) := \sum_{n = 1}^\infty n^{-s}$ denotes the Riemann zeta function for $s > 1$. Using the scaling outlined in Equation~\eqref{eq:S_gamma_scale}, we see that $(R_{\ell, n}(u) : u \in \calU_\ell)$ is $\sigma^2_{n, \beta}$-subGaussian in the sense of Definition~\ref{def:sub-gaussian-process} under the the conditional law $\P(\cdot \mid \X)$, where we recall that $\X$ is the $\sigma$-algebra generated by $X_1, \dots, X_n$, where
\begin{equation}
\label{eq:subgaussian-scale}
\sigma^2_{n, \beta} = \frac{S_{n, -2\beta}}{\left(S_{n, -\beta}\right)^2} \asymp \begin{cases}
n^{-1} &\text{For } \beta \in [0, 1/2), \\
\log(n)/n &\text{For } \beta = 1/2, \text{ and } \\
n^{-2(1 - \beta)} &\text{For } \beta\in (1/2, 1).
\end{cases}
\end{equation}

Indeed, $\E\left(R_{\ell, n}(u) \mid \X\right) = 0$ for any $u \in \calU_\ell$, and for any $\lambda \in \R$ and $u_g, u_{g'} \in \calU_\ell$, we have the bound
\begin{align*}
&\E\left(\exp\left\{\lambda \big(R_{\ell, n}(u_g) - R_{\ell, n}(u_{g'})\big)\right\} \mid \X\right) = \prod_{t = 1}^n \E\left(\exp\left\{\frac{\lambda w_t\epsilon_t}{S_{n, -\beta}}\left(u_g(X_t) - u_{g'}(X_t)\right)\right\} \mid \X\right) \\
&\qquad\leq \prod_{t = 1}^n \E\left(\exp\left\{\frac{\lambda^2t^{-2\beta}}{2(S_{n, -\beta})^2}\left(u_g(X_t) - u_{g'}(X_t)\right)^2\right\} \mid \X\right) \\
&\qquad= \E\left(\exp\left\{\frac{\lambda^2}{2}\frac{S_{n, -2\beta}}{\left(S_{n, -\beta}\right)^2}d_n(u_g, u_{g'})^2\right\} \mid \X\right),
\end{align*}
where the first inequality follows because a random variable that is bounded in absolute value by $B$ is $B^2$ sub-Gaussian. Thus, Lemma~\ref{lem:sub-gaussian-entropy} alongside the fact that $\diam(\calU_\ell, d_n) \leq D_\calU$ yields that
\begin{align*}
\E\left(\sup_{u_g \in \calU_\ell}\left|R_{\ell, n}(u_g)\right| \mid \X \right) &\lesssim  \frac{\sqrt{S_{n, -2\beta}}}{S_{n, -\beta}}\int_0^{D_{\calU}}\sqrt{\log N\big(\epsilon, \calU_\ell, d_n\big)}d\epsilon\\
&= D_\calU\frac{\sqrt{S_{n, -2\beta}}}{S_{n, -\beta}}\int_0^1 \sqrt{\log N\left(D_\calU \epsilon, \calU_\ell, d_n\right)}d\epsilon\\
&\leq |\calA|G\frac{\sqrt{S_{n, -2\beta}}}{S_{n, -\beta}}\int_0^1\sqrt{\log N\left(\frac{\epsilon D_{\calU}}{|\calA|}, \calG, L^2(\wb{\Q}_n)\right)}d\epsilon \\
&\lesssim |\calA|G\frac{\sqrt{S_{n, -2\beta}}}{S_{n, -\beta}}\int_0^1 \sqrt{\frakv \log\left(\frac{A}{\epsilon}\right)}d\epsilon \\
&\asymp \sqrt{\frakv}|\calA|G \kappa_{n, \beta},
\end{align*}
where $\kappa_{n, \beta}$ is as defined as in the statement of the proposition and we have let $\wb{\Q}_n := \Q_n \otimes \rho$, where again $\rho$ is the uniform measure on $\calA$. Note that, in the above, we absorbed the factor of 2 attached to the covering numbers in Lemma~\ref{lem:sub-gaussian-entropy} into constants.
Since this bound does not depend on $X_1, \dots, X_n$, we can marginalize over these covariates to obtain
\[
\E\left[\sup_{u_g \in \calU_\ell}\left|R_{\ell, n}(u_g)\right|\right] \lesssim \sqrt{\frakv}|\calA|G\kappa_{n, \beta}.
\]

We can now use the bound on the expected supremum of the symmetrized process alongside McDiarmid's inequality to obtain a high probability bound on $\sup_{g \in \calG}\left|E_{\ell, n}(g)\right|$.  First, standard results on symmetrization yield that
\begin{align*}
\E\left[\sup_{g \in \calG}\left|E_{\ell,n}(g)\right|\right] &= \E\left[\sup_{u_g \in \calU_\ell}\left|\frac{1}{S_{n, -\beta}}\sum_{t  =1}^nw_t\left\{ u_g(X_t) - \E u_{g}(X_t)\right\}\right|\right] \\
&\leq 2 \E\left[\sup_{u_g \in \calU_\ell} \left|\frac{1}{S_{n, -\beta}}\sum_{t = 1}^n w_t\epsilon_t u_g(X_t)\right|\right] \\
&= 2 \E\left[\sup_{u_g \in \calU_\ell}|R_{\ell, n}(u_g)|\right]
\end{align*}
Next, define the function $f : \calX^n \rightarrow \R$ defined as
\[
f(x_1, \dots, x_n) := \sup_{u_g \in \calU_\ell}\left|\frac{1}{S_{n, -\beta}}\sum_{t = 1}^n w_t\left\{u_g(x_t) - \E u_g(X)\right\}\right|.
\]
We note that this process satisfies a bounded difference inequality. Letting $x_i, \wt{x}_i \in \calX$, we have
\begin{align*}
\left|f(x_{-i}, x_i) - f(x_{-i}, \wt{x}_i)\right| &\leq \sup_{g \in \calG}\left|\frac{i^{-\beta}}{S_{n, -\beta}}\left\{u_g(x_i) - u_g(\wt{x}_i)\right\}\right| \\
& \leq \frac{|\calA|i^{-\beta}}{S_{n, -\beta}}\sup_{\substack{g \in \calG \\ a \in \calA}}\left\{\left|g(a \mid x_i)\right| + \left|g(a \mid \wt{x}_i)\right|\right\} \\
& \leq \frac{2|\calA|G i^{-\beta}}{S_{n, -\beta}}.
\end{align*}
Thus, for any $i \in [n]$, we have $\left\|D_i f\right\|_\infty \leq \frac{2|\calA|Gi^{-\beta}}{S_{n, - \beta}}$. Noting that
\begin{align*}
\sum_{i = 1}^n \left\|D_i f\right\|_\infty^2 \leq 4|\calA|^2G^2\sum_{i = 1}^n \frac{i^{-2\beta}}{\left(S_{n, -\beta}\right)^2} = 4|\calA|^2G^2 \frac{S_{n, -2\beta}}{\left(S_{n, -\beta}\right)^2} \asymp |\calA|^2G^2 \kappa_{n, \beta}^2.
\end{align*}
Applying McDiarmid's inequality (Proposition~\ref{prop:mcdiarmid}) yields that, for any $x > 0$, one has
\[
\P\bigg(\left|f(X_{1:n}) - \E f(X_{1:n})\right| \gtrsim |\calA|G\kappa_{n, \beta}\sqrt{x}\bigg) \leq 2\exp\left\{-x\right\}.
\]
Thus, with probability at least $1 - 2 \exp\left\{-x\right\}$, we have
\begin{align*}
\sup_{g \in \calG}\left|E_{\ell, n}(g)\right| \lesssim \E\left[\sup_{g \in \calG} \big|E_{\ell, n}(g)\big|\right] + |\calA|G\kappa_{n, \beta}\sqrt{x} \lesssim |\calA|G \kappa_{n, \beta}\big(\sqrt{\frakv} + \sqrt{x}\big).
\end{align*}

\end{proof}

\begin{proof}[Proof of Theorem~\ref{thm:high-prob-chain}]
Our proof follows from combining the bounds presented in Propositions~\ref{prop:martingale} and \ref{prop:emp-process}. First, observe that we have
\[
\sup_{g \in \calG}\left|G_{\ell,n}(g)\right| \leq \sup_{g \in \calG}\left|M_{\ell, n}(g)\right| + \sup_{g \in \calG}\left|E_{\ell, n}(g)\right|.
\]
Applying Propositions~\ref{prop:martingale} and \ref{prop:emp-process} alongside a union bound, we know there are constants $C, c > 0$ (not depending on $n$ or $x$) such that with probability at least $1 - C\exp\{-x\} - C\exp\left\{-cn^{1 - \alpha}\right\}$, we simultaneously have
\[
\sup_{g \in \calG}\left|M_{\ell, n}(g)\right| \lesssim Gn^{\frac{\alpha - 1}{2}}\left(\sqrt{\frakv} + \sqrt{x}\right) + Gn^{\alpha - 1}\left(\frakv \log(n) + x\right)\log(n).
\]
and
\[
\sup_{g \in \calG}\left|E_{\ell, n}(g)\right| \lesssim G\kappa_{n, \beta}\left(\sqrt{\frakv} + \sqrt{x}\right).
\]
We now argue that $\kappa_{n,\beta} \leq n^{\frac{\alpha - 1}{2}}$ for all $\beta \in [0, \alpha]$. First, for $\beta \in [0, 1/2 \land \alpha)$, we clearly have
\[
\kappa_{n, \beta} = n^{-1/2} \leq n^{\frac{\alpha - 1}{2}}.
\]
For $\beta = 1/2$ and $\alpha \geq 1/2$, we have
\[
\kappa_{n, \beta} = \left(\frac{\log(n)}{n}\right)^{1/2} \leq \left(\frac{\sqrt{n}}{n}\right)^{1/2} = n^{-1/4} \leq n^{\frac{\alpha - 1}{2}}.
\]
Lastly, for $\alpha > \frac{1}{2}$ and $\beta \in (1/2, \alpha]$, we have
\[
\kappa_{n, \beta} = n^{\beta - 1} \leq n^{\alpha - 1} \leq n^{\frac{\alpha - 1}{2}}.
\]
Thus, we certainly have on the event that
\[
\sup_{g \in \calG}\left|E_{\ell,n}(g)\right| \lesssim Gn^{\frac{\alpha - 1}{2}}\left(\sqrt{\frakv} + \sqrt{x}\right).
\]
Combining these pieces yields the claimed result.

\end{proof}

\subsection{Localized Expected Supremum Bound}

In this subsection, we aim to control the expected supremum of a localized process that naturally appears in upper bounding the regret of weighted empirical risk minimization. In particular, our goal is to study the localized supremum of the following process:
\begin{equation}
\label{eq:diff-martingale}
M_n(\pi, \rho) := \sum_{t = 1}^n w_t \left\{\wt{Y}_t(\pi) - \wt{Y}_t(\rho) - \left(V(\pi) - V(\rho)\right)\right\}
\end{equation}
where $\pi, \rho \in \Pi$ are arbitrary indexing treatment policies, $w_t = t^{-\alpha}$, and $\wt{Y}_t(\pi) := \frac{\1\{\pi(X_t) = A_t\}}{p_t(A_t \mid X_t)}Y_t$ is again an IPW pseudo-outcome. In words, this process contrasts the IPW estimates of the value of two policies, centering by their expectations in order to ensure martingale structure. Very heuristically, we should think of one of the policies as the learned policy and the other as the ``closest'' optimal policy (as determined in terms of disagreement). When the policies $\pi$ and $\rho$ (e.g., learned and optimal policies) are known to be close in probability, we will be able to establish uniform, localized bounds on the expectation of the above process that decay with the distance between policies. These localized bounds, as noted in the basic inequality in the proof sketch of Theorem~\ref{thm:fast-regret} in Section~\ref{sec:regret}, will be able to show fast, in probability convergence of regret. We now state the main theorem of this subsection --- a localized bound on the expected supremum of $M_n(\pi, \rho)$ when weighted by a proxy for accumulated variance.

\begin{theorem}
\label{thm:local-sup-bound}
Suppose Assumption~\ref{ass:observations}, the first part of Assumption~\ref{ass:estimation}, Assumption~\ref{ass:propensity}, and Assumption~\ref{ass:disagreement} hold. Then,  for any $0 < \delta \leq 1$ and $n \geq 2$, we have
\begin{align*}
\E\left[\frac{1}{S_{n, -\alpha}}\sup_{\substack{\pi, \rho \in \Pi \\ d(\pi, \rho) \leq\delta}}\left|M_n(\pi, \rho)\right|\right] &\lesssim \delta n^{\frac{\alpha - 1}{2}}\sqrt{\frakv\log\left(\frac{A}{\delta}\right)} \\
&\qquad + n^{\alpha - 1}\frakv\log\left(\frac{A}{\delta}\right)\log(n)^2,
\end{align*}
where again $S_{n, \gamma} := \sum_{t  = 1}^n t^\gamma$ for any $\gamma \in \R$, $A \geq e$ is a constant, and the omitted constant has no dependence on $n$ or $\delta$.
\end{theorem}

We briefly provide an interpretation of the above theorem. In the above bound, there are two competing forces. The first term on the right hand side of the bound generally shrinks as $\delta \downarrow 0$, and corresponds to a ``subGaussian'' term that appears in the chaining argument. It tends towards zero because localization forces the predictable quadratic variation of the $M_n(\pi, \rho)$ to be small.
The second term, which is a ``Bernstein'' or ``subExponential'' term, grows as $\delta \downarrow 0$, and is a byproduct of the fact that although predictable quadratic variation localizes well, the $L^\infty$ norms of the increments $m_t(\pi, \rho)$ of $M_n(\pi, \rho)$ do not generally shrink with $\delta$. In our application of Theorem~\ref{thm:local-sup-bound} in the proof of Theorem~\ref{thm:fast-regret}, we thus must be careful to not localize too quickly with respect to the sample size.

We now contrast the technical tools that will be needed to prove Theorem~\ref{thm:local-sup-bound} against those used to prove Theorem~\ref{thm:high-prob-chain}. In proving the latter theorem, we leveraged an intricate argument based on decomposing the target process into two simpler processes. We then needed to construct appropriately-weighted empirical metrics alongside uniform covering numbers in order to obtain high-probability bounds. Because our goal is a bound on the \textit{expected supremum} (instead of a high-probability bound) and because the process will naturally be indexed by categorical treatment policies, we can use a more direct approach in this subsection based on chaining under the natural policy disagreement metric:
\[
d(\pi, \rho) :=  \left(\P(\pi(X) \neq \rho(X))\right)^{1/2}.
\]
In particular, the VC structure of the disagreement class will allow us to argue that the error from truncating the chaining argument early is well-behaved under this metric. However, while the chaining argument used in proving the above theorem is significantly simpler, we need to take extra care to get localized convergence rates, which explicitly depend on the distance between policies being considered.

Before proving Theorem~\ref{thm:local-sup-bound}, we state a key combinatorial lemma that will be used to control the error incurred from terminating chaining early in terms of the VC-dimension of the disagreement class $\calD$, which we recall was defined as
\[
\calD := \Big\{ g_{\pi, \rho}(x) := \1\{\pi(x) \neq \rho(x)\} : \pi, \rho \in \Pi\Big\}.
\]

\begin{lemma}
\label{lem:vc-disagreement}
Let $0 < \delta \leq 1$ be arbitrary and define the $\delta$-disagreement class $\calD(\delta)$ by
\[
\calD(\delta) := \bigg\{g(x) = \1\{\pi(x) \neq \rho(x)\} \in \calD  : d(\pi, \rho) \leq \delta\bigg\}.
\]
Suppose Assumption~\ref{ass:disagreement} holds. Then, the following bound holds for $n \geq 2$:
\[
\E\left[\sup_{g \in \calD(\delta)}\sum_{t = 1}^n g(X_t)\right] \lesssim n\delta^2  + \delta \sqrt{n \frakv \log(n)} + \frakv\log(n),
\]
where the omitted constant has no dependence on $n$ and $\delta$.
\end{lemma}

In proving the above, we leverage the following, standard lemma.

\begin{lemma}[Massart's Finite Class Lemma]
\label{lem:massart}
Let $b_1, \dots, b_N \in \R^n$ be fixed vectors, and let $\epsilon_1, \dots, \epsilon_n$ be i.i.d.\ Rademacher random variables. We have
\[
\E\left[\max_{1 \leq j \leq N}\sum_{i = 1}^n \epsilon_i b_{j, i}\right] \leq \sqrt{2\log(N)}\max_{1 \leq j \leq N}\|b_j\|_2
\]
\end{lemma}

\begin{proof}[Proof of Lemma~\ref{lem:vc-disagreement}]
To bound the expected supremum, we center and then apply a standard symmetrization argument. First, we have
\begin{align*}
\E\left[\sup_{g \in \calD(\delta)}\sum_{t = 1}^n g(X_t)\right] &= \E\left[\sup_{g \in \calD(\delta)}\sum_{t = 1}^n\big\{g(X_t) - \E[g(X)] + \E[g(X)]\big\}\right] \\
&\leq n\sup_{g \in \calD(\delta)}\E\left[g(X)\right] + \E\left[\sup_{g \in \calD(\delta)} \sum_{t = 1}^n \left\{g(X_t) - \E[g(X)]\right\}\right] \\
&\leq n\delta^2 + \E\left[\sup_{g \in \calD(\delta)} \sum_{t = 1}^n \left\{g(X_t) - \E[g(X)]\right\}\right].
\end{align*}
We focus now on bounding the second term. We have via symmetrization that
\begin{align*}
\E\left[\sup_{g \in \calD(\delta)}\sum_{t = 1}^n \left\{g(X_t) - \E[g(X)]\right\}\right] &\leq 2\E\left[\E\left(\sup_{g \in \calD(\delta)}\sum_{t = 1}^n\epsilon_t g(X_t) \mid \X\right)\right],
\end{align*}
where $\epsilon_1, \dots, \epsilon_n$ are i.i.d.\ Rademacher random variables and again $\X := \sigma(X_1, \dots, X_n)$ is the sigma algebra generated by $X_1, \dots, X_n$.  Applying Lemma~\ref{lem:massart} conditionally on $\X$ and noting that taking the supremum over the potentially infinite $\calG$ can be swapped to taking the supremum over the finite class
\[
\calG(x_1^n) := \bigg\{\big(g(x_1), \dots, g(x_n)\big) : g \in \calD(\delta)\bigg\},
\]
we have, letting $g(x_1^n) := (g(x_1), \dots, g(x_n))$, that
\begin{align*}
\E\left(\sup_{g \in \calD(\delta)}\sum_{t = 1}^n \epsilon_t g(X_t) \mid \X\right) &= \E\left[\sup_{g_{1:n} \in \calG(x_1^n)} \sum_{t = 1}^n \epsilon_t g(x_t)  \right] \\
&\leq \sqrt{2\log\left|\calG(x_1^n)\right|}\sup_{g_{1:n} \in \calG(x_1^n)}\|g(x_1^n)\|_2 \\
&= \sqrt{2\log\left|\calG(x_1^n)\right|} \sup_{g \in \calD(\delta)}\left(\sum_{t = 1}^n g(X_t)^2 \right)^{1/2} \\
&= \sqrt{2\log\left|\calG(x_1^n)\right|}\sup_{g \in \calD(\delta)}\left(\sum_{t = 1}^n g(X_t)\right)^{1/2},
\end{align*}
where the first inequality follows from Massart's finite class lemma and the final equality follows because $g(x) = g(x)^2$ for all $x \in \calX$ (since functions are binary). Noting that Jensen's inequality implies
\[
\E\left[\sup_{g \in \calD(\delta)}\left(\sum_{t = 1}^ng(X_t)\right)^{1/2}\right] \leq \left(\E\left[\sup_{g \in \calD(\delta)}\sum_{t = 1}^ng(X_t)\right]\right)^{1/2},
\]
we have the following bound in expectation:
\[
\underbrace{\E\left[\sup_{g \in \calD(\delta)}\sum_{t = 1}^n g(X_t)\right]}_{=:U} \leq \underbrace{n\delta^2}_{=:A} + \underbrace{2\sqrt{2\sup_{x_1^n}\log\left|\calG(x_1^n)\right|}}_{=: B} \times \underbrace{\left(\E\left[\sup_{g \in \calD(\delta)}\sum_{t = 1}^n g(X_t)\right]\right)^{1/2}}_{= \sqrt{U}}.
\]
We now use a standard approach to translate the above bound into one that doesn't have the expected supremum appearing on the right hand side. Rearranging and setting $V := \sqrt{U}$, we have the following quadratic inequality
\[
V^2 - BV - A \leq 0.
\]
For the inequality to hold, the quadratic formula implies that the following upper bound on $V$ must be satisfied:
\[
V \leq \frac{B + \sqrt{B^2 + 4A}}{2} \leq  B  + \sqrt{A},
\]
which follows from the inequality $\sqrt{B^2 + 4A} \leq B + 2\sqrt{A}$. Thus, we have
\[
U = V^2 \leq (B + \sqrt{A})^2 = A + 2\sqrt{A}B + B^2
\]
that is, in terms of our original quantities:
\[
\E\left[\sup_{g \in \calD(\delta)} \sum_{t = 1}^n g(X_t)\right] \leq n\delta^2 + 4\delta\sqrt{2n\sup_{x_1^n}\log\left|\calG(x_1^n)\right|} + 8\sup_{x_1^n}\log\left|\calG(x_1^n)\right|
\]
Now, we know by Lemma~\ref{lem:gen-sauer} that $\left|\calG(x_1^n)\right| \leq \sum_{k =0}^{\min\{n, \frakv\}}\binom{n}{k}$ by Sauer's lemma, and so $\sup_{x_1^n}\log|\calG(x_1^n)| \lesssim \frakv \log(e n)$. Plugging this in yields the desired claim.

\end{proof}

We need two more technical lemmas before proceeding with the proof of Theorem~\ref{thm:local-sup-bound}.  The first is a standard result on the expected maximum of a finite collection of sub-Gamma random variables. The second studies the properties of the process $(M_t(\pi, \rho))_{1 \leq t \leq n}$ for fixed $\pi, \rho \in \Pi$. In particular, it confirms that indeed said process is a martingale, and that its increments/predictable quadratic variation admit reasonable bounds.

\begin{lemma}[Freedman's Inequality for Finite Classes, Theorem 2.5 of \citet{boucheron2003concentration}]
\label{lem:finite_max_bd}
Let $(M_n^{(\ell)} : \ell \in \calI)$ be a finite collection of martingales satisfying (i) $\big|\Delta M_t^{(\ell)}\big| \leq b$  for all $1 \leq t \leq n$ and $\ell \in \calI$, and (ii) $V_n^{(\ell)} := \sum_{t = 1}^n \Var\big(\Delta M_t^{(\ell)}\mid \calF_{t - 1}\big) \leq \nu$ for all $\ell \in \calI$ almost surely. Then, we have
\[
\E\left[\max_{\ell \in \calI} |M_n^{\ell}|\right] \leq \sqrt{2\nu\log(2|\calI|)} + \frac{b}{3}\log(2|\calI|).
\]

\end{lemma}

\begin{lemma}
\label{lem:martingale-properties}
For any $\pi, \rho \in \Pi$, the following are true.
\begin{enumerate}
\item The process $(M_n(\pi, \rho))$ is a martingale with respect to the natural filtration $\calF_n := \sigma(Z_i : 1 \leq i \leq n)$.
\item The increments $m_t(\pi, \rho)$ of $M_n(\pi, \rho)$ satisfy $|m_t(\pi, \rho)| \leq B := 1 + \lb{c}^{-1}$, where $\lb{c} > 0$ is as in Assumption~\ref{ass:propensity}.
\item The predictable variances $\sigma^2_t(\pi, \rho) := \Var(m_t(\pi, \rho) \mid \calF_{t - 1})$ satisfy $\sigma^2_t(\pi, \rho) \leq \frac{2}{\lb{c}}w_t d(\pi, \rho)^2$.
\end{enumerate}
Consequently, the process $(M_n(\pi, \rho))$ is sub-Gamma with variance $\nu_{\pi, \rho} :=\frac{2}{\lb{c}} d(\pi, \rho)^2 S_{n, -\alpha}$ and scale $b = \frac{B}{3}$.
\end{lemma}

\begin{proof}
We check each claim in succession. First, the martingale property is straightforward to verify, as we have
\begin{align*}
\E\left(m_t(\pi, \rho) \mid \calF_{t - 1}\right) &= w_t\E\left(\wt{Y}_t(\pi) - V(\pi) \mid \calF_{t - 1}\right) - w_t \E\left(\wt{Y}_t(\rho) - V(\rho) \mid \calF_{t - 1}\right).
\end{align*}
We just check that the conditional expectation involving $\pi$ is zero, as analyzing the other conditional expectation is exactly analogous. First, note that we have the identity
\begin{align*}
&\E\left(\sum_{a \in \calA} \frac{\1\{A_t = a, \pi(X_t) = a\}}{p_t(a \mid X_t)}Y_t \mid X_t, \calF_{t - 1}\right) \\
&\qquad = \sum_{a \in \calA}\E\left(\frac{\1\{A_t = a, \pi(X_t) = a\}}{p_t(a \mid X_t)}Y_t \mid \calX_t, \calF_{t - 1}\right) \\
&\qquad = \sum_{a \in \calA}\E\left(\frac{\1\{A_t = a, \pi(X_t) = a\}}{p_t(a \mid X_t)}Y_t(a) \mid \calX_t, \calF_{t - 1}\right) \\
&\qquad = \sum_{a \in \calA}\1\{\pi(X_t) = a\}\E\left(\frac{\1\{A_t = a\}}{p_t(a \mid X_t)} \mid X_t, \calF_{t - 1}\right)\E\left(Y_t(a) \mid X_t, \calF_{t - 1}\right) \\
&\qquad =\sum_{a \in \calA} \1\{\pi(X_t) = a\}\mu(a, X_t).
\end{align*}
Consequently, we have
\begin{align*}
\E\left(\wt{Y}_t(\pi) \mid \calF_{t - 1}\right) &= \E\left(\sum_{a \in \calA} \frac{\1\{A_t = a, \pi(X_t) = a\}}{p_t(a \mid X_t)}Y_t \mid \calF_{t - 1}\right) \\
&= \E\left[\E\left(\sum_{a \in \calA} \frac{\1\{A_t = a, \pi(X_t) = a\}}{p_t(a \mid X_t)}Y_t \mid X_t, \calF_{t - 1}\right) \mid \calF_{t - 1}\right] \\
&= \E\left[\sum_{a \in \calA} \1\{\pi(X_t) = a\}\mu(a, X_t) \mid \calF_{t - 1}\right] \\
&= \E\left[\sum_{a \in \calA}\1\{\pi(X_t) = a\} \mu(a, X_t)\right] = V(\pi).
\end{align*}
Thus, it follows that $\E\left(m_t(\pi, \rho) \mid \calF_{t - 1}\right) = 0$ and hence $M_n(\pi, \rho)$ is a martingale with respect to $\calF_n$.

\noindent Next, we check boundedness of the increments $m_t(\pi, \rho)$. We have
\begin{align*}
\left|m_t(\pi, \rho)\right| &\leq w_t\left|\wt{Y}(\pi) - \wt{Y}(\rho)\right| + w_t\left|V(\pi) - V(\rho)\right|.
\end{align*}
The difference between true policy values satisfies $\left|V(\pi) - V(\rho)\right| \leq d(\pi, \rho)^2 \leq 1$, as we have
\begin{align*}
\left|V(\pi) - V(\rho)\right| &= \left|\E\left[\mu(\pi(X), X) - \mu(\rho(X), X)\right]\right|\\
&\leq \E\left[\1\left\{\pi(X) \neq \rho(X)\right\}\left|\mu(\pi(X), X) - \mu(\rho(X), X)\right|\right] \\
&\leq \E\left[\1\{\pi(X) \neq \rho(X)\}\right] \leq d(\pi, \rho)^2.
\end{align*}
Thus, since $w_t \leq 1$ by assumption, we have $w_t \left|V(\pi) - V(\rho)\right| \leq 1$.

Next, we have
\begin{align*}
w_t \left|\wt{Y}_t(\pi) - \wt{Y}_t(\rho)\right| &= w_t \left|\frac{\1\{A_t = \pi(X_t)\}}{p_t(A_t\mid X_t)} - \frac{\1\{A_t = \rho(X_t)}{p_t(A_t \mid X_t)}\right|Y_t \\
&\leq w_t \frac{\1\{\pi(X_t) \neq \rho(X_t)\}}{p_t(A_t \mid X_t)} \\
&\leq \frac{1}{\lb{c}}.
\end{align*}
Thus, the increments $m_t(\pi, \rho)$ are bounded as claimed. We conclude by controlling the predictable variance. First, note the following upper bound when we additionally condition on $X_t$:
\begin{align*}
&\E\left[\left(\wt{Y}_t(\pi) - \wt{Y}_t(\rho)\right)^2 \mid X_t, \calF_{t - 1}\right] \\
&\qquad =\E\left[\1\left\{\pi(X_t)\neq \rho(X_t)\right\}\left(\frac{\1\{A_t = \pi(X_t)\}}{p_t(\pi(X_t) \mid X_t)} - \frac{\1\{A_t = \rho(X_t)\}}{p_t(\rho(X_t) \mid X_t)}\right)^2Y_t^2 \mid X_t, \calF_{t -1}\right] \\
&\qquad \leq \1\left\{\pi(X_t)\neq \rho(X_t)\right\}\left(\frac{\mu_2(\pi(X_t), X_t)}{p_t(\pi(X_t) \mid X_t)} + \frac{\mu_{2}(\rho(X_t), X_t)}{p_t(\rho(X_t) \mid X_t)}\right) \\
&\qquad \leq \frac{2t^{\alpha}}{\lb{c}}\1\left\{\pi(X_t) \neq \rho(X_t)\right\}
\end{align*}
Leveraging this bound, we see that we have
\begin{align*}
\Var(m_t(\pi, \rho) \mid \calF_{t - 1}) &\leq w_t^2\E\left[\left(\wt{Y}_t(\pi) - \wt{Y}_t(\rho)\right)^2 \mid \calF_{t - 1}\right] \\
&= w_t^2 \E\left[\E\left(\left(\wt{Y}_t(\pi) - \wt{Y}_t(\rho)\right)^2\mid X_t, \calF_{t -1}\right)\mid \calF_{t - 1}\right] \\
&\leq \frac{2t^\alpha w_t^2}{\lb{c}}\E\left[\1\left\{\pi(X_t) \neq \rho(X_t)\right\} \mid \calF_{t - 1}\right] \\
&= \frac{2w_t}{\lb{c}}d(\pi, \rho)^2,
\end{align*}
where the last line follows because $X_t$ is independent of $\calF_{t - 1}$. This thus completes the proof.

\end{proof}

\begin{proof}[Proof of Theorem~\ref{thm:local-sup-bound}]
Our proof follows from a chaining argument. For each $j \geq 0$, let $\delta_j := 2^{-j}\delta$ and let $K_j$ be a minimal $\delta_j$-covering of $\Pi$ under the disagreement metric $d(\pi, \rho) := \P(\pi(X) \neq \rho(X))^{1/2}$. let $J_n \geq 1$ be the smallest $j$ such that $\delta_{j} \leq n^{-1/2}$, and for each $j \in \{0, 1, \dots, J_n\}$ and $\pi, \rho \in \Pi$, let $\pi_j, \rho_j \in K_j$ be points satisfying $d(\pi, \pi_j) \leq \delta_j$ and $d(\rho, \rho_j) \leq \delta_j$. Performing a truncated chaining decomposition up to scale $J_n$ yields:
\begin{align*}
M_n(\pi, \rho) &= M_n(\pi_0, \rho_0) + \sum_{j = 1}^{J_n}\left\{M_n\big(\pi_j, \rho_j\big) - M_n\big(\pi_{j - 1}, \rho_{j - 1}\big)\right\}  \\
&\qquad + M_n(\pi, \rho) -  M_n\big(\pi_{J_n}, \rho_{J_n}\big) \\
&= M_n(\pi_0, \rho_0) + \sum_{j = 1}^{J_n}\left\{M_n\big(\pi_j, \pi_{j - 1}\big) - M_n\big(\rho_j, \rho_{j - 1}\big)\right\}  \\
&\qquad + M_n\big(\pi, \pi_{J_n}\big) -  M_n\big(\rho, \rho_{J_n}\big),
\end{align*}
where the final equality follows from the fact that
\begin{align*}
M_n(\pi, \rho) - M_n(\pi', \rho') &= \sum_{t = 1}^n w_t\left\{\wt{Y}(\pi) - \wt{Y}(\rho) - (V(\pi) - V(\rho))\right\} \\
&\qquad - \sum_{t = 1}^n w_t \left\{\wt{Y}(\pi') - \wt{Y}(\rho') - (V(\pi') - V(\rho'))\right\} \\
&= \sum_{t = 1}^n w_t\left\{\wt{Y}(\pi) - \wt{Y}(\pi') - (V(\pi) - V(\pi'))\right\} \\
&\qquad - \sum_{t = 1}^n w_t \left\{\wt{Y}(\rho) - \wt{Y}(\rho') - (V(\rho) - V(\rho'))\right\} \\
&= M_n(\pi, \pi') - M_n(\rho, \rho').
\end{align*}

Thus, the triangle inequality yields the following upper bound.
\begin{align*}
\left|M_n(\pi, \rho)\right| &\leq  \underbrace{\left|M_n(\pi_0, \rho_0)\right|}_{\text{Coarse Term}} + \underbrace{\sum_{j = 1}^{J_n}\left|M_n\big(\pi_j, \pi_{j - 1}\big)\right| + \sum_{j = 1}^{J_n}\left|M_n\big(\rho_j, \rho_{j - 1}\big)\right|}_{\text{Chaining Term}}  \\
&\qquad + \underbrace{\left|M_n\big(\pi, \pi_{J_n}\big)\right| + \left|M_n\big(\rho, \rho_{J_n}\big)\right| }_{\text{Remainder Term}}\\
\end{align*}
We now bound each of the terms in sequence.

\paragraph{Bounding the Coarse Term}
First, observe that by the triangle inequality we have
\[
d(\pi_0, \rho_0) \leq d(\pi_0, \pi) + d(\pi, \rho) + d(\rho, \rho_0) \leq 3 \delta.
\]
Next, define relevant set of pairs $\calR_0 \subset K_0 \times K_0$ by
\[
\calR_0 := \left\{(\pi, \rho) \in K_0 \times K_0 : d(\pi, \rho) \leq 3 \delta\right\}
\]
Clearly, $(\pi_0, \rho_0) \in \calR_0$ by construction, so we have the upper bound:
\[
\left|M_n\big(\pi_0, \rho_0\big)\right| \leq \max_{(\pi', \rho') \in \calR_0}\left|M_n(\pi', \rho')\right|.
\]
Next, for any pair $(\pi', \rho')$, we know by Lemma~\ref{lem:martingale-properties} that $M_n(\pi', \rho')$ is sub-Gamma with variance $\nu := \frac{18}{\lb{c}}\delta^2\sum_{t = 1}^n w_t = \frac{18}{\lb{c}}\delta^2 S_{n, -\alpha}$ and scale $b = \frac{B}{3}$, where again $B = 1 + \lb{c}^{-1}$.
Applying Lemma~\ref{lem:finite_max_bd} over the finite set $\calR_0$ and omitting dependence on the constant $\lb{c}$ and the action set cardinality $|\calA|$ thus yields
\begin{align*}
&\E\left[\frac{1}{S_{n, -\alpha}}\max_{(\pi', \rho') \in \calR_0}\left|M_n(\pi', \rho')\right|\right] \lesssim \delta S_{n, -\alpha}^{-1/2}\sqrt{\log\left(2\left|\calR_0\right|\right)} + S_{n, -\alpha}^{-1}\log\left(2\left|\calR_0\right|\right) \\
&\quad \lesssim \delta S_{n, -\alpha}^{-1/2} \sqrt{\log\left(2\left|K_0\right|^2\right)} + S_{n, -\alpha}^{-1}\log\left(2\left|K_0\right|^2\right) \\
&\quad \lesssim \delta S_{n , -\alpha}^{-1/2}\sqrt{\frakv\log\left(\frac{A|\calA|}{\delta}\right)} + S_{n, -\alpha}^{-1}\frakv \log\left(\frac{A|\calA|}{\delta}\right) \\
&\lesssim  \delta n^{\frac{\alpha - 1}{2}} \sqrt{\frakv\log\left(\frac{A}{\delta}\right)} + n^{\alpha - 1}\frakv\log\left(\frac{A}{\delta}\right).
\end{align*}
In the above, the second to last inequality follows from Assumption~\ref{ass:disagreement} followed by an application of Proposition~\ref{prop:vc-bound-pi} and the final inequality follows because $S_{n, -\alpha} \asymp n^{1 - \alpha}$ for $\alpha \in [0, 1)$.

\paragraph{Bounding the Chaining Terms}
We bound $|M_n(\pi_j, \pi_{j - 1})|$ for each $1 \leq j \leq J_n$. Similar reasoning to the above yields that $d(\pi_j, \pi_{j - 1}) \leq 3 \delta_j$ for each $j \geq 1$. For each $j$, let $\calR_j \subset K_j \times K_{j - 1}$ be defined similarly to $\calR_0$ as follows
\[
\calR_j := \left\{(\pi, \rho) \in K_j \times K_{j - 1} : d(\pi, \rho) \leq 3 \delta_j\right\}.
\]
Also note that we have
\[
\E\left[\left|M_n(\pi_j, \pi_{j - 1})\right|\right] \leq \E\left[\max_{(\pi', \rho') \in \calR_j} \left|M_n(\pi', \rho')\right|\right].
\]
Noting that $|\calR_j| \leq |K_j||K_{j - 1}| \leq |K_j|^2$, analogous reasoning to the bounding of the coarse term yields that
\begin{align*}
\E\left[\frac{1}{S_{n, -\alpha}}\max_{(\pi' , \rho') \in \calR_j}\left|M_n(\pi', \rho')\right|\right] &\lesssim  n^{\frac{\alpha - 1}{2}}\delta_j \sqrt{\frakv\log\left(\frac{A}{\delta_j}\right)} \\
&\qquad + n^{\alpha - 1}\frakv\log\left(\frac{A}{\delta_j}\right).
\end{align*}

Now, we upper bound the sum over $j \in \{1, \dots, J_n\}$. First, we have the following upper bound:
\begin{align*}
\E\left[\frac{1}{S_{n, -\alpha}}\sum_{j = 1}^{J_n}\left|M_n(\pi_j, \pi_{j - 1})\right|\right] &\leq \frac{1}{S_{n, -\alpha}}\sum_{j = 1}^{J_n}\E\left[\max_{(\pi', \rho') \in \calR_j}\left|M_n(\pi', \rho')\right|\right] \\
&\lesssim n^{\frac{\alpha - 1}{2}}\sum_{j = 1}^{J_n} \delta_j \sqrt{\frakv \log\left(\frac{A}{\delta_j}\right)} +  n^{\alpha - 1}\sum_{j = 1}^{J_n}\frakv \log\left(\frac{A}{\delta_j}\right)
\end{align*}
The first sum can be bounded as follows.
\begin{align*}
\sum_{j = 1}^{J_n} \delta_j \sqrt{\frakv \log\left(\frac{A}{\delta_j}\right)} &\leq \sum_{j = 1}^\infty \frac{\delta}{2^{j}}\sqrt{\frakv\log\left(\frac{A2^j}{\delta}\right)} \\
&\lesssim  \int_0^\delta \sqrt{\frakv \log\left(\frac{A}{\epsilon}\right)}d\epsilon \\
&\asymp \delta\sqrt{\frakv\log\left(\frac{A}{\delta}\right)},
\end{align*}
where the final inequality follows from the relation
\[
\delta \sqrt{\frakv \log\left(\frac{A}{\delta}\right)} \leq \int_0^\delta \sqrt{\frakv \log\left(\frac{A}{\epsilon}\right)}d\epsilon \leq \delta \sqrt{2\frakv\log\left(\frac{A}{\delta}\right)}
\]
for every $0 < \delta \leq 1$.

We now bound the second sum. Our assumption that $J_n$ is the smallest natural number such that $2^{-J_n}\delta \leq n^{-1/2}$ and that $\delta \leq 1$ implies that $J_n \leq \frac{1}{2}\log_2(n) + 1 \lesssim \log(n)$ so long as $n \geq 2$, so we have
\begin{align*}
\sum_{j = 1}^{J_n} \frakv \log\left(\frac{A}{\delta_j}\right) &\lesssim \sum_{j = 1}^{J_n} \frakv j \log\left(\frac{A}{\delta}\right) \\
&\lesssim J_n^2\frakv\log\left(\frac{A}{\delta}\right) \\
&\lesssim \log(n)^{2}\frakv \log\left(\frac{A}{\delta}\right).
\end{align*}

Noting that the analysis of $\sum_{j = 1}^{J_n} \left|M_n(\rho_j, \rho_{j - 1})\right|$ is exactly analogous, we have that
\begin{align*}
&\E\left[\frac{1}{S_{n, -\alpha}}\sum_{j =1}^{J_n} \left|M_n(\pi_j, \pi_{j - 1})\right|\right] + \E\left[\frac{1}{S_{n, -\alpha}}\sum_{j = 1}^{J_n}\left|M_n(\rho_j, \rho_{j - 1})\right| \right] \\
&\qquad \lesssim n^{\frac{\alpha - 1}{2}} \delta \sqrt{\frakv\log\left(\frac{A}{\delta}\right)} + n^{\alpha - 1}\log(n)^2 \frakv \log\left(\frac{A}{\delta}\right).
\end{align*}

\paragraph{Bounding the Remainder}

First, we have the following upper bound.
\begin{align*}
\frac{1}{S_{n, -\alpha}}\left|M_n(\pi_{J_n}, \pi)\right| &\leq \frac{1}{S_{n, -\alpha}}\left|\sum_{t = 1}^n w_t \left\{\wt{Y}_t(\pi_{J_n}) - \wt{Y}_t(\pi)\right\}\right| \\
&\qquad + \frac{1}{S_{n, -\alpha}}\left|\sum_{t = 1}^n w_t \left\{V(\pi_{J_n}) - V(\pi)\right\}\right|.
\end{align*}
The second term can be bounded as follows:
\begin{align*}
\frac{1}{S_{n, -\alpha}}\left|\sum_{t = 1}^n w_t \left\{V(\pi_{J_n}) - V(\pi)\right\}\right| = \left|V(\pi_{J_n}) - V(\pi)\right| \lesssim  n^{-1},
\end{align*}
where the final inequality follows from the fact that
\[
\left|V(\pi) - V(\pi_{J_n})\right| \leq \E\left[\1\left\{\pi_{J_n}(X) \neq \pi(X)\right\}|Y|\right] \leq \delta_{J_n}^2 \lesssim n^{-1}.
\]

We now bound the first term. Observe that
\begin{align*}
\frac{1}{S_{n, -\alpha}}\left|\sum_{t = 1}^n w_t\left\{\wt{Y}_t(\pi_{J_n}) - \wt{Y}_t(\pi)\right\}\right| &\leq \frac{1}{S_{n, -\alpha}}\sum_{t = 1}^nw_t \frac{\1\left\{\pi_{J_n}(X_t) \neq \pi(X_t)\right\}}{p_t(A_t \mid X_t)}\left|Y_t\right| \\
&\lesssim \frac{1}{S_{n, -\alpha}}\sum_{t = 1}^n \1\left\{\pi_{J_n}(X_t) \neq \pi(X_t)\right\}.
\end{align*}
Using Lemma~\ref{lem:vc-disagreement}, we can bound the expectation of the last term. In particular, we have
\begin{align*}
\E\left[\sum_{t = 1}^n \1\left\{\pi_{J_n}(X_t) \neq \pi(X_t)\right\}\right] &\leq  \E\left[\sup_{d(\pi', \rho') \leq \delta_{J_n}} \sum_{t = 1}^n \1\left\{\pi'(X_t) \neq \rho'(X_t)\right\}\right] \\
&\lesssim n\delta_{J_n}^2 + \delta_{J_n}\sqrt{n\frakv\log(n)} + \frakv\log(n) \\
&\lesssim  1  + \sqrt{\frakv\log(n)} + \frakv\log(n),
\end{align*}
Thus, we have
\begin{align*}
\E\left[\frac{1}{S_{n, -\alpha}}\big|M_n(\pi_{J_n}, \pi)\big|\right] &\lesssim n^{\alpha  -1}\sqrt{\frakv \log(n)} + n^{\alpha - 1}\frakv\log(n) + n^{-1} \\
&\lesssim n^{\alpha - 1}\frakv\log(n) \\
&\lesssim n^{\alpha - 1}\log(n)^2 \frakv\log\left(\frac{A}{\delta}\right).
\end{align*}
Note that an exactly identical argument provides the same bound on the remainder involving $\rho$.

\paragraph{Wrapping up.}

Putting together our bounds on each term, we have that
\begin{align*}
\E\left[\frac{1}{S_{n, -\alpha}}\sup_{\substack{\pi, \rho \in \Pi \\ d(\pi, \rho) \leq\delta}}\left|M_n(\pi, \rho)\right|\right] &\lesssim  \delta n^{\frac{\alpha - 1}{2}}\sqrt{\frakv\log\left(\frac{A}{\delta}\right)} \\
&\qquad + n^{\alpha - 1}\log(n)^2\frakv\log\left(\frac{A}{\delta}\right).
\end{align*}
This yields the claimed rate.
\end{proof}

\section{Additional Proofs}
\label{app:additional}

% In this appendix, we adopt the shorthand notation $\E_{t - 1}[U] := \E_{t - 1}(U \mid \calF_{t - 1})$ and $\Var_{t - 1}[U] := \Var\left(u \mid \calF_{t- 1}\right)$ to ease notation while conditioning.

\begin{proof}[Proof of Lemma~\ref{lem:suff_conditions}]
We start by establishing a lower bound on the predictable variance that we will use in checking all three of the sufficient conditions. Let $\pi \in \Pi$ be an arbitrary policy. First, note that the law of total variance yields
\begin{align*}
\Var\left(\wt{Y}_t(\pi, q)\mid \calF_{t - 1}\right)  &= \E\left[\Var\left(\wt{Y}_t(\pi, q) \mid X_t, A_t, \calF_{t - 1}\right) \mid \calF_{t - 1}\right] \\
&\qquad + \Var\left(\E\left[\wt{Y}_{t}(\pi, q) \mid X_t, A_t, \calF_{t - 1}\right] \mid \calF_{t - 1}\right) \\
&\geq \E\left[\Var\left(\wt{Y}_t(\pi, q) \mid X_t, A_t, \calF_{t - 1}\right) \mid \calF_{t - 1}\right].
\end{align*}
Next, note that Assumption~\ref{ass:observations} and \ref{ass:estimation} taken jointly yield
\begin{align*}
\Var\left(\wt{Y}_t(\pi, q) \mid X_t = x, A_t = a, \calF_{t - 1}\right) &= \frac{g_\pi(a \mid x)^2}{p_t(a \mid x)^2}\Var(Y_t \mid X_t = x, A_t = a, \calF_{t - 1}) \\
&= \frac{g_\pi(a \mid x)^2}{p_t(a \mid x)^2}\Var\left(Y_t(a) \mid X_t = x, \calF_{t - 1}\right) \\
&\geq \frac{g_\pi(a \mid x)^2}{p_t(a \mid x)^2}\lb{\sigma}^2.
\end{align*}
We now lower bound $\E\left[\frac{g_\pi(A_t \mid X_t)^2}{p_t(A_t \mid X_t)^2} \mid \calF_{t - 1}\right]$ in each of the three settings.

\paragraph{Strictly Randomized Baseline.}
In the setting of a strictly randomized baseline, we have

\begin{align*}
\E\left(\frac{g_\pi(A_t \mid X_t)^2}{p_t(A_t \mid X_t)^2} \mid \calF_{t - 1}\right) &=\E\left(\sum_{a \in \calA} \frac{g_\pi(a \mid X_t)^2}{p_t(a \mid X_t)} \mid \calF_{t - 1}\right) \\
&\geq \E\left(\sum_{a \in \calA} \frac{\eta_q^2}{p_t(a \mid X_t)} \mid \calF_{t - 1}\right) \\
&\geq \E\left(\frac{\eta_q^2}{\min_{a \in \calA}p_t(a \mid X_t)}\mid \calF_{t - 1}\right)\\
&\geq \frac{\eta^2_qt^{\alpha}}{\wb{c}},
\end{align*}
which completes the proof in this case.

\paragraph{Deterministic Baseline.}
Using similar reasoning as above, we have that
\begin{align*}
\E\left(\frac{g_\pi(A_t \mid X_t)^2}{p_t(A_t \mid X_t)^2} \mid \calF_{t - 1}\right) &=\E\left(\frac{\1\left\{\pi(X_t) \neq q(X_t)\right\}}{p_t(A_t \mid X_t)^2} \mid \calF_{t - 1}\right) \\
&= \E\left(\1\left\{\pi(X_t) \neq q(X_t)\right\}\left(\frac{1}{p_t(1 \mid X_t)} + \frac{1}{p_t(0 \mid X_t)}\right) \mid \calF_{t - 1}\right) \\
&\geq \E\left(\1\left\{\pi(X_t) \neq q(X_t)\right\}\frac{1}{\min\left\{p_t(1 \mid X_t),p_t(0 \mid X_t)\right\}} \mid \calF_{t - 1}\right) \\
&\geq \frac{t^\alpha}{\wb{c}}\P\left(\pi(X_t) \neq q(X_t)\right) \\
&\geq \frac{\eta_qt^\alpha}{\wb{c}},
\end{align*}
where again the second equality follows from expanding the expectation over just the treatment assignment $A_t$. This thus shows the desired result in this case.

\paragraph{Non-Vanishing Exploration.}
Finally, observe that we have
\begin{align*}
\E\left(\frac{g_\pi(A_t \mid X_t)^2}{p_t(A_t \mid X_t)^2} \mid \calF_{t - 1}\right) &\geq \E\left(g_\pi(A_t \mid X_t)^2 \mid \calF_{t - 1}\right) \geq \eta_g.
\end{align*}
Since $t^\alpha = 1$ in this setting, the result is proved.

\end{proof}

\section{Additional Details From Section~\ref{sec:exp}}
\label{app:exp}

In this appendix, we provide additional details, plots, and interpretations from Section~\ref{sec:exp} --- the main experimental section of our work. We start with a brief road map of the appendix. First, in Appendix~\ref{app:exp:comp}, we rigorously define the various comparators leveraged in experiments. Then, in Appendix~\ref{app:exp:synth}, we provide additional details, tables, and plots from our synthetic experiments. Finally, we provide some additional details on our semi-synthetic experiments.

\subsection{Comparator Details}
\label{app:exp:comp}

We start by providing a detailed overview of the comparators outlined in Table~\ref{tab:comparators}.
Throughout, recall that we have defined $g_\pi(a \mid x) := \1\{\pi(x) = a\} - q(a \mid x)$ as the contrast between the value of a policy
$\pi$ and that of a baseline policy $q$, $p_t(\cdot \mid \cdot)$ is the logging policy's selection
probability at round $t$, and $\wh{\sigma}_t$ is an estimate of the predictable standard deviation. Our adaptively-weighted estimator \aw{} is given in full by
\cref{alg:online-epoch}; we describe the other estimators below.

\paragraph{CADR \citep{bibaut2021post}.} This estimator is closely related to our
inverse standard deviation-weighted estimator and likewise derives from the online estimator of
\citet{luedtke2016statistical}. The key difference between this estimator and our own is that the former is designed to evaluate a single, \emph{fixed} policy
rather than a sequence of policies that are refit over time. Using sample splitting, we can leverage this estimator to perform inference on the value of policy improvement. Given an initial policy $\wh{\pi}$ learned via weighted empirical welfare maximization on the first
 $1 \leq b_n < n$ samples, the \biba{} estimate is given by
\begin{equation}
\label{eq:bibaut-est}
\wh{\theta}_n^{\biba} := \frac{\sum_{t = b_n + 1}^n \wh{\sigma}_t^{-1}\frac{g_{\wh{\pi}}(A_t \mid X_t)Y_t}{p_t(A_t\mid X_t)}}{\sum_{t = b_n + 1}^n \wh{\sigma}_t^{-1}},
\end{equation}
with the corresponding $(1 - \delta)$-confidence interval
\[
C_{1 - \delta}^{\biba} := \left[\wh{\theta}_n - \frac{\sqrt{n - b_n}}{\sum_{t = b_n + 1}^{n}\wh{\sigma}_t^{-1}}z_{1 - \delta/2}, \wh{\theta}_n + \frac{\sqrt{n - b_n}}{\sum_{t = b_n + 1}^n \wh{\sigma}_t^{-1}}z_{1 - \delta/2}\right].
\]
Here $\wh{\sigma}_t$ is the standard deviation estimator defined in Section~\ref{sec:estimator:variance}, except that the single policy
$\wh{\pi}$ is used in place of the sequence of policies refit over epochs.

\paragraph{StabVar \citep{zhan2021off}.} Unlike \aw{} and \biba{}, this estimator does not apply
inverse standard deviation weighting to the IPW pseudo-outcomes. Instead it
sets
\[
h_t^{-1} := \sqrt{\frac{1}{t - b_n - 1}\sum_{s = b_n + 1}^{t - 1}\sum_{a \in \calA}\frac{g_{\wh{\pi}}(a \mid X_s)^2}{p_t(a \mid X_s)}}
\]
and then defines the point estimate $\wh{\theta}^{\zhan}_n$ via
\[
\wh{\theta}^{\zhan}_n := \frac{\sum_{t = b_n + 1}^n h_t\frac{g_{\wh{\pi}}(A_t \mid X_t)Y_t}{p_t(A_t \mid X_t)}}{\sum_{t= b_n + 1}^n h_t},
\]
together with the confidence interval
\[
C_{1 - \delta}^{\zhan} := \left[\wh{\theta}_n - \wh{\sigma}_n z_{1 - \delta/2}, \wh{\theta}_n + \wh{\sigma}_n z_{1 - \delta/2}\right],
\]
where
\[
\wh{\sigma}^2_n := \frac{\sum_{t = b_n + 1}^n h_t^2 \left\{\frac{g_{\wh{\pi}}(A_t \mid X_t)Y_t}{p_t(A_t \mid X_t)} - \wh{\theta}_n^{\zhan}\right\}^2}{\left(\sum_{t = b_n + 1}^n h_t\right)^2}.
\]

\paragraph{Double-dipping.} Finally, the double-dipping estimator $\wh{\theta}_n^{\dd}$ uses the
entire sample to learn a policy by weighted ERM and then evaluates that policy on the same sample.
It is exactly the estimator of \cref{eq:bibaut-est} with $b_n = 0$ and
$\wh{\pi}_n \in \arg\max_{\pi \in \Pi} \sum_{t=1}^n t^{-\alpha}\frac{g_\pi(A_t \mid X_t)Y_t}{p_t(A_t \mid X_t)}$.
Because it reuses the data for both steps, it is the only comparator without out-of-sample
evaluation, which, as discussed in \cref{sec:exp}, is the source of its persistent
under-coverage when the optimal policy is non-unique.
In particular, authors such as \citet{carcamo2020directional} and \citet{fang2019inference} show that the corresponding limiting distribution of the centered version of the estimator is not normal, but rather follows the distribution of the supremum of a limiting Gaussian process.

\paragraph{Burn-in vs.\ $50/50$ splits.} Both \biba{} and \zhan{} require a choice of how much data
to spend on policy learning, and we run each in two modes. In the \emph{$50/50$} mode, we set
$b_n = n/2$, dedicating half of the data to learning the policy and half to evaluating it. This is
the regime in which the comparators enjoy theoretical coverage guarantees, at the cost of discarding
a constant fraction of the data. In the \emph{burn-in} mode we instead set $b_n$ equal to the
(sublinear) burn-in period used by \aw{}; this leaves far more data for evaluation but forfeits the
coverage guarantees. Together with \aw{} and \dd{}, this yields the six estimator variants reported
in the experiments.

\subsection{Additional Synthetic Details}
\label{app:exp:synth}

We now provide remaining details from Section~\ref{sec:exp:synthetic}, which amounts to several lingering tables/figures. Table~\ref{tab:nu-cov-synth} displays the coverage statistics for the best-performing estimators alongside 95\% Wilson confidence intervals. Figure~\ref{fig:multiple-optima} shows two possible depth two decision/policy trees that both attain the optimal policy value under our synthetic DGP.

\begin{table}[htbp]
\centering
\footnotesize
\setlength{\tabcolsep}{6pt}
\caption{Synthetic DGP: Empirical $95\%$ coverage [with $95\%$
Wilson interval] of the three well-calibrated estimators over $250$ independent runs; contrast truth
$\theta^\ast\approx0.350$.}
\label{tab:nu-cov-synth}
\begin{tabular}{c r c c c}
\toprule
$\alpha$ & $n$ & \aw & \zhan~(50/50) & \biba~(50/50) \\
\midrule
\multirow{4}{*}{$0$}
 & 2000  & 0.880 [.83,\,.91] & 0.888 [.84,\,.92] & 0.880 [.83,\,.91] \\
 & 5000  & 0.952 [.92,\,.97] & 0.940 [.90,\,.96] & 0.924 [.88,\,.95] \\
 & 10000 & 0.960 [.93,\,.98] & 0.964 [.93,\,.98] & 0.956 [.92,\,.98] \\
 & 15000 & 0.932 [.89,\,.96] & 0.940 [.90,\,.96] & 0.924 [.88,\,.95] \\
\midrule
\multirow{4}{*}{$0.1$}
 & 2000  & 0.872 [.82,\,.91] & 0.912 [.87,\,.94] & 0.904 [.86,\,.93] \\
 & 5000  & 0.924 [.88,\,.95] & 0.940 [.90,\,.96] & 0.924 [.88,\,.95] \\
 & 10000 & 0.948 [.91,\,.97] & 0.944 [.91,\,.97] & 0.920 [.88,\,.95] \\
 & 15000 & 0.952 [.92,\,.97] & 0.976 [.95,\,.99] & 0.956 [.92,\,.98] \\
\midrule
\multirow{4}{*}{$0.25$}
 & 2000  & 0.856 [.81,\,.89] & 0.880 [.83,\,.91] & 0.824 [.77,\,.87] \\
 & 5000  & 0.920 [.88,\,.95] & 0.968 [.94,\,.98] & 0.936 [.90,\,.96] \\
 & 10000 & 0.944 [.91,\,.97] & 0.956 [.92,\,.98] & 0.944 [.91,\,.97] \\
 & 15000 & 0.972 [.94,\,.99] & 0.940 [.90,\,.96] & 0.912 [.87,\,.94] \\
\midrule
\multirow{4}{*}{$0.5$}
 & 2000  & 0.772 [.72,\,.82] & 0.856 [.81,\,.89] & 0.480 [.42,\,.54] \\
 & 5000  & 0.820 [.77,\,.86] & 0.908 [.87,\,.94] & 0.484 [.42,\,.55] \\
 & 10000 & 0.816 [.76,\,.86] & 0.892 [.85,\,.92] & 0.572 [.51,\,.63] \\
 & 15000 & 0.804 [.75,\,.85] & 0.884 [.84,\,.92] & 0.476 [.41,\,.54] \\
\bottomrule
\end{tabular}
\end{table}

\begin{figure}[t]
  \centering
  \begin{subfigure}[t]{0.45\textwidth}
    \centering
    \includegraphics[width=\linewidth]{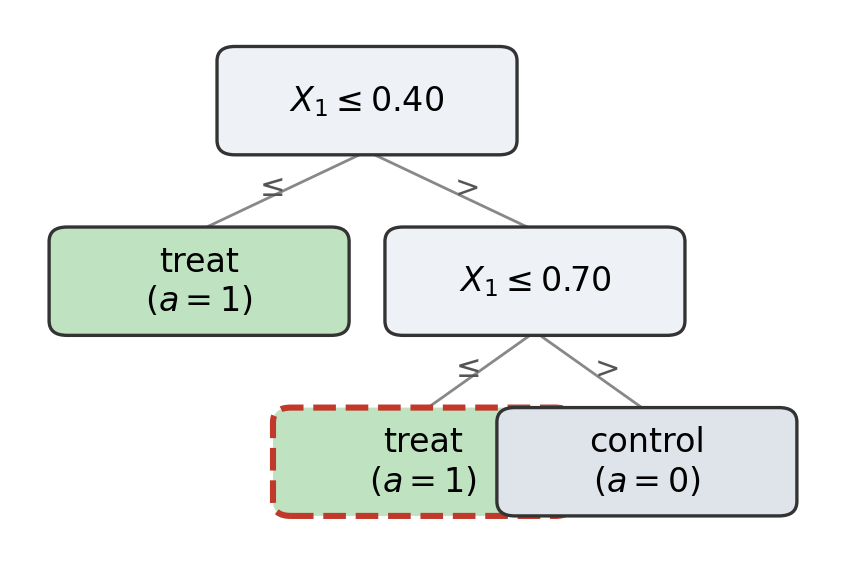}
    \caption{Treat on $[0,0.7)$, control on $[0.7,1)$.}
    \label{fig:multiple-optima-a}
  \end{subfigure}\hfill
  \begin{subfigure}[t]{0.45\textwidth}
    \centering
    \includegraphics[width=\linewidth]{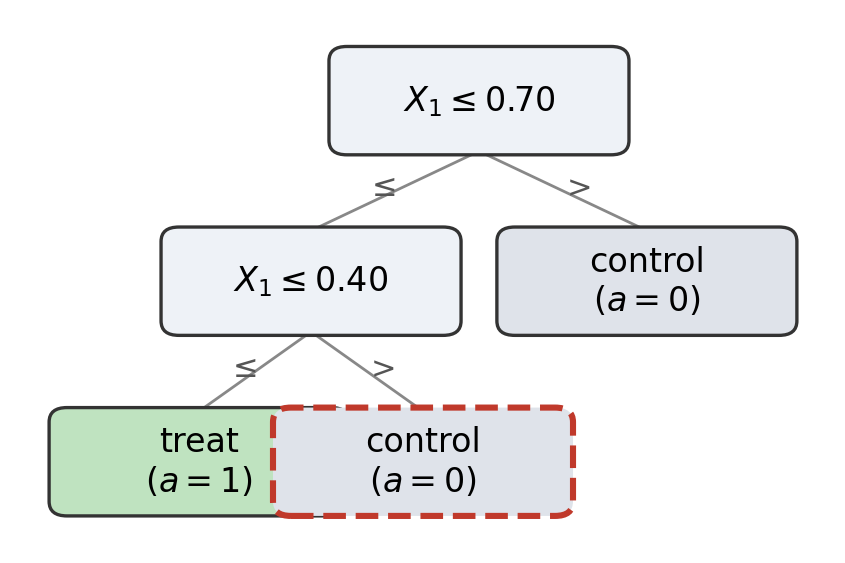}
    \caption{Treat on $[0,0.4)$, control on $[0.4,1)$.}
    \label{fig:multiple-optima-b}
  \end{subfigure}
  \caption{Examples of two depth 2 decision trees with minimum node size of at least 15\%  that  obtain the optimal policy value. Note that these policies just differ in how they assign treatments to units with $x_1 \in [0.4, 0.7]$.}
  \label{fig:multiple-optima}
\end{figure}

\subsection{Additional Semi-Synthetic Details}
\label{app:exp:semi-synth}

We now provide additional details and plots from our semi-synthetic experiments. We first describe how treatments and covariates were selected from the collection of 100 jokes from the Jester dataset, and also provide several summary statistics.

\paragraph{Additional Details on Treatment and Covariate Selection}
As noted in the main body of the paper, we select the five treatments in order to ensure none is strictly preferred over the other. Table~\ref{tab:treatment-summary} displays both the average ratings and rating standard deviations for the five selected jokes. We see that jokes 50 and 81 generally have the highest ratings, whereas jokes 16 and 57 have the lowest. Further, the standard deviations for joke ratings are roughly the same for all jokes, with that of joke 50 being mildly lower than the others.
The same table also displays the ``best arm share'' for each joke, or the fraction of time that the noted joke is preferred over the other four by units in the dataset. We see no joke is uniformly preferred over the others, and that the ranking of jokes by best arm share is the same as ranking by mean rating. Table~\ref{tab:treatment-corr} displays the correlation in ratings amongst the five jokes. The ratings of all jokes are positively correlated, with the rating correlation being strongest between the two least popular jokes.

\begin{table}[ht]
\centering
\caption{Marginal diagnostics for the fixed five-joke treatment benchmark. The ``best arm share'' column gives the marginal fraction of the time that the noted arm wins.}
\label{tab:treatment-summary}
\begin{tabular}{rrrr}
\toprule
joke id & mean & std & best arm share \\
\midrule
16 & -2.595 & 5.027 & 0.058 \\
57 & -1.515 & 5.446 & 0.084 \\
41 & 0.011 & 5.484 & 0.139 \\
81 & 1.986 & 5.318 & 0.293 \\
50 & 3.822 & 4.150 & 0.426 \\
\bottomrule
\end{tabular}

\end{table}

\begin{table}[ht]
\centering
\caption{Pairwise correlation matrix for user ratings in the fixed five joke treatment set.}
\label{tab:treatment-corr}
\begin{tabular}{rrrrrr}
\toprule
joke id & 16 & 57 & 41 & 81 & 50 \\
\midrule
16 & 1.000 & 0.331 & 0.141 & 0.067 & 0.072 \\
57 & 0.331 & 1.000 & 0.288 & 0.166 & 0.121 \\
41 & 0.141 & 0.288 & 1.000 & 0.256 & 0.225 \\
81 & 0.067 & 0.166 & 0.256 & 1.000 & 0.205 \\
50 & 0.072 & 0.121 & 0.225 & 0.205 & 1.000 \\
\bottomrule
\end{tabular}

\end{table}

We select a set of 20  of the remaining 95 jokes to serve as covariates. Explicitly, the selected jokes are:
\[
\calS_{20} = \{71, 2, 75, 7, 80, 60, 3, 98, 44, 5, 90, 15, 79, 85, 51, 13, 43, 73, 86, 100\}.
\]
We construct the set $\calS_{20}$ using an unsupervised selection procedure. For each candidate joke $j$ we compute
the empirical variance $\hat v_j = \tfrac{1}{n}\sum_i (X_{ij} - \bar X_j)^2$ and rescale it to lie in the
unit interval: $\tilde v_j = (\hat v_j - \min_k \hat v_k)/(\max_k \hat v_k - \min_k \hat v_k)$. This is done to ensure that all candidates lie on a common variance scale. Next, writing
$\hat\rho_{jk}$ for the empirical correlation between jokes $j$ and $k$, we build the set greedily.
The first set of selected jokes is given by $\calS_1 = \arg\max_j \tilde v_j$. Having chosen $\calS_{m-1}$, we set $\calS_m := \calS_{m - 1} \cup \{j_m\}$ where
\[
j_m = \arg\max_{j \notin \calS_{m-1}} \left\{ \tilde v_j - \frac{1}{|\calS_{m-1}|}\sum_{k \in \calS_{m-1}} |\hat\rho_{jk}| \right\},
\]
with ties broken by smaller joke ID. This objective rewards high marginal variability while
penalizing jokes that are strongly correlated with those already selected. We note that this approach is treatment-agnostic
by construction, depending only on the raw covariate distribution in the Jester dataset. We display the statistics associated with the chosen jokes in Table~\ref{tab:covariate_selection_table}.

\begin{table}[!ht]
\centering
\caption{Results of the covariate selection procedure. The table reports the selected joke ID at each
step alongside the normalized variance score, average absolute-correlation penalty relative to
the previously selected jokes, and the resulting net objective value.}
\label{tab:covariate_selection_table}
\begin{tabular}{rrrrr}
\toprule
step & joke id & variance score & redundancy penalty & objective value \\
\midrule
1 & 71 & 1.000 & 0.000 & 1.000 \\
2 & 2 & 0.773 & 0.142 & 0.630 \\
3 & 75 & 0.764 & 0.168 & 0.596 \\
4 & 7 & 0.730 & 0.138 & 0.592 \\
5 & 80 & 0.780 & 0.243 & 0.537 \\
6 & 60 & 0.684 & 0.177 & 0.507 \\
7 & 3 & 0.661 & 0.237 & 0.424 \\
8 & 98 & 0.627 & 0.217 & 0.410 \\
9 & 44 & 0.625 & 0.203 & 0.422 \\
10 & 5 & 0.564 & 0.173 & 0.391 \\
11 & 90 & 0.631 & 0.234 & 0.397 \\
12 & 15 & 0.494 & 0.114 & 0.380 \\
13 & 79 & 0.641 & 0.252 & 0.389 \\
14 & 85 & 0.647 & 0.274 & 0.373 \\
15 & 51 & 0.647 & 0.289 & 0.357 \\
16 & 13 & 0.470 & 0.137 & 0.333 \\
17 & 43 & 0.587 & 0.250 & 0.337 \\
18 & 73 & 0.616 & 0.285 & 0.331 \\
19 & 86 & 0.561 & 0.235 & 0.326 \\
20 & 100 & 0.598 & 0.277 & 0.321 \\
\bottomrule
\end{tabular}

\end{table}

\paragraph{Additional Plots.}
We display the precise coverage numbers for the top-performing estimators alongside 95\% Wilson CIs in Table~\ref{tab:nu-cov-semisynth} below.

\begin{table}[htbp]
\centering
\footnotesize
\setlength{\tabcolsep}{6pt}
\caption{Semi-synthetic DGP: Empirical $95\%$ coverage [with $95\%$
Wilson interval] of the three well-calibrated estimators over $250$ independent runs.}
\label{tab:nu-cov-semisynth}
\begin{tabular}{c r c c c}
\toprule
$\alpha$ & $n$ & \aw & \zhan~(50/50) & \biba~(50/50) \\
\midrule
\multirow{6}{*}{$0$}
 & 2000  & 0.952 [.92,\,.97] & 0.940 [.90,\,.96] & 0.864 [.82,\,.90] \\
 & 5000  & 0.964 [.93,\,.98] & 0.944 [.91,\,.97] & 0.912 [.87,\,.94] \\
 & 8000  & 0.948 [.91,\,.97] & 0.956 [.92,\,.98] & 0.932 [.89,\,.96] \\
 & 10000 & 0.948 [.91,\,.97] & 0.936 [.90,\,.96] & 0.932 [.89,\,.96] \\
 & 15000 & 0.968 [.94,\,.98] & 0.912 [.87,\,.94] & 0.920 [.88,\,.95] \\
 & 20000 & 0.948 [.91,\,.97] & 0.936 [.90,\,.96] & 0.924 [.88,\,.95] \\
\midrule
\multirow{6}{*}{$0.1$}
 & 2000  & 0.940 [.90,\,.96] & 0.888 [.84,\,.92] & 0.784 [.73,\,.83] \\
 & 5000  & 0.932 [.89,\,.96] & 0.908 [.87,\,.94] & 0.868 [.82,\,.90] \\
 & 8000  & 0.916 [.88,\,.94] & 0.932 [.89,\,.96] & 0.872 [.82,\,.91] \\
 & 10000 & 0.964 [.93,\,.98] & 0.968 [.94,\,.98] & 0.912 [.87,\,.94] \\
 & 15000 & 0.940 [.90,\,.96] & 0.948 [.91,\,.97] & 0.912 [.87,\,.94] \\
 & 20000 & 0.928 [.89,\,.95] & 0.944 [.91,\,.97] & 0.920 [.88,\,.95] \\
\midrule
\multirow{6}{*}{$0.25$}
 & 2000  & 0.928 [.89,\,.95] & 0.876 [.83,\,.91] & 0.680 [.62,\,.73] \\
 & 5000  & 0.924 [.88,\,.95] & 0.892 [.85,\,.92] & 0.772 [.72,\,.82] \\
 & 8000  & 0.932 [.89,\,.96] & 0.932 [.89,\,.96] & 0.812 [.76,\,.86] \\
 & 10000 & 0.916 [.88,\,.94] & 0.924 [.88,\,.95] & 0.864 [.82,\,.90] \\
 & 15000 & 0.948 [.91,\,.97] & 0.936 [.90,\,.96] & 0.848 [.80,\,.89] \\
 & 20000 & 0.944 [.91,\,.97] & 0.904 [.86,\,.93] & 0.860 [.81,\,.90] \\
\bottomrule
\end{tabular}
\end{table}

\section{Additional Details From Section~\ref{sec:case_study}}
\label{app:exp:case_study}

We now provide additional details from \citet{caria2024adaptive} that are useful for contextualizing our findings. The rough outline of this appendix is as follows. First, we describe which data we leverage from \citet{caria2024adaptive}. The, we describe describe the tempered Thompson sampling algorithm. Next, we describe how the authors compute their estimates for the effect of targeting/value of policy improvement. Lastly, we present and discuss some supplemental statistics, such as the the leveraged sample sizes/burn-in sizes and the learned policies.

\paragraph{Data Structure.}

All estimates are computed from the replication package of \citet{caria2024adaptive}. The base file is \texttt{Data/2020-05-15\_priordata.csv}, the
3{,}709-observation extract on which the authors' Bayesian analysis operates. It
supplies the assignment time stratum (\texttt{covar1}, sixteen cells), the
one-hot treatment indicators \texttt{treatment1}--\texttt{treatment4} (Cash,
Information, Nudge, Control), the six-week wage-employment outcome
(\texttt{outcome}, the value the assignment algorithm conditioned on and observed
for all 3{,}709 units), and the intake date. Three further files are merged onto
it. \texttt{Data/strata\_codes.csv} decodes \texttt{covar1} into nationality,
gender, above-secondary education and prior wage employment.
\texttt{Data/Outcomes.dta}, matched on its \texttt{id} field, supplies the
longer-run outcomes \texttt{employed1} (two months) and \texttt{employed2} (four
months); horizon-specific analyses use complete cases, giving 3{,}631 and 3{,}516
observations. The 388 dated files \texttt{Data/*\_treatmentprobabilities.csv},
each a stratum-by-arm matrix, supply the known assignment propensities: a unit
receives the row for its stratum from the file dated on or immediately before its
intake, with a uniform $0.25$ for the 585 units assigned before the first file
(7 April 2019). We estimate on all 3{,}709 observations rather than the
3{,}668-observation subset defined by \texttt{Data/merged\_full\_data.csv}, which
governs the package's randomization-inference branch.

\paragraph{Tempered Thompson Sampling Algorithm}

Using the above Bayesian hierarchical model noted in Section~\ref{sec:case_study}, \citet{caria2024adaptive} propose a tempered Thompson sampling algorithm that ensures each unit has at least a constant probability of getting any treatment. In more detail, let $m_{t}^{a, x}$ denote the total number of times treatment $a$ has been given to units in stratum $x$ up to (and not including) day $t$ and let $r_t^{a, x}$ denote the total number of successes in said stratum under the given treatment.\footnote{In the paper, the outcomes from treatment assignment on day $t - 1$ are not known on day $t$. The authors thus instead take $r_t^{a, x}$ and $m_t^{a, x}$ to be defined in terms of \textit{units who have reported their status} before the start of day $t$.} The \textit{untempered} probability $p_t^{a, x}$ that a unit is assigned to receive treatment $a$ on day $t$ is then the posterior probability under the hierarchical prior (given the recorded number of success and treatment assignments for each pair $a, x$) of a unit in stratum $x$ having treatment $a$ as the argument maximizer of their expected outcomes, i.e.,
\[
\wh{p}_{t}^{a, x} := \P\left(a = \arg\max_{a' \in \calA} \mu^{a', x} \mid m_t, r_t\right),
\]
where $m_t := \left\{m_t^{a, x} : (a, x) \in \calA \times \calX\right\}$ and $r_t$ is defined analogously. Given a parameter $\gamma \in [0, 1]$, the tempered treatment assignment is then given as follows:
\[
\P(A_{i, t} = a \mid m_t, r_t, X_{i, t} = x) = \frac{\gamma}{|\calA|} + (1 - \gamma)\cdot \wh{p}_t^{a, x}.
\]

\paragraph{Markov Chain Monte Carlo}

We now detail how the authors construct both point-estimates and 95\% credible sets for the effect of targeting. Let
\[
\Theta := \left\{\mu_{a, x} : a \in \calA, x \in \calX\right\}
\]
denote a random sample from the posterior distribution conditioned on the data $\calD$. The authors run the following algorithm (Algorithm~\ref{alg:caria-targeting-short}) to construct estimates and credible intervals for $\theta_{0, B}$.

\begin{algorithm}[htbp]
\caption{Posterior estimate of the value of targeting}
\label{alg:caria-targeting-short}
\begin{algorithmic}[1]
\Require Number of MCMC samples $R$, observed data $\calD$.

\State Draw posterior samples $\{\Theta^{(r)}\}_{r=1}^R$ from the posterior given $\calD$.
\State Estimate per $(a, x)$ posterior means:
$\widehat\mu_{a,x} = R^{-1}\sum_{r=1}^R\mu_{a,x}^{(r)}$.
\For{$x\in\mathcal X$}
    \State Define
    $\widehat\pi(x)
    :=
    \arg\max_{a\in\mathcal A}\widehat\mu_{a,x}.$
\EndFor
\State Compute weights:
\[
p_x := \frac{\sum_{i, t} \1\left\{X_{i, t} = x\right\}}{\sum_{t \in \calT}n_t}.
\]
\State Compute value of policy improvement estimate
\[
\wh{\theta} :=
\sum_{x\in\mathcal X}
p_x
\left[
\widehat\mu_{\widehat\pi(x),x}
-
\widehat\mu_{0,x}
\right].
\]

\For{$r=1,\ldots,R$}
    \State Holding $\widehat\pi$ fixed, compute
    \[
    \theta^{(r)}
    =
    \sum_{x\in\mathcal X}
    p_x
    \left[
    \mu_{\wh{\pi}(x),x}^{(r)}
    -
    \mu_{0,x}^{(r)}
    \right].
    \]
\EndFor

\State \Return point estimate $\widehat\theta$ and
$2.5\%$ and $97.5\%$ quantiles of
$\{\theta^{(r)}\}_{r=1}^R$.
\end{algorithmic}
\end{algorithm}

\paragraph{Mean Gaps Under the Bayesian Prior.}

In the main body of this paper, we noted that the hierarchical Bayesian prior assumed by \citet{caria2024adaptive} ensured that each stratum $x \in \calX$ admitted a unique best intervention. We then noted that, given the data, there was not strong evidence that this assumption held. One lingering question is as follows: what does the distribution of the gap between the first and second best expected outcomes look like. We show the distribution of (a) the sub-optimality gaps for a fixed stratum and (b) the minimum gap across strata in Figure~\ref{fig:bayes-prior-gaps} below. Since the the prior leveraged by \citep{caria2024adaptive} is not proper, we adopt the practice done in the authors' code by enforcing $\alpha, \beta < 1/2$. We further note properties of the realized, empirical distributions in Table~\ref{tab:bayes-prior-gap-summary}.

We find that, for any fixed stratum, the gaps between the expected outcomes under the first and second best actions are not tightly concentrated. The gap falls below 1 pp only 4\% of the time, and the median gap is around 16 pp.
Naturally, the distribution of the minimum sub-optimality gap over the strata is much more tightly concentrated around zero. Here, there is around a 25\% chance of the gap being less than 0.5 pp, and a 43\% chance of it being less than 1 pp.
Thus, we see that the assumed Bayesian model may be likely to generate near ties for some strata, but never can generate exact ties. This latter setting is precisely what leads to a ``winner's curse'' in estimating the VPI.

Under the hierarchical Bayesian model used by Caria et al., the stratum-arm
success probabilities are continuous random variables. Therefore, conditional
on the model, exact equality between the best and second-best treatment in a
stratum has probability zero. This does not imply that the practical
suboptimality gap is large. As a prior-predictive diagnostic, we draw
treatment-specific hyperparameters from the paper's hyperprior, draw
stratum-level treatment means, and record the best-minus-second-best gap.
Table~\ref{tab:bayes-prior-gap-summary} summarizes 200,000 prior-predictive
draws.

\begin{table}[ht]
\centering
\caption{Prior-predictive Bayesian suboptimality gap diagnostics}
\label{tab:bayes-prior-gap-summary}
\begin{tabular}{lrrrrrr}
\toprule
Diagnostic & Mean & 5\% & Median & 95\% & $P(<0.5\text{ pp})$ & $P(<1\text{ pp})$ \\
\midrule
Single stratum & 20.52 & 1.27 & 16.24 & 54.33 & 2.0\% & 4.0\% \\
Minimum over 16 strata & 2.82 & 0.08 & 1.25 & 10.72 & 26.2\% & 43.3\% \\
\bottomrule
\end{tabular}

\end{table}

\begin{figure}[ht]
\centering
\begin{subfigure}[t]{0.48\linewidth}
\centering
\includegraphics[width=\linewidth]{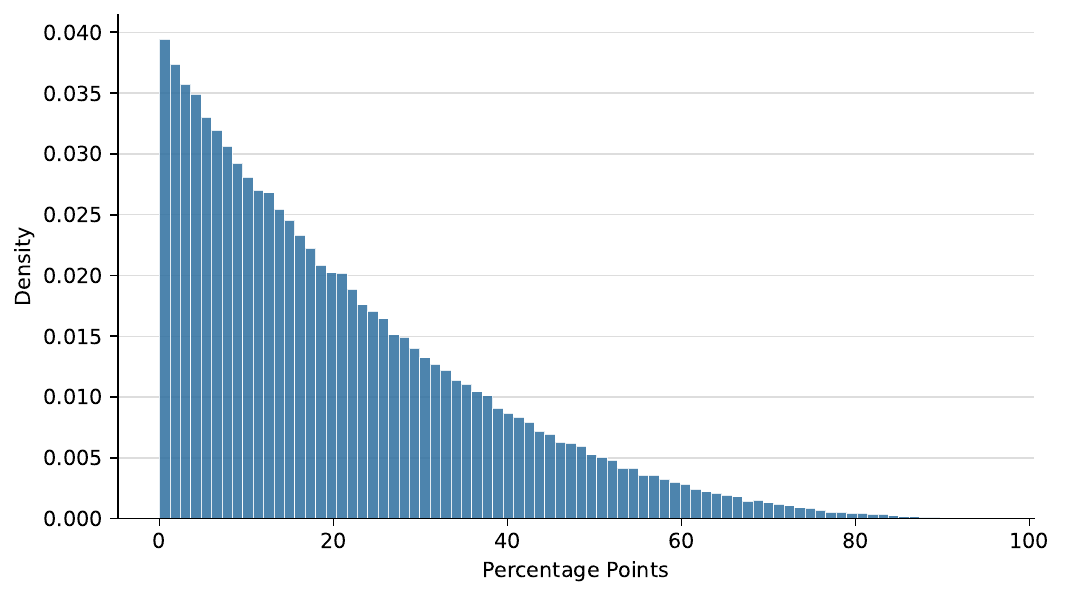}
\caption{Single stratum.}
\label{fig:bayes-prior-single-gap}
\end{subfigure}
\hfill
\begin{subfigure}[t]{0.48\linewidth}
\centering
\includegraphics[width=\linewidth]{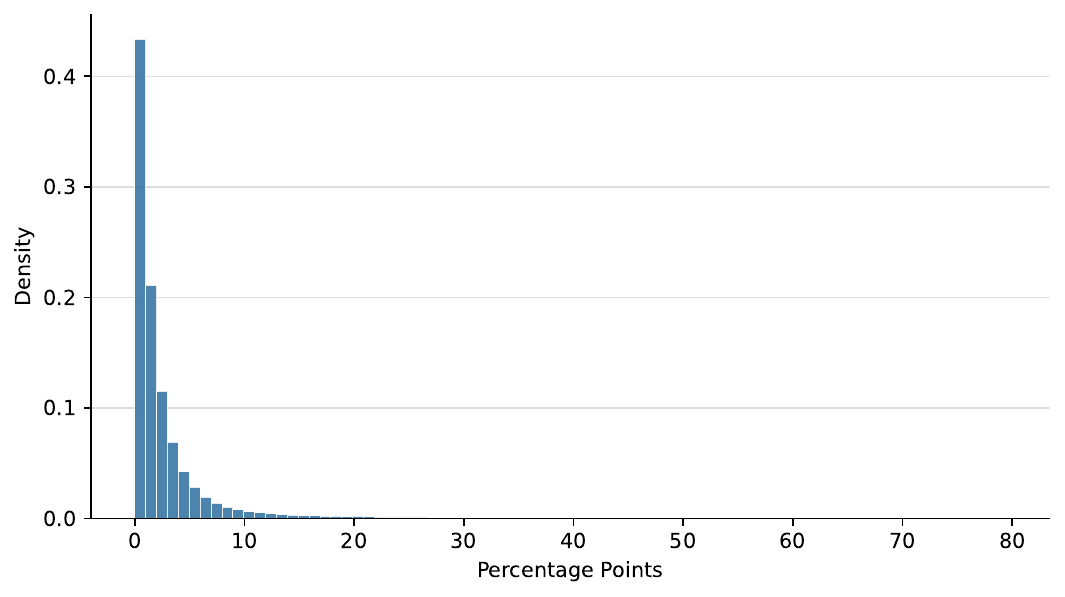}
\caption{Minimum over sixteen strata.}
\label{fig:bayes-prior-min-gap}
\end{subfigure}
\caption{Prior-predictive distributions of best-minus-second-best Bayesian arm gaps.}
\label{fig:bayes-prior-gaps}
\end{figure}

\paragraph{Additional Plots}
We lastly show several plots that are supplemental to our findings in Section~\ref{sec:case_study}. Table~\ref{tab:burn20-cutoffs} shows the realized burn-in rates associated with our estimator alongside raw number of samples. Figures~\ref{fig:policy-heatmap-burn20-all-6w}-\ref{fig:policy-heatmap-burn20-all-4m} compare the realized policies constructed using our estimator  and the Bayesian estimator of \citet{caria2024adaptive}. For our estimator, we only display the final learned policy. Note that the actions chosen by the policies are generally similar.

\begin{table}[ht]
\centering
\caption{Natural day-boundary burn-in cutoffs after 20 percent of each sample}
\label{tab:burn20-cutoffs}
\resizebox{\textwidth}{!}{\begin{tabular}{llrrrrl}
\toprule
Group & Horizon & $n$ & Burn-in rows & Burn-in share & Burn-in days & First eval date \\
\midrule
All & 6 weeks & 3709 & 759 & 20.5\% & 42 & 2019-04-17 \\
All & 2 months & 3631 & 732 & 20.2\% & 38 & 2019-04-17 \\
All & 4 months & 3516 & 724 & 20.6\% & 39 & 2019-04-22 \\
Jordanian & 6 weeks & 2064 & 422 & 20.4\% & 40 & 2019-04-17 \\
Jordanian & 2 months & 2013 & 405 & 20.1\% & 38 & 2019-04-17 \\
Jordanian & 4 months & 1941 & 393 & 20.2\% & 38 & 2019-04-17 \\
Syrian & 6 weeks & 1645 & 337 & 20.5\% & 41 & 2019-04-17 \\
Syrian & 2 months & 1618 & 327 & 20.2\% & 38 & 2019-04-17 \\
Syrian & 4 months & 1575 & 317 & 20.1\% & 39 & 2019-04-22 \\
\bottomrule
\end{tabular}
}
\end{table}

\begin{figure}[p]
\centering
\includegraphics[width=\textwidth]{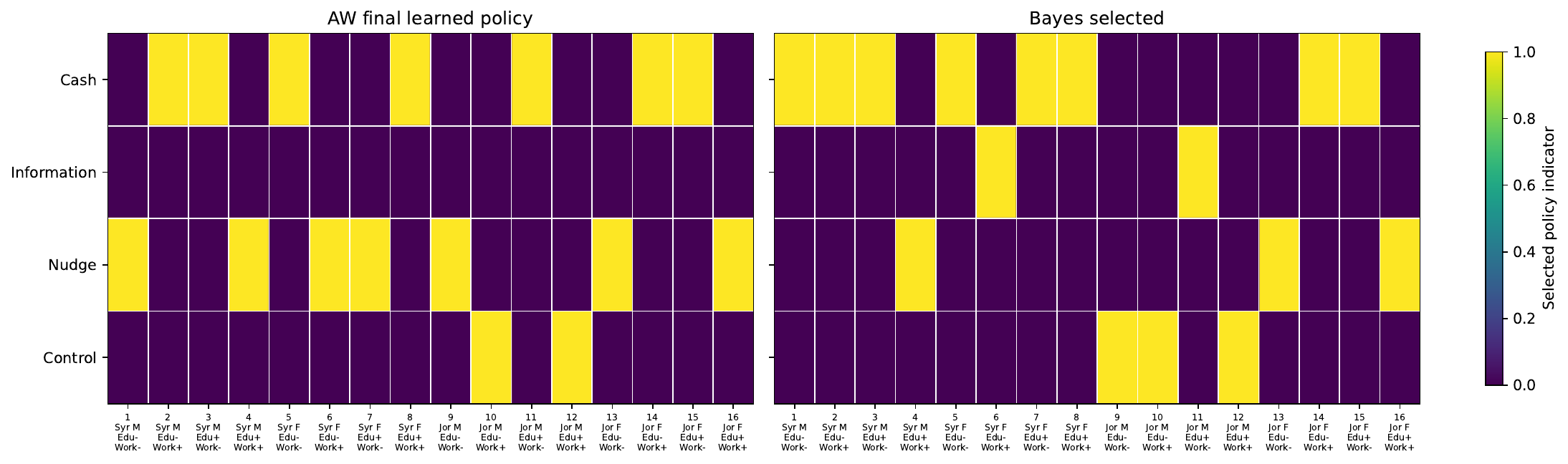}
\caption{Full-sample learned-policy heatmaps, 6-week employment, natural 20 percent burn-in}
\label{fig:policy-heatmap-burn20-all-6w}
\end{figure}

\begin{figure}[p]
\centering
\includegraphics[width=\textwidth]{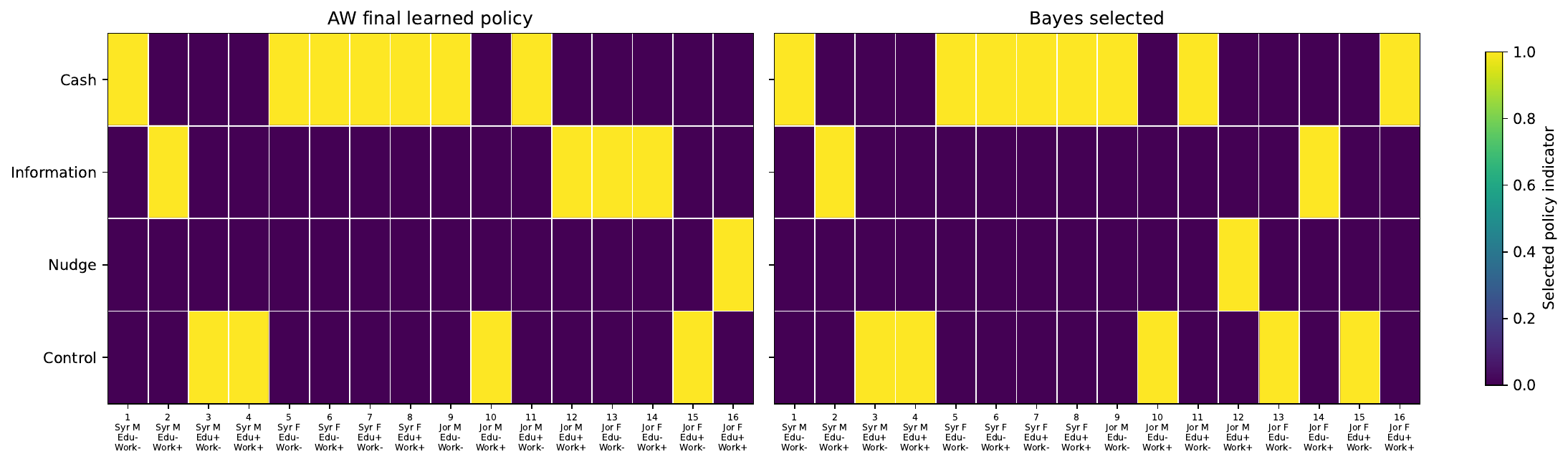}
\caption{Full-sample learned-policy heatmaps, 2-month employment, natural 20 percent burn-in}
\label{fig:policy-heatmap-burn20-all-2m}
\end{figure}

\begin{figure}[p]
\centering
\includegraphics[width=\textwidth]{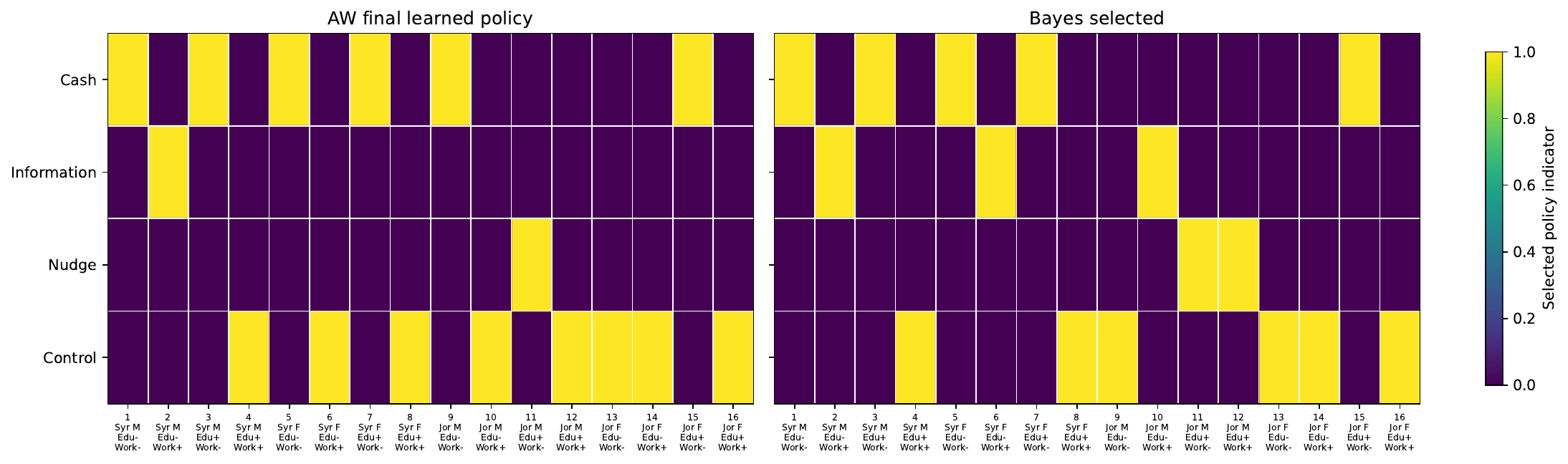}
\caption{Full-sample learned-policy heatmaps, 4-month employment, natural 20 percent burn-in}
\label{fig:policy-heatmap-burn20-all-4m}
\end{figure}

\end{document}